\documentclass[onefignum,onetabnum]{siamonline220329}
\usepackage[margin=1.206in]{geometry} 

\usepackage[utf8]{inputenc} 
\usepackage[T1]{fontenc}    
\usepackage{url}            
\usepackage{booktabs}       
\usepackage{amsfonts}       
\usepackage{nicefrac}       
\usepackage{microtype}      
\usepackage{fbox}
\usepackage{tikz} 
\usepackage{pgfplots}
\usepackage[export]{adjustbox}
\usepackage{amssymb}
\newlength\figureheight
\newlength\figurewidth

\title{Recovering Simultaneously Structured Data via Non-Convex Iteratively Reweighted Least Squares
}

\usepackage{lipsum}
\usepackage{amsfonts}

\usepackage{enumitem}
\setlist[enumerate]{leftmargin=.5in}
\setlist[itemize]{leftmargin=.5in}

\newsiamremark{remark}{Remark}
\newsiamremark{hypothesis}{Hypothesis}
\crefname{hypothesis}{Hypothesis}{Hypotheses}
\newsiamthm{claim}{Claim}

\title{Recovering Simultaneously Structured Data via Non-Convex Iteratively Reweighted Least Squares}
\author{Christian K\"ummerle\thanks{Department of Computer Science, University of North Carolina at Charlotte, Charlotte, NC 28223, USA 
  (\email{kuemmerle@charlotte.edu}, \url{http://ckuemmerle.com}).}
\and Johannes Maly\thanks{Department of Mathematics, Ludwig-Maximilians-Universität München, 80799 Munich, Germany and Munich Center for Machine Learning (MCML), Munich, Germany
  (\email{maly@math.lmu.de}, \url{https://johannes-maly.github.io}).}}

\usepackage{amsopn}

\usepackage[english]{babel}
\usepackage{amsmath}
\usepackage{algorithm, algorithmic}
\usepackage{graphicx}
\usepackage{xcolor}
\usepackage{mathrsfs}
\usepackage{mdframed}

\newcommand{\specnorm}[1]{\Vert #1 \Vert}
\newcommand{\Sp}[2]{\left\| #1 \right\|_{S_{#2}}}
\usepackage[nothm]{thmbox}
\usepackage{comment}
\usepackage{mathrsfs}
\usepackage{cleveref}
\crefformat{equation}{\textup{#2(#1)#3}}
\crefrangeformat{equation}{\textup{#3(#1)#4--#5(#2)#6}}
\crefmultiformat{equation}{\textup{#2(#1)#3}}{ and \textup{#2(#1)#3}}
{, \textup{#2(#1)#3}}{, and \textup{#2(#1)#3}}
\crefrangemultiformat{equation}{\textup{#3(#1)#4--#5(#2)#6}}%
{ and \textup{#3(#1)#4--#5(#2)#6}}{, \textup{#3(#1)#4--#5(#2)#6}}{, and \textup{#3(#1)#4--#5(#2)#6}}
\usepackage{subcaption}
\usepackage{soul}
\usepackage{array}
\newcolumntype{L}{>{$}l<{$}}

\newcommand{\diag}{\mathrm{diag}}
\newcommand{\id}{\mathbf{Id}}

\newcommand{\sign}{\mathrm{sign}}
\newcommand{\supp}{\mathrm{supp}}

\newcommand{\A}{\mathcal{A}}
\renewcommand{\AA}{\mathbf{A}}
\newcommand{\B}{\mathbf{B}}

\newcommand{\F}{\mathcal{F}}
\newcommand{\Flr}{\mathcal{F}_{\textit{lr},\varepsilon}}
\newcommand{\Fsp}{\mathcal{F}_{\textit{sp},\delta}}
\newcommand{\Flrk}[1]{\mathcal{F}_{\textit{lr},\varepsilon_{#1}}}
\newcommand{\Fspk}[1]{\mathcal{F}_{\textit{sp},\delta_{#1}}}

\newcommand{\Hk}[1]{\mathbf H (\boldsymbol \sigma^{(#1)},\varepsilon_{#1})}
\newcommand{\M}{\mathbf{M}}
\newcommand{\MM}[1]{\mathcal{M}_{#1}}

\newcommand{\NN}[1]{\mathcal{N}_{#1}}
\renewcommand{\P}{\mathbb{P}}

\newcommand{\Qlre}{\mathcal{Q}_{\textit{lr},\varepsilon}}
\newcommand{\Qspd}{\mathcal{Q}_{\textit{sp},\delta}}
\newcommand{\Qlrk}{\mathcal{Q}_{\textit{lr},\varepsilon_k}}
\newcommand{\Qspk}{\mathcal{Q}_{\textit{sp},\delta_k}}

\newcommand{\tSk}[1]{ S}
\newcommand{\TT}[1]{T_{#1}}
\newcommand{\U}{\mathbf{U}}
\newcommand{\tU}{\widetilde{\mathbf{U}}}
\newcommand{\Uk}[1]{\U}
\newcommand{\tUk}[1]{ \widetilde{\U}}
\newcommand{\V}{\mathbf{V}}
\newcommand{\tV}{\widetilde{\mathbf{V}}}
\newcommand{\Vk}[1]{\V}
\newcommand{\tVk}[1]{ \widetilde{\V}}
\newcommand{\w}{\mathbf{w}}
\newcommand{\Xk}[1]{\X^{(#1)}}
\newcommand{\Wk}[1]{W_{\Xk{#1},\varepsilon_{#1},\delta_{#1}}}
\newcommand{\Wlr}[1]{W_{\Xk{#1},\varepsilon_{#1}}^{\textit{lr}}}
\newcommand{\Wsp}[1]{\mathbf{W}_{\Xk{#1},\delta_{#1}}^{\textit{sp}}}
\newcommand{\Wspd}{\mathbf{W}_{\X,\delta}^{\textit{sp}}}
\newcommand{\Wlre}{W_{\X,\varepsilon}^{\textit{lr}}}

\newcommand{\W}{\mathbf{W}}

\newcommand{\X}{\mathbf{X}}
\newcommand{\y}{\mathbf{y}}

\newcommand{\z}{\mathbf{z}}
\newcommand{\Z}{\mathbf{Z}}
\newcommand\f[1]{\mathbf{#1}}

\newcommand\hk{^{(k)}}
\newcommand\hkk{^{(k+1)}}

\renewcommand{\epsilon}{\varepsilon}
\newcommand{\0}{\boldsymbol{0}}

\newcommand{\SIGMA}{\boldsymbol{\Sigma}}
\newcommand{\SIGMAepsk}[1]{\SIGMA_{\varepsilon_{#1}}}
\newcommand{\tSIGk}[1]{\SIGMA}

\newcommand{\XI}{\boldsymbol{\Xi}}
\newcommand{\XIk}[1]{\XI^{(#1)}}
\DeclareMathOperator*{\argmin}{arg\,min}

\renewcommand{\H}{\mathsf{H}}
\newcommand{\T}{\mathsf{T}}
\newcommand{\R}{\mathbb{R}}
\newcommand{\C}{\mathbb{C}}
\newcommand{\Rnn}{\R^{n_1 \times n_2}}
\newcommand{\norm}[2]{\left\| #1 \right\|_{#2}}

\newcommand{\round}[1]{\left( #1 \right)}
\renewcommand{\square}[1]{\left[ #1 \right]}

\begin{document}

\maketitle

\begin{abstract}
We propose a new algorithm for the problem of recovering data that adheres to multiple, heterogeneous low-dimensional structures from linear observations. Focusing on data matrices that are simultaneously row-sparse and low-rank, we propose and analyze an iteratively reweighted least squares (\texttt{IRLS}) algorithm that is able to leverage both structures. In particular, it optimizes a combination of non-convex surrogates for row-sparsity and rank, a balancing of which is built into the algorithm. We prove locally quadratic convergence of the iterates to a simultaneously structured data matrix in a regime of minimal sample complexity (up to constants and a logarithmic factor), which is known to be impossible for a combination of convex surrogates. In experiments, we show that the \texttt{IRLS} method exhibits favorable empirical convergence, identifying simultaneously row-sparse and low-rank matrices from fewer measurements than state-of-the-art methods. Code is available at \url{https://github.com/ckuemmerle/simirls}.
\end{abstract}

\section{Introduction}
\label{sec:Introduction}

Reconstructing an image from (noisy) linear observations is maybe the most relevant inverse problem for modern image processing and appears in various applications like medical imaging and astronomy \cite{bertero2021introduction}. If the latent image is $n$-dimensional, for $n \in \mathbb N$, it is well-known that $\Omega(n)$ observations are required for robust identification in general.
In practice, imaging problems are however often ill-posed, i.e., the number of observations is smaller than $n$ or the operator creating the observations is defective \cite{tikhonov1963solution,vogel2002computational}.
In such situations, the fundamental lower bound of $\Omega(n)$ can be relaxed by leveraging structural priors of the latent image in the reconstruction process.

Of the various priors that are used for solving ill-posed inverse problems in the literature, sparsity\footnote{A vector $\mathbf x \in \R^n$ is called $s$-sparse if $\mathbf x$ has at most $s$ non-zero entries. For a matrix $\X \in \R^{n_1 \times n_2}$ there are various ways to count the level of sparsity. In this work, we use the most common definition and call $\X$ $s$-row-sparse (resp.\ -column-sparse) if at most $s$ rows (resp.\ columns) are non-zero.} and low-rankness are most prevalent. This prominent role can be explained with their competitive performance in imaging tasks and the rigorous mathematical analysis they allow \cite{bruckstein2009sparse,mairal2014sparse}.
For instance, consider the recovery of an $n_1 \times n_2$-dimensional image $\X_\star \in \R^{n_1 \times n_2}$ from linear observations
\begin{align} \label{eq:inverse:problem:linear:measurements}
    \y = \A(\X_\star) + \boldsymbol\eta \in \R^m,
\end{align}
where $\A \colon \R^{n_1\times n_2} \to \R^m$ is a linear operator modeling the impulse response of the sensing device and $\boldsymbol\eta \in \R^m$ models additive noise. Whereas this problem is ill-posed for $m < n_1 n_2$, it has been established \cite{foucart:2013,CandesPlan11,recht2010} that it becomes well-posed if $\X_\star$ is sparse or of low rank. The aforementioned works prove that $m = \Omega(s_1 s_2)$ observations suffice for robust reconstruction if $\X_\star$ is $s_1$-row-sparse and $s_2$-column-sparse, and that $m = \Omega(r (n_1 + n_2))$ observations suffice if $\X_\star$ is a rank-$r$ matrix. These bounds, which relax the general lower bound of $m = \Omega(n_1 n_2)$, agree with the degrees of freedom of sparse and low-rank matrices, respectively.

A number of computationally challenging problems in signal processing and machine learning can be formulated as instances of \cref{eq:inverse:problem:linear:measurements} with $\X_\star$ being \emph{simultaneously structured}, i.e., $\X_\star$ is both of rank $r$ and $s_1$-row-sparse/$s_2$-column-sparse. Examples encompass sparse phase retrieval \cite{klibanov1995phase,Jaganathan-Sparse2017,Cai-OptimalRates2016,jagatap2017fast,Soltanolkotabi-Structured2019}, sparse blind deconvolution \cite{lee2016blind,Shi-ManifoldBlindDeconvolution2021}, hyperspectral imaging \cite{haeffele2019structured,GolbabaeeVanderghenyst-2012,Tsinos-Distributed2017,ZhangHuangLiZhangYin-Hyperspectral2022}, sparse reduced-rank regression \cite{Chen-Sparse2012,Yu-SimultaneouslyTwoWay2020}, and graph denoising and refinement \cite{RichardSavalle-Simultaneously2012,Zhang-GraphRefinement2022}. In these settings, the hope is that due to leveraging the simultaneous structure, $\Omega(r(s_1 + s_2))$ observations suffice to identify the data matrix. For $r \ll s_1,s_2 \ll n_1,n_2$, these bounds are significantly smaller than the bounds for single-structured data matrices.

From an algorithmic point of view, however, the simultaneously structured recovery problem poses obstacles that are \emph{not} present for problems where $\X_\star$ is only of low-rank \emph{or} (group) sparse: In the latter case, variational methods \cite{benning2018modern} that formulate the reconstruction method in terms of optimizing a suitable objective with a structural regularization term involving $\ell_p$/$\ell_{2,p}$-(quasi-)norms and and $S_p$-Schatten (quasi-)norms have been well-understood, leading to tractable algorithms in the information theoretically optimal regime \cite{donoho2006compressed,chartrand2007exact,foucart:2013}. \\
For simultaneously structured problems, on the other hand, Oymak et al.\ showed in \cite{Oymak2015} that a mere linear combination of \emph{convex} regularizers for different sparsity structures --- in our case, nuclear and $\ell_{2,1}$-norms --- cannot outperform recovery guarantees of the ``best'' one of them alone. While this indicates that leveraging two priors at once is a way more intricate problem than leveraging a single prior, it was also shown in \cite{Oymak2015} that minimizing a linear combination of rank and row sparsity \emph{can} indeed lead to guaranteed recovery from $\Omega(r(s_1 + s_2))$ measurements. 
The downside is that the combination of these \emph{non-convex and discontinuous} quantities does not lend itself directly to practical optimization algorithms, and to the best of our knowledge, so far, there have been no works directly tackling the optimization of a combination of \emph{non-convex surrogates} that come with any sort of convergence guarantees.

\subsection{Contribution}
In this work, we approach the reconstruction of simultaneously sparse and low-rank matrices by leveraging the positive results of \cite{Oymak2015} for non-convex regularizers. To this end, we introduce a family of non-convex, but continuously differentiable regularizers that are tailored to the recovery problem for simultaneously structured data. The resulting objectives lend themselves to efficient optimization by a novel algorithm from the class of \emph{iteratively reweighted least squares (\texttt{IRLS})} \cite{Daubechies10,forawa11,Mohan12,Beck-JOTA2015,Adil-NeurIPS2019,KM21,KMVS-NeurIPS2021}, the convergence of which we analyze in the information theoretically (near-)optimal regime. Specifically, our main contributions are threefold:

\begin{enumerate}
    \item[(i)] In \Cref{algo:MatrixIRLS}, we propose a novel \texttt{IRLS} method that is tailored to leveraging both structures of the latent solution, sparsity and low-rankness, at once. The core components of the algorithm are the weight operator defined in \Cref{def:weight:operator} and the update of the smoothing parameters in \cref{eq:MatrixIRLS:epsdef}. Notably, the algorithm automatically balances between its low-rank and its sparsity promoting terms, leading to a reliable identification of $s_1$-row-sparse and rank-$r$ ground truths\footnote{To enhance readability of the presented proofs and results, we restrict ourselves to row-sparsity of $\X_\star$ here. It is straight-forward to generalize the arguments to column-sparsity as well.}.
    \item[(ii)] Under the assumption that $\A$ behaves almost isometrically on the set of row-sparse and low-rank matrices, we show in \Cref{thm:QuadraticConvergence} that locally \Cref{algo:MatrixIRLS} exhibits quadratic convergence towards $\X_\star$. Note that if $\A$ is, e.g., a Gaussian operator, the isometry assumption is (up to log factors) fulfilled in the information theoretic (near-)optimal regime $m = \Omega(r (s_1 + n_2))$ \cite{lee2013near}. 
    \item[(iii)] Finally, in \Cref{sec:IRLSasMMAlgo} we identify the underlying family of objectives that are minimized by \Cref{algo:MatrixIRLS}.
    To make this precise, we define for $\tau > 0$ and $e$ denoting Euler's number the real-valued function $f_\tau : \mathbb R \to \mathbb R$ such that
    \begin{equation} \label{eq:f_tau:definition}
    \begin{split}
        f_\tau (t) = \begin{cases}
            \frac{1}{2}\tau^2 \log(e t^2/\tau^2), & \text{ if } |t| > \tau, \\
            \frac{1}{2} t^2, &  \text{ if } |t| \leq \tau,
        \end{cases},
        \end{split}
    \end{equation}
    which is quadratic around the origin and otherwise logarithmic in its argument. Using this definition, we define for $\varepsilon > 0$ the \emph{$(\varepsilon-)$smoothed $\log$-determinant} objective $\Flr: \Rnn \to \mathbb R$ and for $\delta > 0$ the \emph{$(\delta-$)smoothed sum of logarithmic row-wise $\ell_2$-norms} objective $\Fsp: \Rnn \to \mathbb R$ such that
    \begin{align} \label{eq:def:objectives}
        \Flr(\X) = \sum_{r=1}^{\min\{n_1,n_2\}} f_\varepsilon (\sigma_r(\X)),
        \qquad \qquad 
        \Fsp(\X) = \sum_{i=1}^{n_1} f_\delta(\| \X_{i,:} \|_2).
    \end{align}
    Combining the above, we further define the \emph{$(\varepsilon, \delta-)$smoothed logarithmic surrogate} objective $\F_{\varepsilon,\delta}: \Rnn \to \mathbb R$ as
    \begin{align} \label{eq:F:objective:def}
        \F_{\varepsilon,\delta}(\X) := \Flr (\X) + \Fsp (\X).
    \end{align}
    In \Cref{thm:IRLS:majorization:MM}, we prove for \emph{any} $\A$ that the iterates of \Cref{algo:MatrixIRLS} minimize quadratic majorizations of $\F_{\varepsilon,\delta}$ and form a non-increasing sequence on $\F_{\varepsilon,\delta}$. To the best of our knowledge, the proposed method is the so far only approach for recovering simultaneously sparse and low-rank matrices which combines local (quadratic) convergence with a rigorous variational interpretation.
\end{enumerate}

The numerical simulations in \Cref{sec:Numerics} support our theoretical findings and provide empirical evidence for the efficacy of the proposed method.

\subsection{Related Work}
In this section, we review literature relevant for our work.
\paragraph{Sparse and Low-Rank Recovery} Whereas leveraging a single matrix structure like sparsity \emph{or} low-rankness in the reconstruction process can easily be obtained by convex regularizers \cite{recht2010,Chandrasekaran-Convex2012}, Oymak et al.\ \cite{Oymak2015} showed that, if one is interested in near-optimal sampling rates, one cannot expect comparably simple solutions for identifying simultaneously structured objects; a minimization of \cref{eq:F:objective:def} with \emph{convex} terms $\Flr(\cdot)$ and $\Fsp(\cdot)$ would be only as good as using the one structure that is information theoretically more favorable. A closely related problem\footnote{Since observations in SPCA are provided from noisy samples of the underlying distribution, whereas in our case the matrix itself is observed indirectly, it is however hard to directly compare results from sparse and low-rank matrix reconstruction with corresponding results for SPCA.} that appears in statistical literature under the name \emph{Sparse Principal Component Analysis} (SPCA) \cite{ZouHastieTibshirani-2006,d2005direct} is known to be NP-hard in general \cite{magdon2017np}. Despite the the intrinsic hardness of simultaneously structured recovery problems, promising empirical results for hyperspectral image demixing were shown in \cite{Giampouras-2016Simultaneously}, where minimization problems involving the sum of \emph{reweighted} convex surrogates are solved by a proximal scheme based on ADMM for the case of simultaneously sparse, low-rank and non-negative matrices. \\
For the problem of simultaneously sparse and low-rank matrix recovery, there exist only a handful approaches that come with rigorous theoretical analysis. The first line of works \cite{bahmani2016near,foucart2019jointly} aims to overcome the aforementioned limitations of purely convex methods in a neat way. They assume that the operator $\A$ has a nested structure such that basic solvers for low-rank resp.\ row/column-sparse recovery can be applied in two consecutive steps. Despite being an elegant idea, this approach clearly restricts possible choices for $\A$ and is of hardly any practical use. \\
In a second line of work, Lee et al.\ \cite{lee2013near} consider general impulse response operators that satisfy a suitable restricted isometry property for $s_1$-row- and $s_2$-column-sparse rank-$r$ matrices. They propose and analyze a highly efficient, greedy method, the so-called {\it Sparse Power Factorization} (SPF) which is a modified version of power factorization \cite{jain2013low} and uses hard thresholding pursuit \cite{foucart2011hard} to enforce sparsity in addition. In particular, they show that if $\X_\star$ is rank-$R$, has $s_1$-sparse columns and $s_2$-sparse rows, then $m \gtrsim R(s_1 + s_2) \log (\max \{en_1/s_1, en_2/s_2\})$ Gaussian observations suffice for robust recovery,  which is up to the log-factor at the information theoretical limit we discussed above. The result however assumes a low noise level and requires that SPF is initialized by a, in general, intractable method. Only in the special case that $\X_\star$ is spiky, which means that the norms of non-zero rows/columns exhibit a fast decay, a tractable substitute for the initialization method works provably. The analysis of SPF has been extended to the blind deconvolution setup in \cite{lee2016blind}. In \cite{Yu-SimultaneouslyTwoWay2020}, a related approach that combines gradient descent of a smooth objective with hard thresholding is considered, for which the authors show linear convergence from a suitable intialization if the measurement operator satisfies a restricted strong convexity and smoothness assumption.
\\
A third line of work, approaches the problem from a variational point of view. In \cite{fornasier2018robust,maly2021robust} the authors aim at enhancing robustness of recovery by alternating minimization of an $\ell_1$-norm based multi-penalty functional. In essence, the theoretical results bound the reconstruction error of global minimizers of the proposed functional depending on the number of observations. Although the authors only provide local convergence guarantees for the proposed alternating methods, the theoretical error bounds for global minimizers hold for arbitrarily large noise magnitudes and a wider class of ground-truth matrices than the one considered in \cite{lee2013near}. \\
The works \cite{foucart2019jointly,eisenmann2021riemannian}, which build upon generalized projection operators to modify iterative hard thresholding to the simultaneous setting, share the lack of global convergence guarantees. \\
 In \cite{Richard2014tight}, the authors examine the use of atomic norms to perform recovery of simultaneously sparse and low-rank matrices, which uses a related, but different sparsity assumption compared to the row or column sparsity studied here. From a practical point of view, the such norms are hard to compute and the paper only proposes a heuristic polynomial time algorithm for the problem. \\
Finally, the alternative approach of using optimally weighted sums or maxima of convex regularizers \cite{kliesch2019simultaneous} requires optimal tuning of the parameters under knowledge of the ground-truth.

\paragraph{Iteratively Reweighted Least Squares} The herein proposed iteratively reweighted least squares algorithm builds on a long line of research on \texttt{IRLS} going back to Weiszfeld's algorithm proposed in the 1930s for a facility location problem \cite{Weiszfeld37,Beck-JOTA2015}. \texttt{IRLS} is a practical framework for the optimization of non-smooth, possibly non-convex, high-dimensional objectives that minimizes quadratic models which majorize these objectives. Due to its ease of implementation and favorable data-efficiency, it has been widely used in compressed sensing \cite{GorodnitskyRao-1997,chartrand_yin,Daubechies10,Lai-SIAM-J-NA2013,FornasierPeterRauhutWorm-2016,KMVS-NeurIPS2021}, robust statistics \cite{Holland-1977,Aftab-WCACV2015,Mukhoty-AISTATS2019}, computer vision \cite{Chatterjee-RobustRotation2017,Lee-HARA2022,Shi-Robust2022}, low-rank matrix recovery and completion \cite{forawa11,Mohan12,Kummerle-JMLR2018,KM21}, and in inverse problems involving group sparsity \cite{Chen-Preconditioning2014,Zeinalkhani-IRLS2015,Chen-FastAnalysisIRLS2018}. Recently, it has been shown \cite{Lefkimmiatis-LearningSparseLowRankPriors2023} that dictionary learning techniques can be incorporated into \texttt{IRLS} schemes for sparse and low-rank recovery to allow the learning of a sparsifying dictionary while recovering the solution.
Whereas \texttt{IRLS} can be considered as a type of majorize-minimize algorithm \cite{lange2016mm}, optimal performance is achieved if intertwined with a smoothing strategy for the original objective, in which case globally linear (for convex objectives) \cite{Daubechies10,Adil-NeurIPS2019,Mukhoty-AISTATS2019,KMVS-NeurIPS2021,PKV22} and locally superlinear (for non-convex objectives) \cite{Daubechies10,Kummerle-JMLR2018,KM21,PKV22} convergence rates have been shown under suitable conditions on the linear operator $\A$.

However, there has only been little work on \texttt{IRLS} optimizing a sum of heterogenous objectives \cite{Samejima-GeneralNormIRLS2017} --- including the combination of low-rank promoting and sparsity-promoting objectives --- nor on the convergence analysis of any such methods. The sole algorithmically related approach for our setting has been studied in \cite{Chen-Simultaneously2018}, where a method has been derived in a sparse Bayesian learning framework, the main step of which amounts to the minimization of weighted least squares problems. Whereas the algorithm of \cite{Chen-Simultaneously2018} showcases that such a method can empiricially identify simultaneously structured matrices from a small number of measurements, no convergence guarantees or rates have been provided in the information-theoretically optimal regime. Furthermore, \cite{Chen-Simultaneously2018} only focuses on general sparsity rather than row or column sparsity.

\subsection{Notation}

We denote matrices and vectors by bold upper- and lower-case letters to distinguish them from scalars and functions. We furthermore denote the $i$-th row of a matrix $\Z \in \R^{n_1\times n_2}$ by $\Z_{i,:}$ and the $j$-th column of $\Z$ by $\Z_{:,j}$.
We abbreviate $n = \min\{n_1,n_2\}$.
We denote the $r$-th singular value of a matrix $\Z \in \R^{n_1\times n_2}$ by $\sigma_r(\Z)$. Likewise, we denote the in $\ell_2$-norm $s$-largest row of $\Z$ by $\rho_s(\Z)$. (To determine $\rho_s(\Z)$, we form the in $\ell_2$-norm non-increasing rearrangement of the rows of $\Z$ and, by convention, sort rows with equal norm according to the row-index.)
We use $\circ$ to denote the Hadamard (or Schur) product, i.e., the entry-wise product of two vectors/matrices. \\
We denote the Euclidean $\ell_2$-norm of a vector $\z \in \R^n$ by $\| \z \|_2$.
For $\Z \in \R^{n_1\times n_2}$, the matrix norms we use encompass the operator norm $\| \Z \|~:=~\sup_{\| \w \|_2 = 1} \| \Z \w \|_2$, the row-sum $p$-quasinorm $\| \Z \|_{p,2}~:=~ \big(\sum_{i = 1}^{n_1} \| \Z_{i,:} \|_2^p\big)^{1/p}$, the row-max norm $\| \Z \|_{\infty,2}~:=~\max_{i \in [n_1]} \| \Z_{i,:} \|_2$, and the Schatten-p quasinorm $\| \Z \|_{S_p}~:=~\left( \sum_{r=1}^{n} \sigma_r(\Z)^p \right)^{\nicefrac{1}{p}}$. Note that two special cases of Schatten quasinorms are the nuclear norm $\| \Z \|_*~:=~\| \Z \|_{S_1}$ and the Frobenius norm $\| \Z \|_F~:=~\| \Z \|_{S_2}$.

\section{\texttt{IRLS} for Sparse and Low-Rank Reconstruction}

Recall that we are interested in recovering a rank-$r$ and $s$-row-sparse matrix $\X_\star \in \R^{n_1\times n_2}$ from $m$ linear observations
\begin{align}
\label{eq:CS}
    \y = \A(\X_\star)  \in \R^m,
\end{align}
i.e., $\A \colon \R^{n_1\times n_2} \to \R^m$ is linear. We write $\X_\star \in \MM{r,s}^{n_1,n_2} := \MM{r}^{n_1,n_2} \cap \NN{s}^{n_1,n_2}$, 
where $\MM{r}^{n_1,n_2} \subset \R^{n_1\times n_2}$ denotes the set of matrices with rank at most $r$ and $\NN{s}^{n_1,n_2} \subset \R^{n_1\times n_2}$ denotes the set of matrices with at most $s$ non-zero rows. For convenience, we suppress the indices $n_1$ and $n_2$ whenever the ambient dimension is clear from the context. In particular, we know that $\X_\star = \U_\star \SIGMA_\star \V_\star^*$,
where $\U_\star \in \NN{s}^{n_1,r}$, $\SIGMA_\star \in \R^{r\times r}$, and $\V_\star \in \R^{n_2\times r}$ denote the reduced SVD of $\X_\star$. Furthermore, the row supports of $\X_\star$ and $\U_\star$ (the index sets of non-zero rows of $\X_\star$ resp. $\U_\star$) are identical, i.e., $\supp(\X_\star) = \supp(\U_\star) = S_\star \subset [n_1] := \{1,\dots,n_1\}$.

\subsection{How to Combine Sparse and Low-Rank Weighting}

As discussed in \Cref{sec:Introduction}, the challenge in designing a reconstruction method for \cref{eq:CS} lies in simultaneously leveraging both structures of $\X_\star$ in order to achieve the optimal sample complexity of $m \approx r(s+n_2)$. To this end, we propose a novel \texttt{IRLS}-based approach in \Cref{algo:MatrixIRLS}.
The key ingredient for computing the $(k+1)$-st iterate $\Xk{k+1}  \in \Rnn$ in \Cref{algo:MatrixIRLS} is the multi-structural weight operator $\Wk{k}: \Rnn \to \Rnn$ of \Cref{def:weight:operator}, which depends on the current iterate $\Xk{k}$.
\vspace{1mm}
\begin{center}
\fbox{
\begin{minipage}[]{16.0cm}
\begin{definition} \label{def:weight:operator}
    For $\delta_k ,\varepsilon_k > 0$, $\sigma_i^k: = \sigma_i(\Xk{k})$, and $\Xk{k} \in \Rnn$, let 
        \begin{equation} \label{def:rk}
    r_k:= |\{i \in [n]: \sigma_i\hk > \varepsilon_k\}|
    \end{equation} 
     denote the number of singular values $\sigma^{(k)} = \big(\sigma_i^{(k)}\big)_{i=1}^{r_k}$ of $\Xk{k}$ larger than $\varepsilon_k$ and furthermore be
        \begin{equation} \label{eq:def:epsilon} 
        s_k : = \big|\big\{i \in [n_1]: 
     \|\f{X}_{i,:}\hk\|_2 > \delta_k \big\}\big|
     \end{equation} 
   the number of rows of $\Xk{k}$ with $\ell_2$-norm larger than $\delta_k$. Define the $r_k$ left and right singular vectors of $\Xk{k}$ as columns of $\U \in \ R^{n_1 \times r_k}$ and $\V \in \R^{n_2 \times r_k}$, respectively, corresponding to the leading singular values $\boldsymbol \sigma^{(k)} = \big(\sigma_i^{(k)}\big)_{i=1}^{r_k} : = \left(\sigma_i(\Xk{k})\right)_{i=1}^{r_k}$.
    
    We define the \emph{weight operator} $\Wk{k}: \Rnn \to \Rnn$ at iteration $k$ of \Cref{algo:MatrixIRLS} as
    \begin{align} \label{eq:DefW}
    \Wk{k}(\Z) = \Wlr{k}(\Z)
    + \Wsp{k} \Z
    \end{align}
where $\Wlr{k}: \Rnn \to \Rnn$ is its low-rank promoting part 
\begin{equation} \label{eq:W:operator:action}
\begin{split}
    \Wlr{k}(\Z) = \begin{bmatrix}
    \U & \U_\perp 
    \end{bmatrix}  \SIGMAepsk{k}^{-1} \begin{bmatrix}
    \U^* \\ \U_\perp^* 
    \end{bmatrix}  \Z      \begin{bmatrix}
    \V & \V_\perp 
    \end{bmatrix} \SIGMAepsk{k}^{-1}   \begin{bmatrix}
    \V^* \\ \V_\perp^* 
    \end{bmatrix} 
    \end{split}
\end{equation}
with $\SIGMAepsk{k} = \max( \sigma_i^{(k)}/\varepsilon_k, 1 )$ and $\Wsp{k} \in \R^{n_1 \times n_1}$ is its sparsity-promoting part, which is diagonal with
\begin{align} \label{eq:Ws_def}
    \left(\Wsp{k}\right)_{ii} = \max \left( \big\|(\Xk{k})_{i,:}\big\|_2^2/\delta_k^2, 1 \right)^{-1}, \qquad \text{for all $i \in [n_1]$}.
\end{align}
The matrices $\mathbf{U}_\perp$ and $\mathbf{V}_\perp$ are arbitrary complementary orthogonal bases for $\mathbf{U}$ and $\mathbf{V}$ that do  do not need to be computed in \Cref{algo:MatrixIRLS}.
\end{definition}
\end{minipage}
}
\end{center}
\vspace{1mm}

\begin{algorithm}[t]
\caption{\texttt{IRLS} for simultaneously low-rank rand row-sparse matrices}\label{algo:MatrixIRLS}
\begin{algorithmic}[1]
\STATE{\bfseries Input:} Linear operator $\A \colon \R^{n_1\times n_2} \to \R^m$, data $\f{y} \in \R^m$, rank and sparsity estimates $\widetilde{r}$ and $\widetilde{s}$.
\STATE Initialize $k=0$, $\Wk{0} = \id$ and set $\delta_k,\varepsilon_k=\infty$. 
\FOR{$k=1$ to $K$}
\STATE \textbf{Weighted Least Squares:} Update iterate $\Xk{k}$ by
\begin{equation} \label{eq:MatrixIRLS:Xdef}
\f{X}^{(k)} =\argmin\limits_{\f{X}:\A(\f{X})=\f{y}} \langle \f{X}, \Wk{k-1}(\f{X}) \rangle.
\end{equation}
\STATE \textbf{Update Smoothing :} \label{eq:MatrixIRLS:bestapprox}
Compute $\widetilde{r}+1$-st singular value $\sigma_{\widetilde r+1}\hk :=\sigma_{\widetilde r+1}(\Xk{k})$ and $(\widetilde{s}+1)$-st largest row $\ell_2$-norm $\rho_{\widetilde s+1} (\Xk{k})$ of $\Xk{k}$, update
\begin{equation} \label{eq:MatrixIRLS:epsdef}
\varepsilon_k = \min\left(\varepsilon_{k-1}, \sigma_{\widetilde s+1}\hk\right), \quad \delta_k=\min\left(\delta_{k-1}, \rho_{\widetilde{s}+1}(\Xk{k})\right).
\end{equation}

\STATE \textbf{Update Weight Operator:} For $r_k$ and $s_k$ as in \cref{def:rk} and \cref{eq:def:epsilon},
\begin{itemize}
    \item compute first $r_k$ singular triplets $\sigma\hk \in \R^{r_k}$, $\U \in \R^{n_1 \times r_k}$ and $\V \in \R^{n_2 \times r_k}$,
    \item compute $\Wk{k}$ in \cref{eq:DefW} via $\Wlr{k}$ in \cref{eq:W:operator:action} and $\Wsp{k}$ in \cref{eq:Ws_def}.
\end{itemize}
\ENDFOR
\STATE{\bfseries Output:} $\f{X}^{(K)}$.
\end{algorithmic}
\end{algorithm}
\begin{remark}
    For the sake of conciseness, we only consider row-sparsity here. Algorithm \ref{algo:MatrixIRLS} and its analysis can however be modified to cover row- and column-sparse matrices as well. For instance, in the symmetric setting $\mathbf X_\star = \mathbf X_\star^T$ (naturally occurring in applications like sparse phase retrieval) one would define the weight operator $W_{\X^{(k)},\varepsilon_{k},\delta_{k}}$ as in \eqref{eq:DefW}, but with an additional term that multiplies $\W_{\X^{(k)},\delta_{k}}^{sp}$ from the right to $\Z$, which corresponds to minimizing the sum of three smoothed logarithmic surrogates. In this case, the solving modified weighted least squares problem \cref{eq:MatrixIRLS:Xdef} will have similar complexity (potentially smaller complexity, as additional symmetries can be exploited).
\end{remark}

Recall $\F_{\varepsilon_k,\delta_k}$, $\Flrk{k}$, and $\Fspk{k}$ from \eqref{eq:def:objectives}-\eqref{eq:F:objective:def}. The high-level idea of \Cref{algo:MatrixIRLS}, as for other \texttt{IRLS} methods, is to minimize quadratic functionals, which we call $\Qlrk(\,\cdot\,|\Xk{k}): \Rnn \to \mathbb R$ and $\Qspk(\,\cdot\, |\Xk{k}): \Rnn \to \mathbb R$ and define them by
\begin{align}
\label{eq:def:Q}
\begin{split}
    \Qlrk(\Z|\Xk{k})\! &:=\! \Flrk{k}(\Xk{k}) + \langle \nabla \Flrk{k} (\Xk{k}), \Z - \Xk{k} \rangle + \frac{1}{2} \langle \Z - \Xk{k}, \Wlr{k} (\Z-\Xk{k}) \rangle, \\
    \Qspk(\Z|\Xk{k}) &:= \Fspk{k}(\Xk{k}) + \langle \nabla \Fspk{k}(\Xk{k}), \Z - \Xk{k} \rangle + \frac{1}{2} \langle \Z - \Xk{k}, \Wsp{k} (\Z-\Xk{k}) \rangle,
\end{split}
\end{align}
that majorize $\F_{\varepsilon_k,\delta_k}(\cdot)$ (see \Cref{thm:IRLS:majorization:MM} below) for any iteration $k$. This minimization leads to the weighted least squares problem \cref{eq:MatrixIRLS:Xdef} in \Cref{algo:MatrixIRLS}. This step can be implemented by standard numerical linear algebra (see the supplementary material for a discussion of its computational complexity). As a second ingredient of the method, the smoothing parameters $\varepsilon_k$ and $\delta_k$ of $\F_{\varepsilon_k,\delta_k}$ are updated (i.e., decreased) in step \cref{eq:MatrixIRLS:epsdef} before the weight operator is updated according to the current iterate information. For the weight operator update, it is only necessary to compute row norms and leading singular triplets of $\f{X}^{(k)}$.

\begin{remark}
    The particular form of the low-rank promoting part of the weight operator $\Wlr{k}$ in \cref{eq:W:operator:action} is due to \cite{K19,KM21} and captures optimally spectral information both in the column and row space, unlike prior work on low-rank IRLS \cite{forawa11,Mohan12}, while retaining the property that the induced quadratic model $\Qlrk(\,\cdot\,|\Xk{k})$ majorizes  $\Flrk{\varepsilon_k}(\cdot)$ (see proof of \Cref{thm:IRLS:majorization:MM}). This choice is critical to enable a fast local rate as established in \Cref{thm:QuadraticConvergence}.
\end{remark}

\subsection{Local Quadratic Convergence of \texttt{IRLS}}

Our first main result states that \Cref{algo:MatrixIRLS} exhibits quadratic convergence in a local neighborhood of $\X_\star$, a property \Cref{algo:MatrixIRLS} shares with several methods from the \texttt{IRLS} family. We only need to assume that $\A$ acts almost isometrically on the set $\MM{r,s}$.

\begin{definition} \label{def:RIP}
    We say that a linear operator $\A \colon \R^{n_1 \times n_2} \to \R^m$ satisfies the rank-$r$ and row-$s$-sparse restricted isometry property (or $(r,s)$-RIP) with RIP-constant $\delta \in (0,1)$ if
    \begin{align*}
        (1-\delta) \norm{\Z}{F}^2 \le \norm{\A(\Z)}{2}^2 \le (1+\delta) \norm{\Z}{F}^2,
    \end{align*}
    for all $\Z \in \MM{r,s}$.
\end{definition}

It is worth highlighting that Gaussian operators satisfy the above RIP with high-probability if $m \ge c r(s+n_2) \log(en_1/s)$, for some absolute constant $c > 0$, see for instance \cite{lee2013near}. Up to log-factors, this is at the information theoretic limit which we discussed in the beginning. The convergence result for \Cref{algo:MatrixIRLS} now reads as follows.

\begin{theorem}[Local Quadratic Convergence]
\label{thm:QuadraticConvergence}
    Let $\X_\star \in \MM{r,s}$ be a fixed ground-truth matrix that is $s$-row-sparse and of rank $r$. Let linear observations $\y = \A(\X_\star)$ be given and assume that $\A$ has the $(r,s)$-RIP with $\delta \in (0,1)$.  Assume that the $k$-th iterate $\Xk{k}$ of Algorithm \ref{algo:MatrixIRLS} with $\widetilde{r} = r$ and $\widetilde{s} = s$ updates the smoothing parameters in \cref{eq:MatrixIRLS:epsdef} such that one of the statements  $\varepsilon_k = \sigma_{r+1} (\Xk{k})$ or $\delta_k =\rho_{s+1}(\Xk{k}) $ is true, and that $r_k \geq r$ and $s_k \geq s$.
	If $\Xk{k}$ satisfies 
    \begin{align} \label{eq:closeness:assumption:mainthm_Simple}
        \| \Xk{k} - \X_\star \|
        &\le \frac{1}{48 \sqrt{n} c_{\norm{\A}{2\to 2}}^3} \min \left\{ \frac{\sigma_r(\X_\star)}{r}, \frac{\rho_s(\X_\star)}{s} \right\}
    \end{align}
    where $c_{\norm{\A}{2\to 2}} = \sqrt{1 + \tfrac{\norm{\A}{2\to 2}^2}{(1-\delta)}}$ and $n = \min\{n_1,n_2\}$,
    then the local convergence rate is quadratic in the sense that 
    \begin{align*}
        \| \Xk{k+1} - \X_\star \| \le \min\{ \mu \| \Xk{k} - \X_\star \|^2, 0.9 \| \Xk{k} - \X_\star \| \},
    \end{align*}
    for
    \begin{align} \label{eq:mu:def}
        \mu = 4.179 c_{\norm{\A}{2\to 2}}^2 \Big( \frac{5r}{\sigma_r(\X_\star)} + \frac{2s}{\rho_{s}(\X_\star)} \Big),
    \end{align}  
    and $\Xk{k+\ell} \overset{\ell \to \infty}{\to} \X_\star$.
\end{theorem}

The proof of \Cref{thm:QuadraticConvergence} is presented in the supplementary material. To the best of our knowledge, so far no other method exists for recovering simultaneously sparse and low-rank matrices that exhibits local quadratic convergence. In particular, the state-of-the-art competitor methods \cite{lee2013near,foucart2019jointly,maly2021robust,eisenmann2021riemannian} reach a local linear error decay at best.\\
On the other hand, \cref{eq:closeness:assumption:mainthm_Simple} is rather pessimistic since for Gaussian $\A$ the constant $c_{\norm{\A}{2\to 2}}$ scales like $\sqrt{(n_1n_2)/m}$, which means that the right-hand side of \cref{eq:closeness:assumption:mainthm_Simple} behaves like $m^{3/2}/(n (n_1n_2)^{3/2})$, whereas we observe quadratic convergence in experiments within an empirically much larger convergence radius. Closing this gap between theory and practical performance is future work. \\
It is noteworthy that the theory in \cite{lee2013near} --- to our knowledge the only other related work explicitly characterizing the convergence radius --- holds on a neighborhood of $\X_\star$ that is independent of the ambient dimension. 
The authors of \cite{lee2013near} however assume that the RIP-constant decays with the conditioning number $\kappa$ of $\A$, a quantity that might be large in applications. Hence, ignoring log-factors the sufficient number of measurements in \cite{lee2013near} scales like $m=\Omega(\kappa^2 r (s_1+n_2))$. In contrast, \Cref{thm:QuadraticConvergence} works for any RIP-constant less than one which means for $m=\Omega(r (s_1+n_2))$.

\subsection{\texttt{IRLS} as Quadratic Majorize-Minimize Algorithm}
\label{sec:IRLSasMMAlgo}
With \Cref{thm:QuadraticConvergence}, we have provided a local convergence theorem that quantifies the behavior of \Cref{algo:MatrixIRLS} in a small neighbourhood of the simultaneously row-sparse and low-rank ground-truth $\X_\star \in \Rnn$. The result is based on sufficient regularity of the measurement operator $\A$, which in turn is satisfied with high probability if $\A$ consists of sufficiently generic random linear observations that concentrate around their mean.\\
In this section we establish that, for \emph{any} measurement operator $\A$, \Cref{algo:MatrixIRLS} can be interpreted within the framework of iteratively reweighted least squares (\texttt{IRLS}) algorithms \cite{Daubechies10,Mohan12,PKV22}, which implies a strong connection to the minimization of a suitable smoothened objective function. In our case, the objective $\F_{\varepsilon,\delta}$ in \cref{eq:F:objective:def} is a linear combination of sum-of-logarithms terms penalizing both non-zero singular values \cite{Fazel2003log,Mohan12,Candes13} as well as non-zero rows of a matrix $\X$ \cite{Ke2021-iterativelysumoflog}. \\
We show in \Cref{thm:IRLS:majorization:MM} below that the \texttt{IRLS} algorithm \Cref{algo:MatrixIRLS} studied in this paper is based on minimizing at each iteration quadratic models that majorize $\F_{\varepsilon,\delta}$, and furthermore, that the iterates $(\Xk{k})_{k\geq 1}$ of \Cref{algo:MatrixIRLS} define a non-increasing sequence $\left(\F_{\varepsilon_k,\delta_k}(\Xk{k})\right)_{k \geq 1}$ with respect to the objective $\F_{\varepsilon,\delta}$ of \cref{eq:F:objective:def}. The proof combines the fact that for fixed smoothing parameters $\varepsilon_k$ and $\delta_k$, the weighted least squares and weight update steps \Cref{algo:MatrixIRLS} can be interpreted as a step of a Majorize-Minimize algorithm \cite{Sun-TSP2016,lange2016mm}, with a decrease in the underlying objective \cref{eq:F:objective:def} for updated smoothing parameters.

\begin{theorem}\label{thm:IRLS:majorization:MM}
Let $\f{y} \in \R^m$, let the linear operator $\A \colon \R^{n_1\times n_2} \to \R^m$ be arbitrary. If $(\Xk{k})_{k\geq 1}$ is a sequence of iterates of \Cref{algo:MatrixIRLS} and $(\delta_k)_{k\geq 1}$ and $(\varepsilon_k)_{k\geq 1}$ are the sequences of smoothing parameters as defined therein, then the following statements hold.

\begin{enumerate}
    \item The quadratic model functions $\Qlrk(\,\cdot\,|\Xk{k})$ and $\Qspk(\,\cdot\, |\Xk{k})$ defined in \cref{eq:def:Q} globally majorize the $(\varepsilon_k, \delta_k)-$smoothed logarithmic surrogate objective $\F_{\varepsilon_k,\delta_k}$, i.e., for \emph{any} $\Z \in \Rnn$, it holds that
    \begin{equation} \label{eq:majorization:Fepsdeltak}
    \F_{\varepsilon_k,\delta_k}(\Z) \leq \Qlrk(\Z| \Xk{k}) + \Qspk(\Z|\Xk{k}).
    \end{equation}
    \item The sequence $\left(\F_{\varepsilon_k,\delta_k}(\Xk{k})\right)_{k\geq 1}$ is non-increasing.
    \item If $\overline{\varepsilon} := \lim_{k \to \infty} \varepsilon_k > 0$ and $\overline{\delta} := \lim_{k \to \infty} \delta_k > 0$, then $\lim_{k\to \infty} \|\Xk{k} - \Xk{k+1} \|_F = 0$. Furthermore, in this case, every accumulation point of $(\Xk{k})_{k\geq 1}$ is a stationary point of
    \[
    \min\limits_{\f{X}:\A(\f{X})=\f{y}} \F_{\overline{\varepsilon},\overline{\delta}}(\f{X}).
    \]
\end{enumerate}
\end{theorem}

\section{Discussion of Computational Complexity}
\label{sec:Complexity}
It is well-known that the solution of the linearly constrained weighted least squares problem \cref{eq:MatrixIRLS:Xdef} can be written as
\begin{equation} \label{eq:formula:WLS:explicit}
\f{X}^{(k)} = W_{k-1}^{-1} \A^* \left( \A W_{k-1}^{-1} \A^*\right)^{-1} \f{y}
\end{equation}
where $W_{k-1} := \Wk{k-1}$ is the weight operator \cref{eq:DefW} of iteration $k-1$
\cite{Daubechies10,KM21}. In \cite[Theorem 3.1 and Supplementary Material]{KM21}, it was shown that in the case of low-rank matrix completion without the presence of a row-sparsity inducing term, this weighted least squares problem can be solved by solving an equivalent, well-conditioned linear system via an iterative solver that uses the application of a system matrix whose matrix-vector products have time complexity of $O(m r + r^2 \max(n_1,n_2))$. 

In the case of \Cref{algo:MatrixIRLS}, the situations is slightly more involved as we cannot provide an explicit formula for the inverse of the weight operator $W_{k-1}$ as it amounts to the sum of the weight operators $\Wlr{k-1}$ and $\Wsp{k-1}$ that are diagonalized by different, mutually incompatible bases. However, computing this inverse is facilitated by the \emph{Sherman-Morrison-Woodbury} formula \cite{Woodbury50} 
\begin{equation*}
(\f{E} \f{C} \f{F}^* + \f{B})^{-1} = \f{B}^{-1} - \f{B}^{-1} \f{E} ( \f{C}^{-1} + \f{F}^* \f{B}^{-1} \f{E} )^{-1} \f{F}^* \f{B}^{-1}
\end{equation*}
for suitable matrices of compatible size $\f{E},\f{F}$ and invertible $\f{C},\f{B}$ and the fact that both $\Wlr{k-1}$ and  $\Wsp{k-1}$ exhibit a ``low-rank plus (scaled) identity`` or a ``sparse diagonal plus (scaled) identity`` structure. After a simple application of the SMW formula, \cref{eq:formula:WLS:explicit} can be rewritten such that the computational bottleneck becomes the assembly and inversion of a $O(r_k \max(n_1,n_2))$ linear system. We note that in general, this can be done exactly in a time complexity of $O(r_k^3 max(n_1,n_2)^3)$ using standard linear algebra. A crucial factor in the computational cost of the method is also the structure of the measurement operator $\mathcal{A}$ defining the problem, as the application of itself and its adjoint can significantly influence the per-iteration cost of IRLS; for dense Gaussian measurements, just processing the information of $\mathcal{A}$ amounts to $m n_1 n_2$ flops. If rank-one or Fourier-type measurements are taken, this cost can significantly be reduced, see \cite[Table 1 and Section 3]{eisenmann2021riemannian} for an analogous discussion.

We refer to the MATLAB implementation available in the repository \url{https://github.com/ckuemmerle/simirls} for further details. While our implementation is not optimized for large-scale problems, the computational cost of \Cref{algo:MatrixIRLS} was observed to be comparable to the implementations of \texttt{SPF} or \texttt{RiemAdaIHT} provided by the authors \cite{maly2021robust,eisenmann2021riemannian}. We leave further improvements and adaptations to large-scale settings to future work.

\section{Numerical Evaluation} \label{sec:Numerics}
In this section, we explore the empirical performance of \texttt{IRLS} in view of the theoretical results of \Cref{thm:QuadraticConvergence,thm:IRLS:majorization:MM}, and compare its ability to recover simultaneous low-rank and row-sparse data matrices with the state-of-the-art methods \texttt{Sparse Power Factorization (SPF)} \cite{lee2013near} and \texttt{Riemannian adaptive iterative hard thresholding} \texttt{(RiemAdaIHT)} \\\cite{eisenmann2021riemannian}, which are among the methods with the best empirical performance reported in the literature. The method \texttt{ATLAS} \cite{fornasier2018robust} and its successor \cite{maly2021robust} are not used in our empirical studies since they are tailored to robust recovery and yield suboptimal performance when seeking high-precision reconstruction in low noise scenarios. We use spectral initialization for \texttt{SPF} and \texttt{RiemAdaIHT}. The weight operator of \texttt{IRLS} is initialized by the identity as described in \Cref{algo:MatrixIRLS}, solving an unweighted least squares problem in the first iteration.

\paragraph{Performance in Low-Measurement Regime}

\Cref{fig:Performance_Rank1,fig:Performance_Rank5} show the empirical probability of successful recovery when recovering $s$-row sparse ground-truths $\X_\star \in \R^{256\times 40}$ of rank $r=1$ (resp. $r=5$) from Gaussian measurements under oracle knowledge on $r$ and $s$. The results are averaged over 64 random trials. As both figures illustrate, the region of success of \texttt{IRLS} comes closest to the information theoretic limit of $r(s+n_2-r)$ which is highlighted by a red line, requiring a significantly lower oversampling factor than the baseline methods.

\begin{figure}[h]
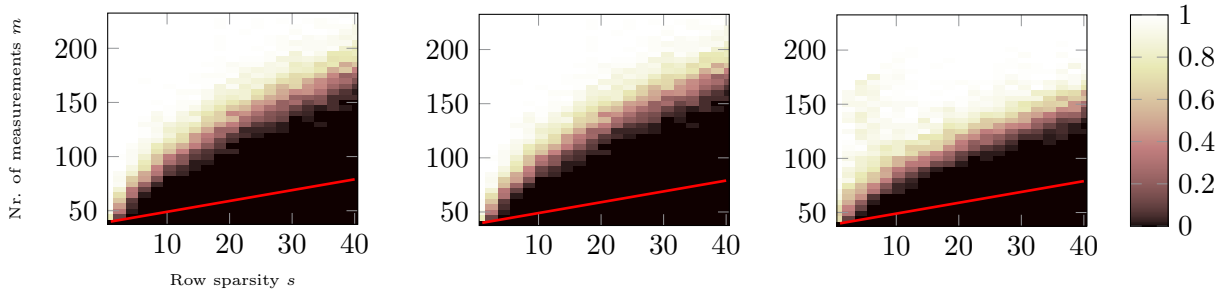

\begin{subfigure}[c]{0.32\textwidth}
    \setlength\figureheight{28mm} 
    \setlength\figurewidth{38mm}
\input{Rank1_base_RAdaIHT.tex}
\end{subfigure}
\hspace*{0mm}
\begin{subfigure}[c]{0.28\textwidth}
\vspace*{-3.7mm}
    \setlength\figureheight{28mm} 
    \setlength\figurewidth{38mm}
\input{Rank1_base_SPF.tex}
\end{subfigure}
\begin{subfigure}[c]{0.38\textwidth}
\vspace*{-6mm}
    \setlength\figureheight{28mm} 
    \setlength\figurewidth{38mm}
\input{Rank1_base_IRLS.tex}
\end{subfigure}
\caption{Left column: \texttt{RiemAdaIHT}, center: \texttt{SPF}, right: \texttt{IRLS}. Phase transition experiments with $n_1=256$, $n_2=40$, $r=1$, Gaussian measurements. Algorithmic hyperparameters informed by model order knowledge (i.e., $\widetilde{r}=r$ and $\widetilde{s}=s$ for \texttt{IRLS}). White corresponds to empirical success rate of $1$, black to $0$.}
\label{fig:Performance_Rank1}
\end{figure}

\begin{figure}[h]
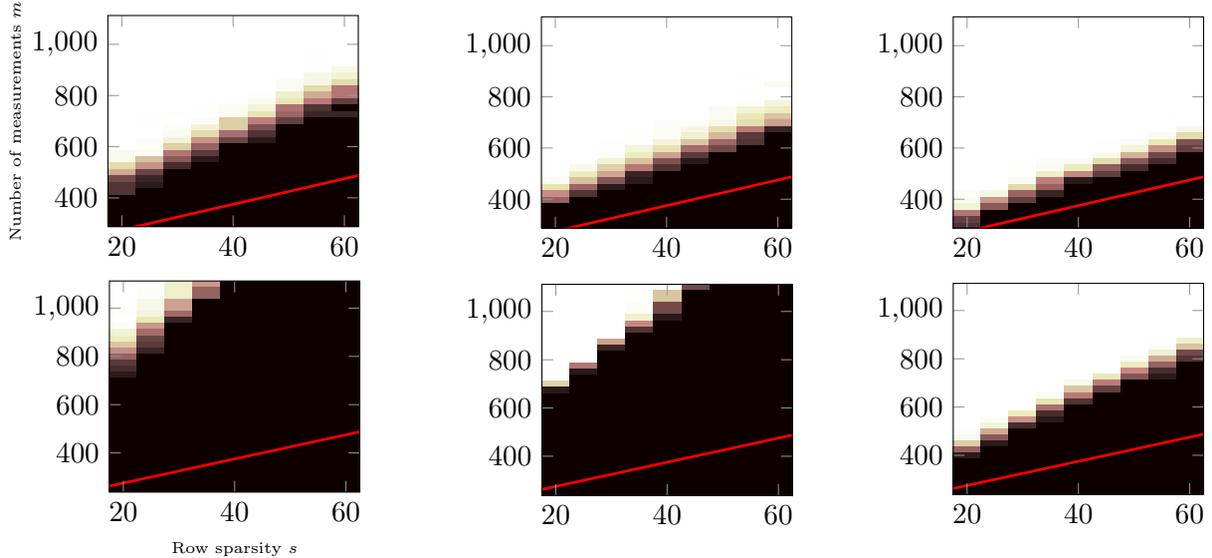

\vspace*{-3mm}
\begin{subfigure}[c]{0.34\textwidth}
\vspace*{0mm}
\hspace*{-3mm}
    \setlength\figureheight{28mm} 
    \setlength\figurewidth{38mm}
\input{Rank5_base_RAdaIHT.tex}
\end{subfigure}
\begin{subfigure}[c]{0.30\textwidth}
 \vspace*{3.5mm}
    \setlength\figureheight{28mm} 
    \setlength\figurewidth{38mm}
\input{Rank5_base_SPF.tex}
\end{subfigure}
\begin{subfigure}[c]{0.30\textwidth}
 \vspace*{3.5mm}
    \setlength\figureheight{28mm} 
    \setlength\figurewidth{38mm}
\input{Rank5_base_IRLS.tex}
\end{subfigure}
\hspace*{3mm}
\begin{subfigure}[c]{0.32\textwidth}
\vspace*{1.5mm}
    \setlength\figureheight{28mm} 
    \setlength\figurewidth{38mm}
\input{Rank5_missspec_RAdaIHT.tex}
\end{subfigure}
\hspace*{2.5mm}
\begin{subfigure}[c]{0.30\textwidth}
\vspace*{-2.25mm}
    \setlength\figureheight{28mm} 
    \setlength\figurewidth{38mm}
\input{Rank5_missspec_SPF.tex}
\end{subfigure}
\hspace*{3mm}
\begin{subfigure}[c]{0.30\textwidth}
\vspace*{-2.5mm}
    \setlength\figureheight{28mm} 
    \setlength\figurewidth{38mm}
\input{Rank5_missspec_IRLS.tex}
\end{subfigure}

\caption{Left column: \texttt{RiemAdaIHT}, center: \texttt{SPF}, right: \texttt{IRLS}. First row: As in \Cref{fig:Performance_Rank1}, but for data matrix $\X_\star$ of rank $r=5$. Second row: As first row, but hyper-parameters $r$ and $s$ are overestimated as $\widetilde{r}= 2r = 10$, $\widetilde{s} = \lfloor 1.5s \rfloor$.} 
\label{fig:Performance_Rank5}
\end{figure}
In \Cref{sec:rankone:measurements} and \Cref{sec:Fourier:measurements} in the supplementary material, we report on similar experiments conducted for other measurement operators than dense Gaussians, in which cases the empirical relative behavior of the methods is comparable.

\paragraph{Sensitivity to Parameter Choice}
In applications of our setting, the quantities $r$ and $s$ might be unknown or difficult to estimate. In the second row of \Cref{fig:Performance_Rank5}, we repeat the experiment of the first row (rank-$5$ ground truth), but run the algorithms with rank and sparsity estimates of $\widetilde{r} =2r$ and $\widetilde{s} = \lfloor 1.5s \rfloor$. Whereas all considered methods suffer a deterioration of performance, we observe that \texttt{IRLS} deteriorates relatively the least by a large margin. Furthermore, we observe that even if \texttt{IRLS} does not recovery $\X_\star$, it converges typically to a matrix that is still low-rank and row-sparse (with larger $r$ and $s$) satisfying the data constraint, while the other methods fail to convergence to such a matrix.

\paragraph{Convergence Behavior}

Finally, we examine the convergence rate of the iterates to validate the theoretical prediction of \Cref{thm:QuadraticConvergence} in the setting of \Cref{fig:Performance_Rank5}. \Cref{fig:ConvergenceRate} depicts in log-scale the approximation error over the iterates of \texttt{SPF}, \texttt{RiemAdaIHT}, and \texttt{IRLS}. We observe that the \texttt{IRLS} indeed   exhibits empirical quadratic convergence within a few iterations (around $10$), whereas the other methods clearly only exhibit linear convergence. The experiment further suggests that the rather pessimistic size of the convergence radius established by \Cref{thm:QuadraticConvergence} could possibly be improved by future investigations.

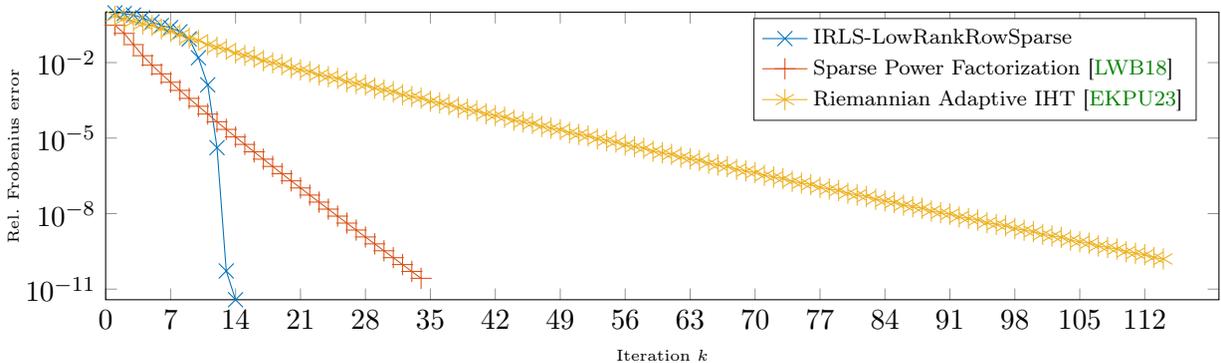
\begin{figure}[h]
    \setlength\figureheight{40mm} 
    \setlength\figurewidth{148mm}
%
%
\definecolor{mycolor1}{rgb}{0.00000,0.44700,0.74100}%
\definecolor{mycolor2}{rgb}{0.85000,0.32500,0.09800}%
\definecolor{mycolor3}{rgb}{0.92900,0.69400,0.12500}%
\begin{tikzpicture}

\begin{axis}[%
width=\figurewidth,
height=0.953\figureheight,
at={(0\figurewidth,0\figureheight)},
scale only axis,
xmin=0,
xmax=120,
xtick={0,7,14,21,28,35,42,49,56,63,70,77,84,91,98,105,112},
xlabel style={font=\color{white!15!black}},
xlabel={Iteration $k$},
ymode=log,
ymin=3.7822839889722e-12,
ymax=1,
yminorticks=true,
ylabel style={font=\color{white!15!black}},
ylabel={Rel. Frobenius error},
axis background/.style={fill=white},
legend style={legend cell align=left, align=left, draw=white!15!black,font=\fontsize{8}{30}\selectfont},
xlabel style={font=\tiny},ylabel style={font=\tiny},
]
\addplot [color=mycolor1, mark size=4.0pt, mark=x, mark options={solid, mycolor1}]
  table[row sep=crcr]{%
1	0.944543944025058\\
2	0.881682568847589\\
3	0.771442042846924\\
4	0.603564553257811\\
5	0.415770755798122\\
6	0.295631982437666\\
7	0.245030083745771\\
8	0.167358670209573\\
9	0.0873549229965664\\
10	0.0155383785368501\\
11	0.00130436097831419\\
12	4.15929209267789e-06\\
13	5.25100482253023e-11\\
14	3.7822839889722e-12\\
};
\addlegendentry{IRLS-LowRankRowSparse}

\addplot [color=mycolor2, mark size=4.0pt, mark=+, mark options={solid, mycolor2}]
  table[row sep=crcr]{%
1	0.303512142442868\\
2	0.144604748595718\\
3	0.0517013953824398\\
4	0.0187205475674369\\
5	0.00789011463973419\\
6	0.0035502372876764\\
7	0.00165251799031338\\
8	0.000784730692432444\\
9	0.000377683107265319\\
10	0.000183642713227771\\
11	9.00647084122406e-05\\
12	4.45140093409536e-05\\
13	2.21606818305215e-05\\
14	1.11087449149793e-05\\
15	5.60549613290796e-06\\
16	2.84641186298463e-06\\
17	1.45401078137097e-06\\
18	7.46889377882611e-07\\
19	3.85637631866643e-07\\
20	2.00051132292355e-07\\
21	1.04217323835957e-07\\
22	5.44974562484999e-08\\
23	2.85928220223391e-08\\
24	1.50452939029406e-08\\
25	7.93669860667505e-09\\
26	4.19588226403667e-09\\
27	2.22236813429006e-09\\
28	1.1789602220387e-09\\
29	6.26281741462123e-10\\
30	3.33072388164878e-10\\
31	1.77308187993858e-10\\
32	9.44658929820078e-11\\
33	5.03644110687026e-11\\
34	2.68671890910221e-11\\
};
\addlegendentry{Sparse Power Factorization \cite{lee2013near}}

\addplot [color=mycolor3, mark size=4.0pt, mark=asterisk, mark options={solid, mycolor3}]
  table[row sep=crcr]{%
1	0.701608831964487\\
2	0.539930175779277\\
3	0.42632581228311\\
4	0.342287486916806\\
5	0.298808677868574\\
6	0.218893669162785\\
7	0.17168811164682\\
8	0.12517892768718\\
9	0.0952214118247024\\
10	0.0762034060784759\\
11	0.0514701770853623\\
12	0.0409953857094529\\
13	0.0336873911311695\\
14	0.0238838077249988\\
15	0.0194546675564221\\
16	0.0162699354925078\\
17	0.0118721321080647\\
18	0.00980502193469019\\
19	0.00828090058246037\\
20	0.00614699261381788\\
21	0.00511920478153404\\
22	0.00435188962307056\\
23	0.00326314362014113\\
24	0.00273350010482163\\
25	0.00233376776526025\\
26	0.00176142254674268\\
27	0.00148160195835606\\
28	0.00126882463461453\\
29	0.000961845060902326\\
30	0.000811569672151853\\
31	0.000696621377038949\\
32	0.000529712971755063\\
33	0.00044806819220582\\
34	0.000385302357342381\\
35	0.000293668030691965\\
36	0.000248914007578598\\
37	0.00021436585316169\\
38	0.000163686233354019\\
39	0.000138983059743746\\
40	0.000119843566788371\\
41	9.165360673388e-05\\
42	7.79379830122702e-05\\
43	6.72785628816109e-05\\
44	5.15239898849612e-05\\
45	4.38710735140524e-05\\
46	3.79074328213452e-05\\
47	2.90674922841382e-05\\
48	2.47787266171494e-05\\
49	2.14289950942279e-05\\
50	1.64514612082716e-05\\
51	1.403869305944e-05\\
52	1.21504829323972e-05\\
53	9.33891299867356e-06\\
54	7.9767864031461e-06\\
55	6.90898446859562e-06\\
56	5.31621940083565e-06\\
57	4.54477673677557e-06\\
58	3.93912271224398e-06\\
59	3.03432494774184e-06\\
60	2.59613062247067e-06\\
61	2.25164797242049e-06\\
62	1.73629882432988e-06\\
63	1.48671523853407e-06\\
64	1.29026505094673e-06\\
65	9.95978136246478e-07\\
66	8.5345823192678e-07\\
67	7.41146247276334e-07\\
68	5.72669551850546e-07\\
69	4.91090192421519e-07\\
70	4.26726017323126e-07\\
71	3.30033047204335e-07\\
72	2.83229970516373e-07\\
73	2.46258255207093e-07\\
74	1.90626078174394e-07\\
75	1.63716356540713e-07\\
76	1.42431535685717e-07\\
77	1.10344927337926e-07\\
78	9.4840993218901e-08\\
79	8.25604823393193e-08\\
80	6.40090588692637e-08\\
81	5.50588202451644e-08\\
82	4.79584340709249e-08\\
83	3.72068486604546e-08\\
84	3.20301457252491e-08\\
85	2.7916368405068e-08\\
86	2.16705131913704e-08\\
87	1.86708879843386e-08\\
88	1.62827062370412e-08\\
89	1.26459577108409e-08\\
90	1.09047622573065e-08\\
91	9.51565376983733e-09\\
92	7.39333907944156e-09\\
93	6.38090086979388e-09\\
94	5.57139316315944e-09\\
95	4.33016939926375e-09\\
96	3.74050454971053e-09\\
97	3.26790306126755e-09\\
98	2.54046046524501e-09\\
99	2.19647970620022e-09\\
100	1.9200843779425e-09\\
101	1.49289903627625e-09\\
102	1.29192842348515e-09\\
103	1.13000868530523e-09\\
104	8.78669063654026e-10\\
105	7.61077045074582e-10\\
106	6.66066420135727e-10\\
107	5.17920981061397e-10\\
108	4.49016291392606e-10\\
109	3.93179651531548e-10\\
110	3.05710943196728e-10\\
111	2.65279160958359e-10\\
112	2.32415925496132e-10\\
113	1.80690100697097e-10\\
114	1.5693358929011e-10\\
};
\addlegendentry{Riemannian Adaptive IHT \cite{eisenmann2021riemannian}}

\end{axis}
\end{tikzpicture}%
    \caption{Comparison of convergence rate. Setting as in \Cref{fig:Performance_Rank5} with $s=40$ and $m=1125$.}
\label{fig:ConvergenceRate}
\end{figure}

\paragraph{Further experiments.} In \Cref{sec:evolution:objective}, we provide additional experiments investigating the self-balancing property of the objective \cref{eq:F:objective:def}, as well as the experiments on the noise robustness of the method in \Cref{sec:noise:robustness}.

\section{Conclusion, Limitations and Future Work}
\subsection*{Conclusion} In this paper, we adapted the \texttt{IRLS} framework to the problem of recovering simultaneously structured matrices from linear observations focusing on the special case of row sparsity and low-rankness. Our convergence guarantee \Cref{thm:QuadraticConvergence} is hereby the first one for any method minimizing combinations of structural surrogate objectives that holds in the information-theoretic near-optimal regime and exhibits local quadratic convergence.
The numerical experiments we conducted for synthetic data suggest that, due to its weak dependence on the choice of hyperparameters, \texttt{IRLS} in the form of \Cref{algo:MatrixIRLS} can be a practical method  
for identifying simultaneously structured data even in difficult problem instances. 

\subsection*{Limitations and Future Work}
As in the case of established \texttt{IRLS} methods that optimize non-convex surrogate objectives representing a single structure \cite{Daubechies10,Kummerle-JMLR2018,KM21,PKV22}, the radius of guaranteed quadratic convergence in \Cref{thm:QuadraticConvergence} is the most restrictive assumption. Beyond the interpretation in terms of surrogate minimization as presented in \Cref{thm:IRLS:majorization:MM}, which holds without any assumptions on the initialization, our method shares a lack of a global convergence guarantees with other non-convex \texttt{IRLS} algorithms \cite{Mohan12,KM21,PKV22}.

The generalization and application of our framework to combinations of structures beyond rank and row- (or column-)sparsity lies outside the scope of the present paper, but could involve subspace-structured low-rankness \cite{Fazel2003log,ChenChi-SpectralCS2014,Ye-Framelets2018,FazelOymakSun-Hankel2020} or analysis sparsity \cite{EladAharon-2006,RavishankarYe-2019}. A generalization of the presented \texttt{IRLS} framework to higher-order objects such as low-rank tensors is of future interest as convexifications of structure-promoting objectives face similar challenges \cite{Mu-SquareDeal2014,Oymak2015,Yao-EfficientTensorCompletion2019} in this case.

In parameter-efficient deep learning, both sparse \cite{Frankle-Lottery2018,Hoefler2021sparsity,Fedus-2022,Varma-SparseWinning2022} and low-rank \cite{Wen2017-coordinating,Wang2020linformer,Schotthoefer-LowRankLottery2022} weight parameter models have gained considerable attention due to the challenges of training and storing, e.g., in large transformer-based models \cite{Vaswani2017attention,Brown2020language}. It will be of interest to study whether in this non-linear case \texttt{IRLS}-like preconditioning of the parameter space can find network weights that are simultaneously sparse and low-rank, and could potentially lead to further increases in efficiency.

\section*{Acknowledgements}
The authors thank Max Pfeffer for providing their implementation of the algorithm \texttt{RiemAdaIHT} as considered in \cite{eisenmann2021riemannian}.

\bibliography{mybib}

\newcommand{\etalchar}[1]{$^{#1}$}
\begin{thebibliography}{HABN{\etalchar{+}}21}

\bibitem[AH15]{Aftab-WCACV2015}
K.~Aftab and R.~Hartley.
\newblock Convergence of iteratively re-weighted least squares to robust
  m-estimators.
\newblock In {\em IEEE Winter Conference on Applications of Computer Vision
  (WACV 2015)}, pages 480--487, 2015.

\bibitem[APS19]{Adil-NeurIPS2019}
D.~Adil, R.~Peng, and S.~Sachdeva.
\newblock Fast, Provably convergent {IRLS} Algorithm for p-norm Linear
  Regression.
\newblock {\em Advances in Neural Information Processing Systems (NeurIPS
  2019)}, 2019.

\bibitem[ARR14]{ahmed2014blind}
A.~Ahmed, B.~Recht, and J.~Romberg.
\newblock Blind deconvolution using convex programming.
\newblock {\em IEEE Transactions on Information Theory}, 60(3):1711--1732,
  2014.

\bibitem[BB18]{benning2018modern}
M.~Benning and M.~Burger.
\newblock Modern regularization methods for inverse problems.
\newblock {\em Acta Numerica}, 27:1--111, 2018.

\bibitem[BBDM21]{bertero2021introduction}
M.~Bertero, P.~Boccacci, and C.~De~Mol.
\newblock {\em Introduction to inverse problems in imaging}.
\newblock CRC Press, 2021.

\bibitem[BDE09]{bruckstein2009sparse}
A.~M. Bruckstein, D.~L. Donoho, and M.~Elad.
\newblock From sparse solutions of systems of equations to sparse modeling of
  signals and images.
\newblock {\em SIAM Review}, 51(1):34--81, 2009.

\bibitem[BMR{\etalchar{+}}20]{Brown2020language}
T.~Brown, B.~Mann, N.~Ryder, M.~Subbiah, J.~D. Kaplan, P.~Dhariwal,
  A.~Neelakantan, P.~Shyam, G.~Sastry, A.~Askell, et~al.
\newblock Language models are few-shot learners.
\newblock {\em Advances in Neural Information Processing Systems (NeurIPS
  2020)}, 33:1877--1901, 2020.

\bibitem[Bou23]{boumal2023introduction}
N.~Boumal.
\newblock {\em An introduction to optimization on smooth manifolds}.
\newblock Cambridge University Press, 2023.

\bibitem[BR16]{bahmani2016near}
S.~Bahmani and J.~Romberg.
\newblock Near-optimal estimation of simultaneously sparse and low-rank
  matrices from nested linear measurements.
\newblock {\em Information and Inference: A Journal of the IMA}, 5(3):331--351,
  2016.

\bibitem[BS15]{Beck-JOTA2015}
A.~Beck and S.~Sabach.
\newblock Weiszfeld's method: Old and new results.
\newblock {\em Journal of Optimization Theory and Applications}, 164(1):1--40,
  2015.

\bibitem[CC14]{ChenChi-SpectralCS2014}
Y.~Chen and Y.~Chi.
\newblock Robust Spectral Compressed Sensing via Structured Matrix Completion.
\newblock {\em IEEE Transactions on Information Theory}, 60(10):6576--6601,
  2014.

\bibitem[CE17]{Campisi-Blind2017}
P.~Campisi and K.~Egiazarian.
\newblock {\em Blind Image Deconvolution: Theory and Applications}.
\newblock CRC Press, 2017.

\bibitem[CESV13]{Candes13}
E.~J. Cand\`{e}s, Y.~Eldar, T.~Strohmer, and V.~Voroninski.
\newblock {Phase Retrieval via Matrix Completion}.
\newblock {\em SIAM J. Imag. Sci.}, 6(1):199--225, 2013.

\bibitem[CG17]{Chatterjee-RobustRotation2017}
A.~Chatterjee and V.~M. Govindu.
\newblock Robust relative rotation averaging.
\newblock {\em IEEE Transactions on Pattern Analysis and Machine Intelligence},
  40(4):958--972, 2017.

\bibitem[CH12]{Chen-Sparse2012}
L.~Chen and J.~Z. Huang.
\newblock Sparse reduced-rank regression for simultaneous dimension reduction
  and variable selection.
\newblock {\em Journal of the American Statistical Association},
  107(500):1533--1545, 2012.

\bibitem[Cha07]{chartrand2007exact}
R.~Chartrand.
\newblock Exact reconstruction of sparse signals via nonconvex minimization.
\newblock {\em IEEE Signal Processing Letters}, 14(10):707--710, 2007.

\bibitem[Che18]{Chen-Simultaneously2018}
W.~Chen.
\newblock Simultaneously sparse and low-rank matrix reconstruction via
  nonconvex and nonseparable regularization.
\newblock {\em IEEE Transactions on Signal Processing}, 66(20):5313--5323,
  2018.

\bibitem[CHHL14]{Chen-Preconditioning2014}
C.~Chen, J.~Huang, L.~He, and H.~Li.
\newblock Preconditioning for accelerated iteratively reweighted least squares
  in structured sparsity reconstruction.
\newblock In {\em Proceedings of the IEEE Conference on Computer Vision and
  Pattern Recognition (CVPR 2014)}, pages 2713--2720, 2014.

\bibitem[CHLH18]{Chen-FastAnalysisIRLS2018}
C.~Chen, L.~He, H.~Li, and J.~Huang.
\newblock Fast iteratively reweighted least squares algorithms for
  analysis-based sparse reconstruction.
\newblock {\em Medical Image Analysis}, 49:141--152, 2018.

\bibitem[CLM16]{Cai-OptimalRates2016}
T.~T. Cai, X.~Li, and Z.~Ma.
\newblock Optimal rates of convergence for noisy sparse phase retrieval via
  thresholded Wirtinger flow.
\newblock {\em The Annals of Statistics}, pages 2221--2251, 2016.

\bibitem[CP11]{CandesPlan11}
E.~J. Cand\`{e}s and Y.~Plan.
\newblock {Tight Oracle Inequalities for Low-Rank Matrix Recovery From a
  Minimal Number of Noisy Random Measurements}.
\newblock {\em IEEE Trans. Inf. Theory}, 57(4):2342--2359, 2011.

\bibitem[CRPW12]{Chandrasekaran-Convex2012}
V.~Chandrasekaran, B.~Recht, P.~A. Parrilo, and A.~S. Willsky.
\newblock The Convex Geometry of Linear Inverse Problems.
\newblock {\em Foundations of Computational Mathematics}, 12(6):805--849, 2012.

\bibitem[CY08]{chartrand_yin}
R.~Chartrand and W.~Yin.
\newblock Iteratively reweighted algorithms for compressive sensing.
\newblock In {\em {IEEE International Conference on Acoustics, Speech and
  Signal Processing (ICASSP)}}, pages 3869--3872, 2008.

\bibitem[CZ18]{Cai-RateOptimal2018}
T.~T. Cai and A.~Zhang.
\newblock Rate-optimal perturbation bounds for singular subspaces with
  applications to high-dimensional statistics.
\newblock {\em The Annals of Statistics}, 46(1):60--89, 2018.

\bibitem[Dax10]{dax2010extremum}
A.~Dax.
\newblock On extremum properties of orthogonal quotients matrices.
\newblock {\em Linear Algebra and its Applications}, 432(5):1234--1257, 2010.

\bibitem[DDFG10]{Daubechies10}
I.~Daubechies, R.~DeVore, M.~Fornasier, and C.~G\"unt\"urk.
\newblock Iteratively Reweighted Least Squares Minimization for Sparse
  Recovery.
\newblock {\em Commun. Pure Appl. Math.}, 63:1--38, 2010.

\bibitem[dGJL05]{d2005direct}
A.~d'Aspremont, L.~E. Ghaoui, M.~I. Jordan, and G.~R. Lanckriet.
\newblock A direct formulation for sparse {PCA} using semidefinite programming.
\newblock In {\em Advances in Neural Information Processing Systems (NIPS
  2005)}, pages 41--48, 2005.

\bibitem[Don06]{donoho2006compressed}
D.~L. Donoho.
\newblock Compressed sensing.
\newblock {\em IEEE Transactions on Information Theory}, 52(4):1289--1306,
  2006.

\bibitem[EA06]{EladAharon-2006}
M.~Elad and M.~Aharon.
\newblock Image Denoising Via Sparse and Redundant Representations Over Learned
  Dictionaries.
\newblock {\em IEEE Transactions on Image Processing}, 15(12):3736--3745, 2006.

\bibitem[EKPU23]{eisenmann2021riemannian}
H.~Eisenmann, F.~Krahmer, M.~Pfeffer, and A.~Uschmajew.
\newblock Riemannian thresholding methods for row-sparse and low-rank matrix
  recovery.
\newblock {\em Numerical Algorithms}, 93(2):669--693, 2023.

\bibitem[FC19]{Frankle-Lottery2018}
J.~Frankle and M.~Carbin.
\newblock The Lottery Ticket Hypothesis: Finding Sparse, Trainable Neural
  Networks.
\newblock In {\em International Conference on Learning Representations}, 2019.

\bibitem[FGJR20]{foucart2019jointly}
S.~Foucart, R.~Gribonval, L.~Jacques, and H.~Rauhut.
\newblock Jointly low-rank and bisparse recovery: Questions and partial
  answers.
\newblock {\em Analysis and Applications}, 18(01):25--48, 2020.

\bibitem[FHB03]{Fazel2003log}
M.~Fazel, H.~Hindi, and S.~P. Boyd.
\newblock Log-det heuristic for matrix rank minimization with applications to
  Hankel and Euclidean distance matrices.
\newblock In {\em Proceedings of the 2003 American Control Conference},
  volume~3, pages 2156--2162. IEEE, 2003.

\bibitem[FMN21]{fornasier2018robust}
M.~Fornasier, J.~Maly, and V.~Naumova.
\newblock Robust recovery of low-rank matrices with non-orthogonal sparse
  decomposition from incomplete measurements.
\newblock {\em Applied Mathematics and Computation}, 392, 2021.

\bibitem[Fou11]{foucart2011hard}
S.~Foucart.
\newblock Hard thresholding pursuit: an algorithm for compressive sensing.
\newblock {\em SIAM Journal on Numerical Analysis}, 49(6):2543--2563, 2011.

\bibitem[FPRW16]{FornasierPeterRauhutWorm-2016}
M.~Fornasier, S.~Peter, H.~Rauhut, and S.~Worm.
\newblock Conjugate Gradient Acceleration of Iteratively Re-Weighted Least
  Squares Methods.
\newblock {\em Comput. Optim. Appl.}, 65(1):205--259, 2016.

\bibitem[FR13]{foucart:2013}
S.~Foucart and H.~Rauhut.
\newblock {\em A Mathematical Introduction to Compressive Sensing}.
\newblock Birkh\"auser Basel, 2013.

\bibitem[FRW11]{forawa11}
M.~{Fornasier}, H.~{Rauhut}, and R.~{Ward}.
\newblock {Low-rank matrix recovery via iteratively reweighted least squares
  minimization}.
\newblock {\em {SIAM Journal on Optimization}}, 21(4):1614--1640, 2011.

\bibitem[FZS22]{Fedus-2022}
W.~Fedus, B.~Zoph, and N.~Shazeer.
\newblock Switch transformers: Scaling to trillion parameter models with simple
  and efficient sparsity.
\newblock {\em The Journal of Machine Learning Research}, 23(1):5232--5270,
  2022.

\bibitem[GR97]{GorodnitskyRao-1997}
I.~Gorodnitsky and B.~Rao.
\newblock Sparse signal reconstruction from limited data using FOCUSS: A
  re-weighted minimum norm algorithm.
\newblock {\em IEEE Trans. Signal Process.}, 45(3):600--616, 1997.

\bibitem[GTRK16]{Giampouras-2016Simultaneously}
P.~V. Giampouras, K.~E. Themelis, A.~A. Rontogiannis, and K.~D. Koutroumbas.
\newblock Simultaneously sparse and low-rank abundance matrix estimation for
  hyperspectral image unmixing.
\newblock {\em IEEE Transactions on Geoscience and Remote Sensing},
  54(8):4775--4789, 2016.

\bibitem[GV12]{GolbabaeeVanderghenyst-2012}
M.~Golbabaee and P.~Vandergheynst.
\newblock Hyperspectral Image Compressed Sensing via Low-Rank and Joint-Sparse
  Matrix Recovery.
\newblock {\em Proceedings of ICASSP 2012}, 2012.

\bibitem[HABN{\etalchar{+}}21]{Hoefler2021sparsity}
T.~Hoefler, D.~Alistarh, T.~Ben-Nun, N.~Dryden, and A.~Peste.
\newblock Sparsity in deep learning: Pruning and growth for efficient inference
  and training in neural networks.
\newblock {\em The Journal of Machine Learning Research}, 22(1):10882--11005,
  2021.

\bibitem[HV19]{haeffele2019structured}
B.~D. Haeffele and R.~Vidal.
\newblock Structured low-rank matrix factorization: Global optimality,
  algorithms, and applications.
\newblock {\em IEEE Transactions on Pattern Analysis and Machine Intelligence},
  42(6):1468--1482, 2019.

\bibitem[HW77]{Holland-1977}
P.~W. Holland and R.~E. Welsch.
\newblock Robust regression using iteratively reweighted least-squares.
\newblock {\em Communications in Statistics - Theory and Methods},
  6(9):813--827, 1977.

\bibitem[JC93]{Jefferies-Restoration1993}
S.~M. Jefferies and J.~C. Christou.
\newblock Restoration of astronomical images by iterative blind deconvolution.
\newblock {\em The Astrophysical Journal}, 415:862, 1993.

\bibitem[JH17]{jagatap2017fast}
G.~Jagatap and C.~Hegde.
\newblock Fast, sample-efficient algorithms for structured phase retrieval.
\newblock {\em Advances in Neural Information Processing Systems (NeurIPS
  2017)}, 30, 2017.

\bibitem[JNS13]{jain2013low}
P.~Jain, P.~Netrapalli, and S.~Sanghavi.
\newblock Low-rank matrix completion using alternating minimization.
\newblock {\em Proceedings of the forty-fifth Annual ACM Symposium on Theory of
  Computing}, pages 665--674, 2013.

\bibitem[JOH17]{Jaganathan-Sparse2017}
K.~Jaganathan, S.~Oymak, and B.~Hassibi.
\newblock Sparse phase retrieval: Uniqueness guarantees and recovery
  algorithms.
\newblock {\em IEEE Transactions on Signal Processing}, 65(9):2402--2410, 2017.

\bibitem[KASL21]{Ke2021-iterativelysumoflog}
C.~Ke, M.~Ahn, S.~Shin, and Y.~Lou.
\newblock Iteratively Reweighted Group Lasso Based on Log-Composite
  Regularization.
\newblock {\em SIAM Journal on Scientific Computing}, 43(5):S655--S678, 2021.

\bibitem[KMV21]{KM21}
C.~K{\"u}mmerle and C.~Mayrink~Verdun.
\newblock A Scalable Second Order Method for Ill-Conditioned Matrix Completion
  from Few Samples.
\newblock In {\em International Conference on Machine Learning (ICML 2021)},
  2021.

\bibitem[KMVS21]{KMVS-NeurIPS2021}
C.~K{\"u}mmerle, C.~Mayrink~Verdun, and D.~St{\"o}ger.
\newblock Iteratively Reweighted Least Squares for Basis Pursuit with Global
  Linear Convergence Rate.
\newblock {\em Advances in Neural Information Processing Systems (NeurIPS
  2021)}, 2021.

\bibitem[KS18]{Kummerle-JMLR2018}
C.~K{\"u}mmerle and J.~Sigl.
\newblock Harmonic mean iteratively reweighted least squares for low-rank
  matrix recovery.
\newblock {\em The Journal of Machine Learning Research}, 19(1):1815--1863,
  2018.

\bibitem[KSJ19]{kliesch2019simultaneous}
M.~Kliesch, S.~J. Szarek, and P.~Jung.
\newblock Simultaneous Structures in Convex Signal Recovery---Revisiting the
  Convex Combination of Norms.
\newblock {\em Frontiers in Applied Mathematics and Statistics}, 5:23, 2019.

\bibitem[KST95]{klibanov1995phase}
M.~V. Klibanov, P.~E. Sacks, and A.~V. Tikhonravov.
\newblock The phase retrieval problem.
\newblock {\em Inverse problems}, 11(1):1, 1995.

\bibitem[K{\"u}m19]{K19}
C.~K{\"u}mmerle.
\newblock {\em Understanding and Enhancing Data Recovery Algorithms: From
  Noise-Blind Sparse Recovery to Reweighted Methods for Low-Rank Matrix
  Optimization}.
\newblock Dissertation, Technical University of Munich, 2019.
\newblock Available at
  \url{https://mediatum.ub.tum.de/doc/1521436/1521436.pdf}.

\bibitem[Lan16]{lange2016mm}
K.~Lange.
\newblock {\em MM optimization algorithms}.
\newblock SIAM, 2016.

\bibitem[LC22]{Lee-HARA2022}
S.~H. Lee and J.~Civera.
\newblock HARA: A hierarchical approach for robust rotation averaging.
\newblock In {\em Proceedings of the IEEE/CVF Conference on Computer Vision and
  Pattern Recognition (CVPR 2022)}, pages 15777--15786, 2022.

\bibitem[LK23]{Lefkimmiatis-LearningSparseLowRankPriors2023}
S.~Lefkimmiatis and I.~S. Koshelev.
\newblock Learning Sparse and Low-Rank Priors for Image Recovery via Iterative
  Reweighted Least Squares Minimization.
\newblock In {\em The Eleventh International Conference on Learning
  Representations}, 2023.

\bibitem[LLJB16]{lee2016blind}
K.~Lee, Y.~Li, M.~Junge, and Y.~Bresler.
\newblock Blind recovery of sparse signals from subsampled convolution.
\newblock {\em IEEE Transactions on Information Theory}, 63(2):802--821, 2016.

\bibitem[LLSW19]{Li-RapidBlindDeconvolution2019}
X.~Li, S.~Ling, T.~Strohmer, and K.~Wei.
\newblock Rapid, robust, and reliable blind deconvolution via nonconvex
  optimization.
\newblock {\em Applied and Computational Harmonic Analysis}, 47(3):893--934,
  2019.

\bibitem[LS05]{Lewis05_Nonsm1}
A.~S. Lewis and H.~S. Sendov.
\newblock {Nonsmooth Analysis of Singular Values. Part I: Theory}.
\newblock {\em Set-Valued Analysis}, 13(3):213--241, 2005.

\bibitem[LW20]{Lyu-ExactSinTheta2020}
H.~Lyu and R.~Wang.
\newblock An exact sin $\Theta$ formula for matrix perturbation analysis and
  its applications.
\newblock {\em arXiv preprint arXiv:2011.07669}, 2020.

\bibitem[LWB18]{lee2013near}
K.~Lee, Y.~Wu, and Y.~Bresler.
\newblock Near-Optimal Compressed Sensing of a Class of Sparse Low-Rank
  Matrices Via Sparse Power Factorization.
\newblock {\em IEEE Transactions on Information Theory}, 64(3):1666--1698,
  2018.

\bibitem[LXY13]{Lai-SIAM-J-NA2013}
M.-J. Lai, Y.~Xu, and W.~Yin.
\newblock Improved iteratively reweighted least squares for unconstrained
  smoothed $\ell_q$-minimization.
\newblock {\em SIAM Journal on Numerical Analysis}, 51(2):927--957, 2013.

\bibitem[Mal23]{maly2021robust}
J.~Maly.
\newblock Robust sensing of low-rank matrices with non-orthogonal sparse
  decomposition.
\newblock {\em Applied and Computational Harmonic Analysis}, 2023.

\bibitem[MBP{\etalchar{+}}14]{mairal2014sparse}
J.~Mairal, F.~Bach, J.~Ponce, et~al.
\newblock Sparse modeling for image and vision processing.
\newblock {\em Foundations and Trends{\textregistered} in Computer Graphics and
  Vision}, 8(2-3):85--283, 2014.

\bibitem[MF12]{Mohan12}
K.~Mohan and M.~Fazel.
\newblock Iterative Reweighted Algorithms for Matrix Rank Minimization.
\newblock {\em J. Mach. Learn. Res.}, 13(1):3441--3473, 2012.

\bibitem[MGJK19]{Mukhoty-AISTATS2019}
B.~Mukhoty, G.~Gopakumar, P.~Jain, and P.~Kar.
\newblock Globally-convergent Iteratively Reweighted Least Squares for Robust
  Regression Problems.
\newblock In {\em International Conference on Artificial Intelligence and
  Statistics}, pages 313--322, 2019.

\bibitem[MHWG14]{Mu-SquareDeal2014}
C.~Mu, B.~Huang, J.~Wright, and D.~Goldfarb.
\newblock Square Deal: Lower Bounds and Improved Relaxations for Tensor
  Recovery.
\newblock In {\em Proceedings of the 31st International Conference on Machine
  Learning (ICML 2014)}, volume~32 of {\em Proceedings of Machine Learning
  Research}, pages 73--81, 2014.

\bibitem[MI17]{magdon2017np}
M.~Magdon-Ismail.
\newblock NP-hardness and inapproximability of sparse PCA.
\newblock {\em Information Processing Letters}, 126:35--38, 2017.

\bibitem[MWCC20]{Ma-Implicit2020}
C.~Ma, K.~Wang, Y.~Chi, and Y.~Chen.
\newblock Implicit Regularization in Nonconvex Statistical Estimation: Gradient
  Descent Converges Linearly for Phase Retrieval, Matrix Completion, and Blind
  Deconvolution.
\newblock {\em Foundations of Computational Mathematics}, 2020.

\bibitem[ODBP15]{Ochs-SIAM-J-IS2015}
P.~Ochs, A.~Dosovitskiy, T.~Brox, and T.~Pock.
\newblock On iteratively reweighted algorithms for nonsmooth nonconvex
  optimization in computer vision.
\newblock {\em SIAM Journal on Imaging Sciences}, 8(1):331--372, 2015.

\bibitem[OJF{\etalchar{+}}15]{Oymak2015}
S.~Oymak, A.~Jalali, M.~Fazel, Y.~C. Eldar, and B.~Hassibi.
\newblock Simultaneously structured models with application to sparse and
  low-rank matrices.
\newblock {\em IEEE Transactions on Information Theory}, 61(5):2886--2908,
  2015.

\bibitem[PKV22]{PKV22}
L.~Peng, C.~K\"{u}mmerle, and R.~Vidal.
\newblock {Global Linear and Local Superlinear Convergence of IRLS for
  Non-Smooth Robust Regression}.
\newblock In {\em Advances in Neural Information Processing Systems (NeurIPS
  2022)}, volume~35, pages 28972--28987, 2022.

\bibitem[PKV23]{PengKuemmerleVidal-CVPR2023}
L.~Peng, C.~K\"ummerle, and R.~Vidal.
\newblock {On the Convergence of IRLS and Its Variants in Outlier-Robust
  Estimation}.
\newblock In {\em {IEEE Conference on Computer Vision and Pattern Recognition
  (CVPR 2023)}}, 2023.

\bibitem[RFP10]{recht2010}
B.~Recht, M.~Fazel, and P.~A. Parrilo.
\newblock Guaranteed Minimum-Rank Solutions of Linear Matrix Equations via
  Nuclear Norm Minimization.
\newblock {\em SIAM Review}, 52(3):471--501, 2010.

\bibitem[ROV14]{Richard2014tight}
E.~Richard, G.~R. Obozinski, and J.-P. Vert.
\newblock Tight convex relaxations for sparse matrix factorization.
\newblock {\em Advances in Neural Information Processing Systems (NIPS 2014)},
  27, 2014.

\bibitem[RSV12]{RichardSavalle-Simultaneously2012}
E.~Richard, P.-A. Savalle, and N.~Vayatis.
\newblock Estimation of simultaneously sparse and low rank matrices.
\newblock In {\em Proceedings of the 29th International Conference on Machine
  Learning (ICML 2012)}, pages 51--58, 2012.

\bibitem[RYF19]{RavishankarYe-2019}
S.~Ravishankar, J.~C. Ye, and J.~A. Fessler.
\newblock Image reconstruction: From sparsity to data-adaptive methods and
  machine learning.
\newblock {\em Proceedings of the IEEE}, 108(1):86--109, 2019.

\bibitem[SBP16]{Sun-TSP2016}
Y.~Sun, P.~Babu, and D.~P. Palomar.
\newblock Majorization-minimization algorithms in signal processing,
  communications, and machine learning.
\newblock {\em IEEE Transactions on Signal Processing}, 65(3):794--816, 2016.

\bibitem[SC21]{Shi-ManifoldBlindDeconvolution2021}
L.~Shi and Y.~Chi.
\newblock Manifold gradient descent solves multi-channel sparse blind
  deconvolution provably and efficiently.
\newblock {\em IEEE Transactions on Information Theory}, 67(7):4784--4811,
  2021.

\bibitem[SM17]{Samejima-GeneralNormIRLS2017}
M.~Samejima and Y.~Matsushita.
\newblock Fast general norm approximation via iteratively reweighted least
  squares.
\newblock In {\em Computer Vision--ACCV 2016 Workshops: ACCV 2016 International
  Workshops, Taipei, Taiwan, November 20-24, 2016}, pages 207--221. Springer,
  2017.

\bibitem[SOF20]{FazelOymakSun-Hankel2020}
Y.~Sun, S.~Oymak, and M.~Fazel.
\newblock Finite Sample System Identification: Optimal Rates and the Role of
  Regularization.
\newblock In {\em Proceedings of the 2nd Conference on Learning for Dynamics
  and Control}, volume 120, pages 16--25. PMLR, 10--11 Jun 2020.

\bibitem[Sol19]{Soltanolkotabi-Structured2019}
M.~Soltanolkotabi.
\newblock Structured signal recovery from quadratic measurements: Breaking
  sample complexity barriers via nonconvex optimization.
\newblock {\em IEEE Transactions on Information Theory}, 65(4):2374--2400,
  2019.

\bibitem[SWL22]{Shi-Robust2022}
Y.~Shi, C.~M. Wyeth, and G.~Lerman.
\newblock Robust Group Synchronization via Quadratic Programming.
\newblock In {\em International Conference on Machine Learning (ICML 2022)},
  pages 20095--20105. PMLR, 2022.

\bibitem[SZK{\etalchar{+}}22]{Schotthoefer-LowRankLottery2022}
S.~Schotth\"{o}fer, E.~Zangrando, J.~Kusch, G.~Ceruti, and F.~Tudisco.
\newblock Low-rank lottery tickets: finding efficient low-rank neural networks
  via matrix differential equations.
\newblock In {\em Advances in Neural Information Processing Systems (NeurIPS
  2022)}, volume~35, pages 20051--20063, 2022.

\bibitem[Tik63]{tikhonov1963solution}
A.~N. Tikhonov.
\newblock Solution of incorrectly formulated problems and the regularization
  method.
\newblock {\em Soviet Math.}, 4:1035--1038, 1963.

\bibitem[TRB17]{Tsinos-Distributed2017}
C.~G. Tsinos, A.~A. Rontogiannis, and K.~Berberidis.
\newblock Distributed blind hyperspectral unmixing via joint sparsity and
  low-rank constrained non-negative matrix factorization.
\newblock {\em IEEE Transactions on Computational Imaging}, 3(2):160--174,
  2017.

\bibitem[Van13]{Vandereycken13}
B.~Vandereycken.
\newblock {Low-Rank Matrix Completion by {R}iemannian Optimization}.
\newblock {\em SIAM J. Optim.}, 23(2), 2013.

\bibitem[Vog02]{vogel2002computational}
C.~R. Vogel.
\newblock {\em Computational methods for inverse problems}.
\newblock SIAM, 2002.

\bibitem[VSP{\etalchar{+}}17]{Vaswani2017attention}
A.~Vaswani, N.~Shazeer, N.~Parmar, J.~Uszkoreit, L.~Jones, A.~N. Gomez,
  {\L}.~Kaiser, and I.~Polosukhin.
\newblock Attention is all you need.
\newblock {\em Advances in Neural Information Processing Systems (NeurIPS
  2017)}, 30, 2017.

\bibitem[VTCZ{\etalchar{+}}22]{Varma-SparseWinning2022}
M.~Varma~T, X.~Chen, Z.~Zhang, T.~Chen, S.~Venugopalan, and Z.~Wang.
\newblock Sparse Winning Tickets are Data-Efficient Image Recognizers.
\newblock {\em Advances in Neural Information Processing Systems (NeurIPS
  2022)}, 35:4652--4666, 2022.

\bibitem[Wei37]{Weiszfeld37}
E.~Weiszfeld.
\newblock Sur le point pour lequel la somme des distances de n points
  donn{\'e}s est minimum.
\newblock {\em Tohoku Mathematical Journal, First Series}, 43:355--386, 1937.

\bibitem[WLK{\etalchar{+}}20]{Wang2020linformer}
S.~Wang, B.~Z. Li, M.~Khabsa, H.~Fang, and H.~Ma.
\newblock Linformer: Self-attention with linear complexity.
\newblock {\em arXiv preprint arXiv:2006.04768}, 2020.

\bibitem[Woo50]{Woodbury50}
M.~A. Woodbury.
\newblock Inverting modified matrices.
\newblock {\em Memorandum report}, 42(106):336, 1950.

\bibitem[WXW{\etalchar{+}}17]{Wen2017-coordinating}
W.~Wen, C.~Xu, C.~Wu, Y.~Wang, Y.~Chen, and H.~Li.
\newblock Coordinating filters for faster deep neural networks.
\newblock In {\em Proceedings of the IEEE International Conference on Computer
  Vision (CVPR 2017)}, pages 658--666, 2017.

\bibitem[YGK20]{Yu-SimultaneouslyTwoWay2020}
M.~Yu, V.~Gupta, and M.~Kolar.
\newblock Recovery of simultaneous low rank and two-way sparse coefficient
  matrices, a nonconvex approach.
\newblock {\em Electronic Journal of Statistics}, 14:413--457, 2020.

\bibitem[YHC18]{Ye-Framelets2018}
J.~C. Ye, Y.~Han, and E.~Cha.
\newblock Deep convolutional framelets: A general deep learning framework for
  inverse problems.
\newblock {\em SIAM Journal on Imaging Sciences}, 11(2):991--1048, 2018.

\bibitem[YKH19]{Yao-EfficientTensorCompletion2019}
Q.~Yao, J.~T.-Y. Kwok, and B.~Han.
\newblock Efficient nonconvex regularized tensor completion with
  structure-aware proximal iterations.
\newblock In {\em International Conference on Machine Learning (ICML 2019)},
  pages 7035--7044. PMLR, 2019.

\bibitem[ZB15]{Zeinalkhani-IRLS2015}
Z.~Zeinalkhani and A.~H. Banihashemi.
\newblock Iterative Reweighted $\ell_{2}$/$\ell_{1}$ Recovery Algorithms for
  Compressed Sensing of Block Sparse Signals.
\newblock {\em IEEE Transactions on Signal Processing}, 63(17):4516--4531,
  2015.

\bibitem[ZHL{\etalchar{+}}22]{ZhangHuangLiZhangYin-Hyperspectral2022}
Y.~Zhang, L.-T. Huang, Y.~Li, K.~Zhang, and C.~Yin.
\newblock Low-rank and sparse matrix recovery for hyperspectral image
  reconstruction using Bayesian learning.
\newblock {\em Sensors}, 22(1):343, 2022.

\bibitem[ZHT06]{ZouHastieTibshirani-2006}
H.~Zou, T.~Hastie, and R.~Tibshirani.
\newblock Sparse Principal Component Analysis.
\newblock {\em Journal of Computational and Graphical Statistics},
  15(2):265--286, 2006.

\bibitem[ZZL22]{Zhang-GraphRefinement2022}
Z.~Zhang, Z.~Zhai, and L.~Li.
\newblock Graph Refinement via Simultaneously Low-Rank and Sparse
  Approximation.
\newblock {\em SIAM Journal on Scientific Computing}, 44(3):A1525--A1553, 2022.

\end{thebibliography}
\bibliographystyle{alphaabbr}


\appendix

\section*{Supplementary material for \emph{Recovering Simultaneously Structured Data via Non-Convex Iteratively Reweighted Least Squares}}

This supplement is structured as follows.

\begin{itemize}
    \item \Cref{sec:AdditionalNumerics} presents some details about the experimental setup as well as additional numerical experiments.
    \item \Cref{sec:ProofOfQuadraticConvergence} presents the proof of \Cref{thm:QuadraticConvergence}.
    \item \Cref{sec:ProofOfMajorization} presents the proof of \Cref{thm:IRLS:majorization:MM}.
    \item \Cref{sec:TechnicalAddendum} details some technical results that are used in \Cref{sec:ProofOfQuadraticConvergence,sec:ProofOfMajorization}.
\end{itemize}

\section{Experimental Setup and Supplementary Experiments}
\label{sec:AdditionalNumerics}
In this section, we elaborate on the detailed experimental setup that was used in \Cref{sec:Numerics} of the main paper. Furthermore, we provide additional experiments comparing the behavior of the three methods studied in \Cref{sec:Numerics} for linear measurement operators $\mathcal{A}$ that are closer to operators that can be encountered in applications of simultaneous low-rank and group-sparse recovery. Finally, we shed light on the evolution of the objective function \cref{eq:F:objective:def} of \texttt{IRLS} (\Cref{algo:MatrixIRLS}), including in situations where the algorithm does not manage to recover the ground truth.

\subsection{Experimental Setup}
The experiments of \Cref{sec:Numerics} were conducted using MATLAB implementations of the three algorithms on different Linux machines using MATLAB versions R2019b or R2022b. In total, the preparation and execution of the experiments used approximately 1200 CPU hours. The CPU models used in the simulations are Dual 18-Core Intel Xeon Gold 6154, Dual 24-Core Intel Xeon Gold 6248R, Dual 8-Core Intel Xeon E5-2667, 28-Core Intel Xeon E5-2690 v3, 64-Core Intel Xeon Phi KNL 7210-F. For \texttt{Sparse Power Factorization (SPF)} \cite{lee2013near}, we used our custom implementation of \cite[{Algorithm 4 "rSPF\_HTP"}]{lee2013near} and for \texttt{Riemannian adaptive iterative hard thresholding (RiemAdaIHT)} \cite{eisenmann2021riemannian}, we used an implementation provided to us by Max Pfeffer in private communications. We refer to \Cref{sec:Complexity} for implementation details for the \texttt{IRLS} method \Cref{algo:MatrixIRLS}.

In all phase transition experiments, we define \emph{successful recovery} such that the relative Frobenius error $\frac{\norm{\Xk{K}-\X_\star}{F}}{\norm{\X_\star}{F}}$ of the iterate $\Xk{K}$ returned by the algorithm relative to the simultaneously low-rank and row-sparse ground truth matrix $\X_\star$ is smaller than the threshold $10^{-4}$. As stopping criteria, we used the criterion that the relative change of Frobenius norm satisfies $\frac{\norm{\Xk{k} - \Xk{k-1}}{F}}{\norm{\Xk{k}}{F}} < \operatorname{tol}$ for \texttt{IRLS}, the change in the matrix factors norms satisfy $\norm{\f{U}_k - \f{U}_{k-1}}{} < \operatorname{tol}$ and $\norm{\f{V}_k - \f{V}_{k-1}}{} < \operatorname{tol}$ for \texttt{SPF}, and the norm of the Riemannian gradient in \texttt{RiemAdaIHT} being smaller than $\operatorname{tol}$ for $\operatorname{tol} = 10^{-10}$, or if a maximal number of iterations is reached. This iteration threshold was chosen as $\operatorname{max\_iter} = 250$ for \texttt{IRLS} and \texttt{SPF} and as  $\operatorname{max\_iter} = 2000$ for \texttt{RiemAdaIHT}, reflecting the fact that \texttt{RiemAdaIHT} is a gradient-type method which might need many iterations to reach a high-accuracy solution. The parameters were chosen so that the stopping criteria do not prevent a method's iterates reaching the recovery threshold if they were to reach $\X_\star$ eventually.

In the experiments, we chose random ground truths $\X_\star \in \R^{n_1 \times n_2}$ of rank $r$ and row-sparsity $s$ such that $\X_\star = \tilde{\X}_\star / \norm{\tilde{\X}_\star}{F}$, where $\tilde{\X}_\star = \f{U}_{\star} \diag(\f{d}_\star) \f{V}_{\star}^*$, and where $\f{U}_{\star} \in \R^{n_1 \times r}$ is a matrix with $s$ non-zero rows whose location is chosen uniformly at random and whose entries are drawn from i.i.d. standard Gaussian random variables, $\f{d}$ has i.i.d. standard Gaussian entries and $\f{V}_{\star} \in \R^{n_2 \times r}$ has likewise i.i.d. standard Gaussian entries.

\subsection{Random Rank-One Measurements}
\label{sec:rankone:measurements}
In \Cref{sec:Numerics}, we considered only measurement operator $\A: \R^{n_1\times n_2} \to \R^m$ whose matrix representation consists of i.i.d. Gaussian entries, i.e., operators such that there are independent matrices $\f{A}_1,\ldots \f{A}_m$ with i.i.d. standard Gaussian entries such that
\[
\A(\X)_{j} = \langle \f{A}_j , \X \rangle_F
\]
for any $\X \in \R^{n_1 \times n_2}$. While it is known that such Gaussian measurement operators satisfy the $(r,s)$-RIP of \Cref{sec:Numerics}, which is the basis of our convergence theorem \Cref{thm:QuadraticConvergence}, in a regime of a near-optimal number of measurements with high probability, practically relevant measurement operators are often more structured; another downside of dense Gaussian measurements is that it is computationally expensive to implement their action on matrices.

In relevant applications of our setup, however, e.g., in sparse phase retrieval \cite{Jaganathan-Sparse2017,Cai-OptimalRates2016,jagatap2017fast} or blind deconvolution \cite{lee2016blind,Shi-ManifoldBlindDeconvolution2021}, the measurement operator consists of rank-one measurements. For this reason, we now conduct experiments in settings related to the ones depicted in \Cref{fig:Performance_Rank1} and \Cref{fig:Performance_Rank5} \Cref{sec:Numerics}, but for \emph{random rank-one measurements} where the action of $\A: \R^{n_1\times n_2} \to \R^m$ on $\f{X}$ can be written as
\[
\A(\X)_{j} = \langle \f{a}_j \f{b}_j^* , \X \rangle_F
\]
for each $j =1,\ldots, m$, where $\f{a}_j , \f{b}_j$ are independent random standard Gaussian vectors. In \Cref{fig:RandomRankOne_Meas}, we report the phase transition performance of \texttt{RiemAdaIHT}, \texttt{SPF} and \texttt{IRLS} for $(256 \times 40)$-dimensional ground truths of different row-sparsities and different ranks if we are given such random rank-one measurements.

\begin{figure}[h]
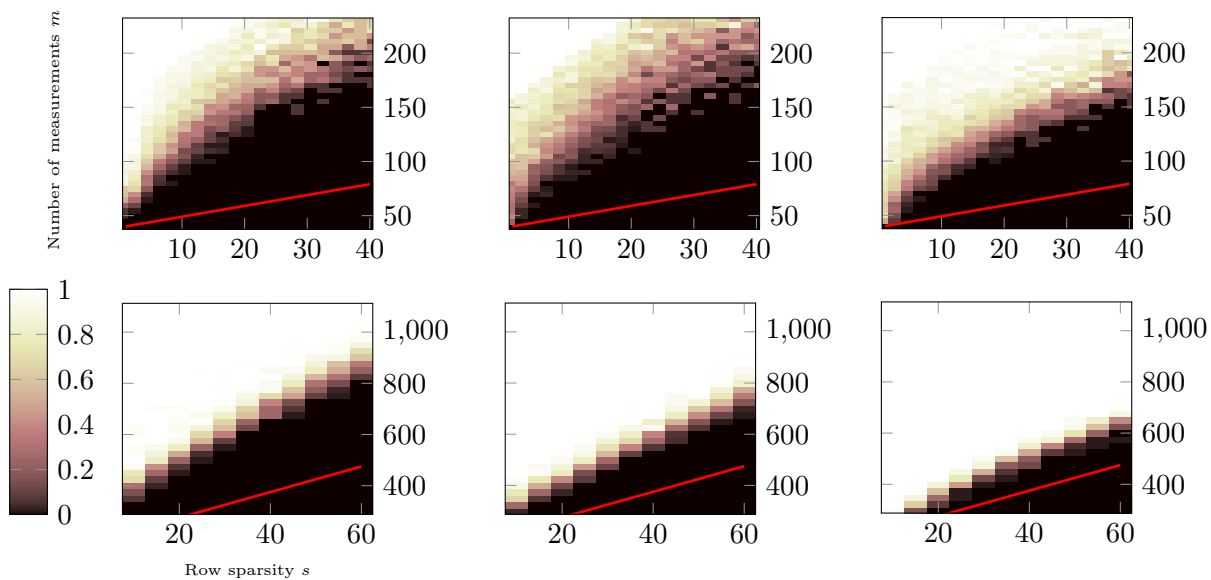

\begin{subfigure}[c]{0.345\textwidth}
\vspace*{-1mm}
 \hspace*{0.25mm}
    \setlength\figureheight{28mm} 
    \setlength\figurewidth{38mm}
\input{GR1Meas_rank1_RAdaIHT.tex}
\end{subfigure}
\begin{subfigure}[c]{0.31\textwidth}
 \hspace*{4mm}
 \vspace*{-1.3mm}
    \setlength\figureheight{28mm} 
    \setlength\figurewidth{38mm}
\input{GR1Meas_rank1_SPF.tex}
\end{subfigure}
\hspace*{3mm}
\begin{subfigure}[c]{0.30\textwidth}
\vspace*{1.2mm}
    \setlength\figureheight{28mm} 
    \setlength\figurewidth{38mm}
\input{GR1Meas_rank1_IRLS-LRRS.tex}
\end{subfigure}
\begin{subfigure}[c]{0.36\textwidth}
\hspace*{-2.4mm}
    \setlength\figureheight{28mm} 
    \setlength\figurewidth{38mm}
\input{GR1Meas_rank5_RAdaIHT.tex}
\end{subfigure}
\hspace*{3.75mm}
\begin{subfigure}[c]{0.28\textwidth}
\vspace*{-0.4mm}
    \setlength\figureheight{28mm} 
    \setlength\figurewidth{38mm}
\input{GR1Meas_rank5_SPF.tex}
\end{subfigure}
\hspace*{1mm}
\begin{subfigure}[c]{0.29\textwidth}
\vspace*{-1.7mm}
\vspace*{1mm}
    \setlength\figureheight{28mm} 
    \setlength\figurewidth{38mm}
\input{GR1Meas_rank5_IRLS-LRRS.tex}
\end{subfigure}

\caption{Left column: \texttt{RiemAdaIHT}, center: \texttt{SPF}, right: \texttt{IRLS}. Success rates for the recovery of low-rank and row-sparse matrices from random rank-one measurements. First row: Rank-$1$ ground truth $\X_\star$ (cf. \Cref{fig:Performance_Rank1}. Second row: Rank-$5$ ground truth $\X_\star$ (cf. \Cref{fig:Performance_Rank5}).} 
\label{fig:RandomRankOne_Meas}
\end{figure}
We observe in \Cref{fig:RandomRankOne_Meas} that compared to the setting of dense Gaussian measurements, the phase transitions of all three algorithms deteriorate slightly; especially for $r=1$, one can observe that the transition between no success and high empirical success rate is extends across a larger area. \texttt{IRLS} performs clearly best for both $r=1$ and $r=5$, whereas \texttt{SPF} has the second best performance for $r=5$. For $r=1$, it is somewhat unclear whether \texttt{RiemAdaIHT} or \texttt{SPF} performs better.

\subsection{Discrete Fourier Rank-One Measurements}
\label{sec:Fourier:measurements}
We now revisit the experiments of \Cref{sec:rankone:measurements} for a third measurement setup motivated from blind deconvolution problems \cite{ahmed2014blind,lee2016blind,Li-RapidBlindDeconvolution2019,Ma-Implicit2020,Shi-ManifoldBlindDeconvolution2021,eisenmann2021riemannian}, which are prevalent in astronomy, medical imaging and communications engineering \cite{Jefferies-Restoration1993,Campisi-Blind2017}. In particular, in these settings, if $\f{z} \in \R^{m}$ is an (unknown) signal and $\f{w} \in \R^{m}$ is an (unknown) convolution kernel, assume we are given the entries of their convolution $\widetilde{y} = \f{z} \ast \f{w}$. If we know that $\f{z} = \f{A}\f{u}$ for some known matrix $\f{A} \in \R^{m \times n_1}$ and an $s$-sparse vector $\f{u} \in \R^{n_1}$ and $\f{w} = \f{B} \f{v}$ for some known matrix $\f{B} \in \R^{m \times n_2}$ and arbitrary vector $\f{v} \in \R^{n_2}$, applying the discrete Fourier transform (represented via the DFT matrix $\f{F} \in \C^{m \times m}$), we can write the coordinates of 
\[
\f{y} = \f{F}\widetilde{y} = \diag(\f{F} \f{z}) \f{F}\f{w} = \diag(\f{F}\f{A}\f{u}) \f{F}\f{B} \f{v}
\]
as
\[
\f{y}_j = \A(\f{u} \f{v}^*)_j =  \langle (\f{F}\f{A})_{j,:}^* \overline{\f{F}\f{B}}_{j,:} , \f{u} \f{v}^* \rangle_F
\]
for each $j = 1,\ldots,m$, which allows us to write the problem as a simultaneously rank-$1$ and $s$-row sparse recovery problem from Fourier-type measurements.

\begin{figure}[h]
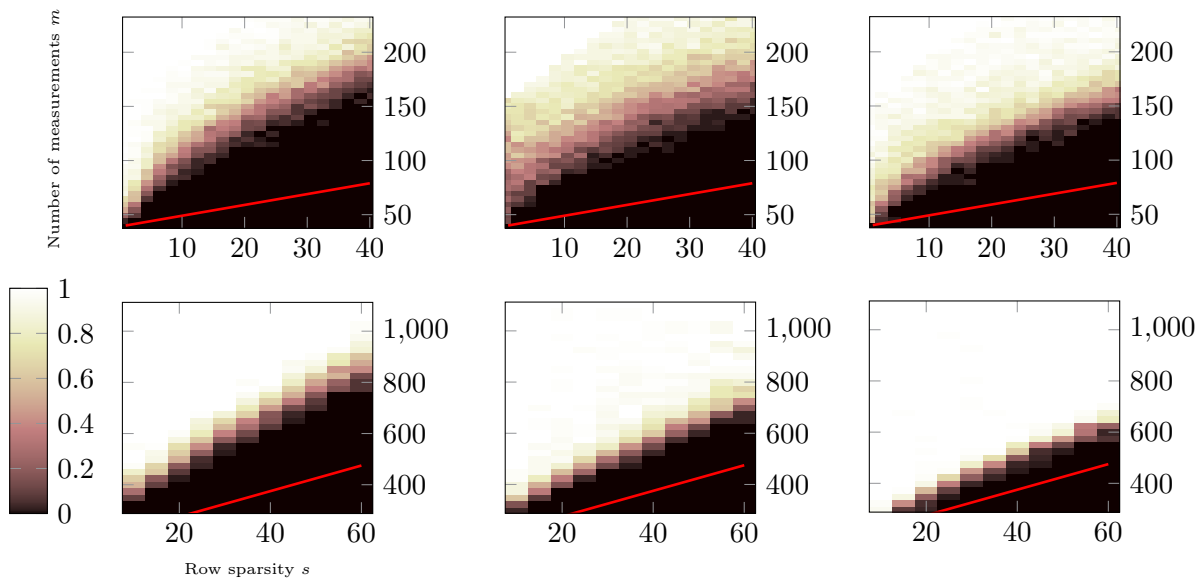

\begin{subfigure}[c]{0.345\textwidth}
\vspace*{-1mm}
 \hspace*{0.25mm}
    \setlength\figureheight{28mm} 
    \setlength\figurewidth{38mm}
\input{FourierR1Meas_rank1_RAdaIHT.tex}
\end{subfigure}
\begin{subfigure}[c]{0.31\textwidth}
 \hspace*{4mm}
 \vspace*{-1.3mm}
    \setlength\figureheight{28mm} 
    \setlength\figurewidth{38mm}
\input{FourierR1Meas_rank1_SPF.tex}
\end{subfigure}
\hspace*{3mm}
\begin{subfigure}[c]{0.31\textwidth}
\vspace*{1.2mm}
    \setlength\figureheight{28mm} 
    \setlength\figurewidth{38mm}
\input{FourierR1Meas_rank1_IRLS-LRRS.tex}
\end{subfigure}
\begin{subfigure}[c]{0.36\textwidth}
\hspace*{-2.4mm}
    \setlength\figureheight{28mm} 
    \setlength\figurewidth{38mm}
\input{FourierR1Meas_rank5_RAdaIHT.tex}
\end{subfigure}
\hspace*{3.75mm}
\begin{subfigure}[c]{0.28\textwidth}
\vspace*{-0.4mm}
    \setlength\figureheight{28mm} 
    \setlength\figurewidth{38mm}
\input{FourierR1Meas_rank5_SPF.tex}
\end{subfigure}
\hspace*{-0.6mm}
\begin{subfigure}[c]{0.29\textwidth}
\vspace*{-1.7mm}
\vspace*{1mm}
    \setlength\figureheight{28mm} 
    \setlength\figurewidth{38mm}
\input{FourierR1Meas_rank5_IRLS-LRRS.tex}
\end{subfigure}
\caption{Left column: \texttt{RiemAdaIHT}, center: \texttt{SPF}, right: \texttt{IRLS}. Success rates for the recovery of low-rank and row-sparse matrices from Fourier rank-one measurements. First row: Rank-$1$ ground truth $\X_\star$. Second row: Rank-$5$ ground truth $\X_\star$ (cf. \Cref{fig:Performance_Rank5}).} 
\label{fig:Fourier_Meas}
\end{figure}

In \Cref{fig:Fourier_Meas}, we report the results of simulations with $\f{A}$ and $\f{B}$ chosen generically as standard real Gaussians for these Fourier-based rank-$1$ measurements (including for rank-$5$ ground truths, which goes beyond a blind deconvolution setting). We observe that the transition from no recovery to exact recovery for an increasing number of measurement (with fixed dimension parameters $s$, $n_1$ and $n_2$) happens earlier than for the random Gaussian rank-one measurements of \Cref{sec:rankone:measurements}, but slightly later than for dense Gaussian measurements. Again, \texttt{IRLS} exhibits the best empirical data-efficiency with sharpest phase transition curves.

As a summary, we observe that \texttt{IRLS} is able to recovery simultaneously low-rank and row-sparse matrices empirically from fewer measurements than state-of-the-art methods for a variety of linear measurement operators, including in cases where the RIP assumption of \Cref{def:RIP} is not satisfied and in cases that are relevant for applications.

\subsection{Evolution of Objective Values} \label{sec:evolution:objective}
While \Cref{thm:QuadraticConvergence} guarantees local convergence if the measurement operator $\A$ is generic enough and contains enough measurements (RIP-assumption), it is instructive to study the behavior of \Cref{algo:MatrixIRLS} in situations where there are \emph{not} enough measurements available to identify a specific low-rank and row-sparse ground truth $\X_\star$ which respect to which the measurements have been taken. 

In this setting, \Cref{thm:IRLS:majorization:MM} guarantees that the behavior of the \texttt{IRLS} methods is still benign as the sequence of $\epsilon$- and $\delta$-smoothed log-objectives $\left(\F_{\varepsilon_k,\delta_k}(\Xk{k})\right)_{k\geq 1}$ from \cref{eq:F:objective:def} is non-increasing. In \Cref{fig:ObjectiveEvolution}, we illustrate the evolution of the relative Frobenius error of an iterate to the ground truth $\X_\star$, the $(\varepsilon_k,\delta_k)$-smoothed logarithmic surrogates $\F_{\varepsilon_k,\delta_k}(\Xk{k})$ as well as of the rank and row-sparsity parts $\Flrk{k}(\Xk{k})$ and $\Fspk{k}(\Xk{k})$ of the objective, respectively, in two typical situations. 

In particular, we can see the evolution of these four quantities in the setting of data of dimensionality $n_1=128$, $n_2 \in \{20,40\}$, $s=20$ and $r=5$ created as in the other experiments, where a number of $m=875$ and $m=175$ (corresponding to an oversampling factor of $3.0$ and $1.0$, respectively) dense Gaussian measurements are provided to \Cref{algo:MatrixIRLS}.

\begin{figure}[h]
\begin{subfigure}[c]{1\textwidth}
\vspace*{4mm}
    \setlength\figureheight{50mm} 
    \setlength\figurewidth{65mm}
%
%
\definecolor{mycolor1}{rgb}{0.00000,0.44700,0.74100}%
\definecolor{mycolor2}{rgb}{0.85000,0.32500,0.09800}%
\definecolor{mycolor3}{rgb}{0.92900,0.69400,0.12500}%
\definecolor{mycolor4}{rgb}{0.49400,0.18400,0.55600}%
\begin{tikzpicture}

\begin{axis}[%
width=0.951\figurewidth,
height=\figureheight,
at={(0\figurewidth,0\figureheight)},
scale only axis,
xmin=0,
xmax=12,
xlabel style={font=\color{white!15!black}},
xlabel={iteration $k$},
ymode=log,
ymin=1e-14,
ymax=100,
yminorticks=true,
axis background/.style={fill=white}, 
legend style={at={(0.3,1.05)}, anchor=south west, legend cell align=left, align=left, draw=white!15!black, font=\tiny, legend columns = 2},
xlabel style={font=\tiny},ylabel style={font=\tiny},
]
\addplot [color=mycolor1, mark size=4.0pt, mark=x, mark options={solid, mycolor1}]
  table[row sep=crcr]{%
1	0.908957122304221\\
2	0.660395812170821\\
3	0.467398778374254\\
4	0.288343883883047\\
5	0.181861531199963\\
6	0.0912688847742466\\
7	0.0330267601215323\\
8	0.0173193082206571\\
9	0.00478086055360586\\
10	0.000462809725872112\\
11	3.69435994800441e-07\\
12	2.25272566609678e-12\\
};
\addlegendentry{Rank Objective $\sqrt{\mathcal{F}_{lr,\varepsilon_k}(\mathbf{X}^{(k)})}$}

\addplot [color=mycolor2, mark size=4.0pt, mark=+, mark options={solid, mycolor2}]
  table[row sep=crcr]{%
1	0.545888565192268\\
2	0.426197862838523\\
3	0.35098298949931\\
4	0.206950890256743\\
5	0.129459215066694\\
6	0.0591273930202549\\
7	0.017956306323623\\
8	0.0101563281144222\\
9	0.00239019674790693\\
10	5.45290987229027e-05\\
11	2.09708775309565e-07\\
12	4.3992092158033e-13\\
};
\addlegendentry{Sparsity Objective $\sqrt{\mathcal{F}_{sp,\delta_k}(\mathbf{X}^{(k)})}$}

\addplot [color=mycolor3, mark size=4.0pt, mark=asterisk, mark options={solid, mycolor3}]
  table[row sep=crcr]{%
1	1.06028174359235\\
2	0.785981709087993\\
3	0.584508919473105\\
4	0.354923747234274\\
5	0.223233744976156\\
6	0.108747680128389\\
7	0.0375924968698944\\
8	0.0200775854626478\\
9	0.00534506016118898\\
10	0.00046601101378546\\
11	4.24806690973705e-07\\
12	2.29527848940716e-12\\
};
\addlegendentry{IRLS Objective $\sqrt{\mathcal{F}_{\varepsilon_k,\delta_k}(\mathbf{X}^{(k)})}$}

\addplot [color=mycolor4, mark size=4.0pt, mark=o, mark options={solid, mycolor4}]
  table[row sep=crcr]{%
1	0.913106509790089\\
2	0.757275769947384\\
3	0.517673901317392\\
4	0.328935299390968\\
5	0.189917620955421\\
6	0.0911879292870496\\
7	0.0277563870342638\\
8	0.0146070679581609\\
9	0.00312419360977059\\
10	4.66559490611061e-05\\
11	1.99836420787947e-07\\
12	1.17518638256287e-13\\
};
\addlegendentry{Rel. Frob. error $\|\mathbf{X}^{(k)} - \mathbf{X}_\star\|_F/\|\mathbf{X}_\star\|_F$}

\end{axis}
\end{tikzpicture}%
\hspace*{-48mm}
    \setlength\figureheight{50mm} 
    \setlength\figurewidth{65mm}
\input{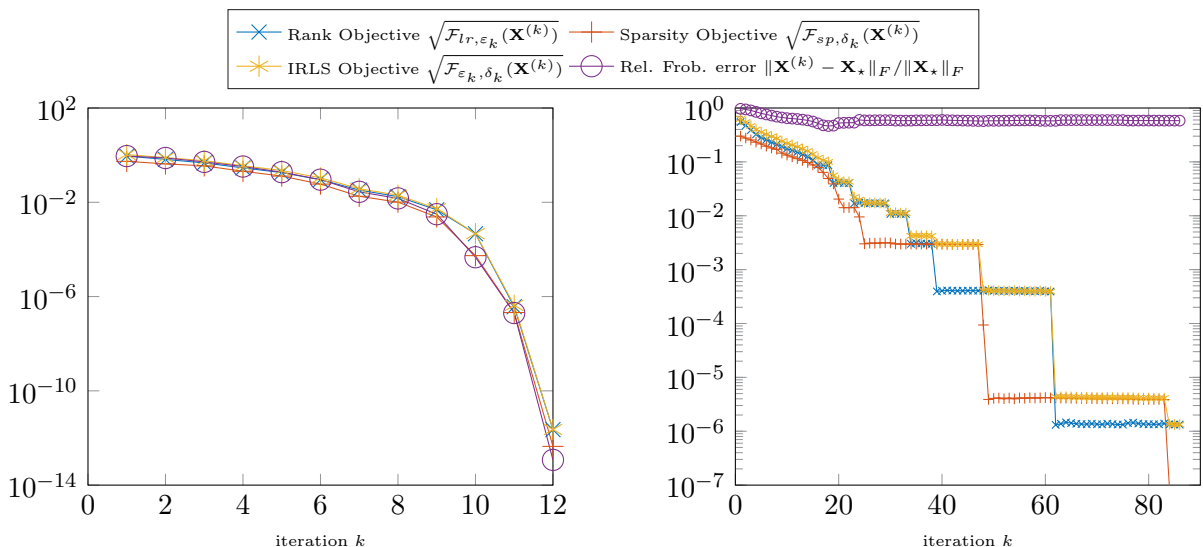}
\end{subfigure}
\caption{Objective/ error quantities of iterates $\Xk{k}$ for iterations $k$. Left: Typical result for $n_1=128$, $n=40$, $m=875$. Right: Typical result for $n_1=128$, $n=20$, $m=175$.} 
\label{fig:ObjectiveEvolution}
\end{figure}

In the left plot of \Cref{fig:ObjectiveEvolution}, which corresponds to setting of abundant measurements, we observe that the four quantities all track each other relatively well on a semilogarithmic scale (note that we plot the square roots of the objective values to match the order of the (unsquared) relative Frobenius error), converging to values between $10^{-13}$ and $10^{-11}$ (at which point the stopping criterion of the method applies) within $12$ iterations. 

In the second plot of \Cref{fig:ObjectiveEvolution}, the number of measurements exactly matches the number of degrees of freedom of the ground truth, in which case the $\Xk{k}$ does \emph{not} converge to $\X_{\star}$. However, it can be seen that \Cref{algo:MatrixIRLS} still finds very meaningful solutions: It can be seen that within $86$ iterations, $\F_{\varepsilon_k,\delta_k}(\Xk{k})$ converges to $\approx 10^{-12}$ (since $\sqrt{\F_{\varepsilon_k,\delta_k}(\Xk{k})} \approx 10^{-6}$) in a manner that is partially ``staircase-like``: After $20$ initial iterations where $\F_{\varepsilon_k,\delta_k}(\Xk{k})$ decreases significantly at each iteration, its decrease is dominated by relatively sudden, alternating drops of the (blue) sparsity objective $\Fspk{k}(\Xk{k})$ and the (red) rank objective $\Flrk{k}(\Xk{k})$, which typically do not occur simultaneously. 

This illustrates the \emph{self-balancing} property of the two objective terms in the IRLS objective $\F_{\varepsilon_k,\delta_k}(\Xk{k})$: while the final iterate at iteration $k=86$ is not of the target row-sparsity $s=20$ and $r=5$, it is still $20$-row sparse and has essentially rank $6$. This means that \Cref{algo:MatrixIRLS} has found an alternative parsimonious solution to the simultaneous low-rank and row-sparse recovery problem that is just slightly less parsimonious. 

Arguably, this robust performance in the low-data regime of \texttt{IRLS} is rather unique, and to the best of our knowledge, not shared by methods such as \texttt{SPF} or \texttt{RiemAdaIHT}, which typically breakdown in such a regime.

\subsection{Robustness under Noisy Measurements} \label{sec:noise:robustness}
The convergence theory for the \texttt{IRLS} method \Cref{algo:MatrixIRLS}  established in \Cref{thm:QuadraticConvergence} assume that \emph{exact} linear measurements $\y = \A(\X_\star)$ of a row-sparse and low-rank ground truth $\X_\star$ are provided to the algorithm. However, in practice, one would expect that the linear measurement model is only approximately accurate. For IRLS for sparse vector recovery, theoretical guarantees have been established for this case in \cite{Daubechies10,Lai-SIAM-J-NA2013}. We do not extend such results to the simultaneously structured case, but we provide numerical evidence that \texttt{IRLS} as defined in \Cref{algo:MatrixIRLS} can be used directly also for noisy measurements.

To this end, we conduct an experiment in the problem setup of \Cref{fig:Performance_Rank1} in \Cref{sec:Numerics} for a fixed row-sparsity of $s=40$, in which the measurements provided to the algorithms \texttt{IRLS}, \texttt{RiemAdaIHT} and \texttt{SPF} are such that
\[
\y = \A(\X_\star) + \mathbf{w},
\]
where $\mathbf{w}$ is a Gaussian vector (i.i.d. entries) with standard deviation of $\sigma = \sqrt{ \frac{\|\A(\X_\star)\|_2^2}{m \cdot \operatorname{SNR}} }$ and where $\operatorname{SNR}$ is a varying signal-to-noise ratio. We consider SNRs between $10$ and $10^{12}$, and report the resulting relative Frobenius error statistics in \Cref{fig:NoisyExperiments}. 

\begin{figure}[h]
    \setlength\figureheight{50mm} 
    \setlength\figurewidth{134mm}
%
%
\definecolor{mycolor1}{rgb}{0.00000,0.44700,0.74100}%
\definecolor{mycolor2}{rgb}{0.85000,0.32500,0.09800}%
\definecolor{mycolor3}{rgb}{0.92900,0.69400,0.12500}%
\begin{tikzpicture}

\begin{axis}[%
width=0.947\figurewidth,
height=\figureheight,
at={(0\figurewidth,0\figureheight)},
scale only axis,
xmode=log,
xmin=10,
xmax=1000000000000,
xminorticks=true,
xlabel style={font=\color{white!15!black}},
xlabel={Signal-to-noise ratio SNR},
ymode=log,
ymin=1e-07,
ymax=1,
yminorticks=true,
ylabel style={font=\color{white!15!black}},
ylabel={relative Frobenius reconstruction error $\|\hat{\mathbf{X}}-\mathbf{X}_0\|_{F}/\|\mathbf{X}_0\|_{F}$},
axis background/.style={fill=white},
yticklabel pos=left,
legend style={legend cell align=left, align=left, draw=white!15!black},
xlabel style={font=\tiny},ylabel style={font=\tiny},
]
\addplot [color=mycolor1, line width=1.0pt]
 plot [error bars/.cd, y dir = both, y explicit]
 table[row sep=crcr, y error plus index=2, y error minus index=3]{%
10	0.348650890403048	0.0330578476866618	0.0284318554904581\\
100	0.0910089755689084	0.00796739163324166	0.0111550766262097\\
1000	0.024722941938275	0.00222200029677052	0.00122449899581323\\
10000	0.00749333859715644	0.000567020092563629	 0.000570496616301164\\
100000	0.002331717440457	0.000168065312607915	 0.00017882830232946\\
1000000	0.000708426236248503	 7.41531402029841e-05	5.66878960834395e-05\\
10000000	0.000227631683330547	 1.97222193063333e-05	1.85465466199373e-05\\
100000000	7.26178115683212e-05 5.09250138788165e-06 4.78351308051351e-06\\
1000000000	2.32734018216041e-05 1.66971399421388e-06 1.66301294234276e-06\\
10000000000	7.26171425376397e-06	 5.47046997265171e-07 4.40822577453967e-07\\
100000000000	2.25369307370622e-06 1.52059183004652e-07 1.68955410791571e-07\\
1000000000000	7.21117356356925e-07 5.05747013948652e-08 4.17718580727481e-08\\
};
\addlegendentry{IRLS-LowRankRowSparse}
\addplot [color=mycolor2, line width=1.0pt]
 plot [error bars/.cd, y dir = both, y explicit]
 table[row sep=crcr, y error plus index=2, y error minus index=3]{%
10	0.396770564796004	0.0486348231294986	0.0364036977566766\\
100	0.0955581136487725	0.00888453347575602	0.0105450724924156\\
1000	0.0250321502141839	0.00242154712103063	0.00182631595153982\\
10000	0.00740979279098082	0.00061884712456946	0.000647568349143837\\
100000	0.00228051744276195	0.000136580166126271	0.000183566820767888\\
1000000	0.000694206830049285	 7.48345908625392e-05	6.10337685427675e-05\\
10000000	0.000223446210436675	 1.78824139583254e-05	1.9941001666272e-05\\
100000000	7.05041269960804e-05	 4.83944901285867e-06	4.37744286326844e-06\\
1000000000	2.2633561545307e-05	1.57433222874718e-06	 1.68820335902632e-06\\
10000000000	7.07645435904116e-06 	5.24407121495868e-07	 4.42449247117077e-07\\
100000000000	2.19762301141902e-06	 1.56352630085408e-07 1.76823940094158e-07\\
1000000000000	7.05226118362792e-07 	4.9649037277558e-08	4.10987075680683e-08\\
};
\addlegendentry{Sparse Power Factorization \cite{lee2013near}}
\addplot [color=mycolor3, line width=1.0pt]
 plot [error bars/.cd, y dir = both, y explicit]
 table[row sep=crcr, y error plus index=2, y error minus index=3]{%
10	0.431734542654288	0.0633825386766134	0.0460799694274802\\
100	0.0968081742293434	0.0121404444383488	0.00910674967240065\\
1000	0.0250769207661475	0.00216623792867203	0.00190549159920542\\
10000	0.00739347812634115	0.000586711367818588	0.0006214944659304\\
100000	0.00228101088486121	0.000197503048494734	0.000183268376679519\\
1000000	0.000699382330066529 	6.86127352301289e-05 	6.24576550437562e-05\\
10000000	0.000223214628430534 	1.77714375363973e-05 	1.84124175219336e-05\\
100000000	7.08220741763661e-05	 4.71635908284678e-06	 4.83926236482959e-06\\
1000000000	2.27814302703603e-05	 1.51189380452636e-06 	1.86345411395474e-06\\
10000000000	7.10711450955656e-06	 4.61438617983076e-07	 4.79900488503965e-07\\
100000000000	2.20131493122781e-06 	1.5252229609801e-07	1.92503601743653e-07\\
1000000000000	7.12499532987756e-07 	4.65422612883578e-08	 4.65133668410877e-08\\
};
\addlegendentry{Riemannian Adaptive IHT \cite{eisenmann2021riemannian}}
\end{axis}
\begin{axis}[%
width=1.227\figurewidth,
height=1.227\figureheight,
at={(-0.16\figurewidth,-0.135\figureheight)},
scale only axis,
xmin=0,
xmax=1,
ymin=0,
ymax=1,
axis line style={draw=none},
ticks=none,
axis x line*=bottom,
axis y line*=left,
legend style={legend cell align=left, align=left, draw=white!15!black},
xlabel style={font=\tiny},ylabel style={font=\tiny},
]
\end{axis}
\end{tikzpicture}%
\caption{Median relative Frobenius reconstruction errors of different algorithms given noisy Gaussian measurements,  $n_1 = 256$, $n_2= 40$, row-sparsity $s=40$ and rank $r=1$, oversampling factor of $3$. Error bars correspond to $25\%$ and $75\%$ percentiles.} 
\label{fig:NoisyExperiments}
\end{figure}
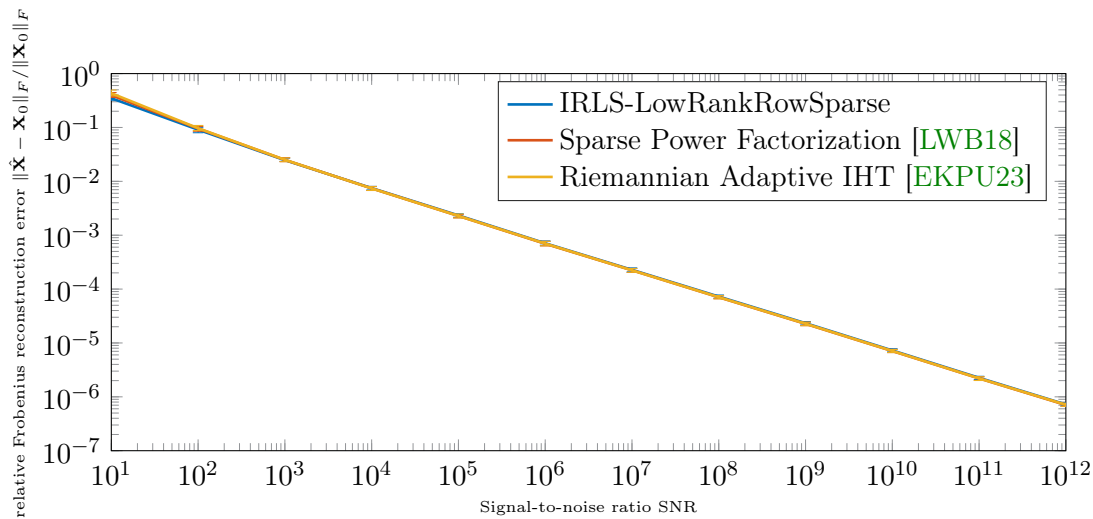

We observe that the reconstruction error is consistently roughly proportional to the inverse square root of the signal-to-noise ratio, for all three algorithms considered. This suggests that \texttt{IRLS} is as noise robust as comparable algorithms, and expected to be return estimates of the original ground truth that has a reconstruction error that is of the order of the norm of the noise.

\section{Proofs} \label{sec:proofs}
The following two sections contain the proofs of our main results. Let us begin with some helpful observations. 

First note that the low-rank promoting part $\Wlr{k}: \Rnn \to \Rnn$ of our weight operator can be re-written as
\begin{equation} 
\label{eq:W:operator:action2}
    \Wlr{k}(\Z) = \begin{bmatrix}
    \U & \U_\perp 
    \end{bmatrix} \round{ \Hk{k} \circ \round{
    \begin{bmatrix}
    \U^* \\ \U_\perp^* 
    \end{bmatrix}
    \Z 
    \begin{bmatrix}
    \V & \V_\perp 
    \end{bmatrix}
    } }
    \begin{bmatrix}
    \V^* \\ \V_\perp^* 
    \end{bmatrix},
\end{equation}
where
\begin{align*}
    \Hk{k} &:= \left[ \min\left(\varepsilon_k/\sigma_i^{(k)},1\right)  \min\left(\varepsilon_k/\sigma_j^{(k)},1\right)\right]_{i,j=1}^{n_1,n_2} \\
    &=
    \square{
    \begin{array}{c|c}
    \round{\frac{\varepsilon_k^2}{\sigma_i^{(k)} \sigma_j^{(k)}}}_{i,j = 1}^{r_k} & 
    \round{\frac{\varepsilon_k}{\sigma_i^{(k)}}}_{i,j = 1}^{r_k,d_2} \\ \hline 
    \round{\frac{\varepsilon_k}{ \sigma_j^{(k)}}}_{i,j = 1}^{d_1,r_k} &
    \mathbf{1}
    \end{array}
    } \in \R^{n_1 \times n_2}.
\end{align*}

Consequently, all weight operators in \Cref{def:weight:operator} are self-adjoint and positive. Whereas for $\Wsp{k}$ this is obvious, for $\Wlr{k}$ it follows from the matrix representation
$$
\Wlr{k} = \Big(\begin{bmatrix}
    \U & \U_\perp 
    \end{bmatrix} \otimes     \begin{bmatrix}
    \V & \V_\perp 
    \end{bmatrix} \Big) 
    \f{D}_{\Hk{k}} 
    \Big(\begin{bmatrix}
    \U & \U_\perp 
    \end{bmatrix}^* \otimes     \begin{bmatrix}
    \V & \V_\perp 
    \end{bmatrix}^*\Big) 
$$ 
where $\f{D}_{\Hk{k}} \in \R^{n_1 n_2 \times n_1 n_2}$ is a diagonal matrix with the entries of $\Hk{k}$, which are all positive, on its diagonal.

\subsection{Proof of \Cref{thm:QuadraticConvergence}}
\label{sec:ProofOfQuadraticConvergence}

Before approaching the proof of \Cref{thm:QuadraticConvergence}, let us collect various important observations. In order to keep the presentation concise, we defer part of the proofs to \Cref{sec:TechnicalAddendum}.

For a rank-$r$ matrix $\Z = \U\SIGMA\V^*$, we define the tangent space of the manifold of rank-$r$ matrices at $\Z$ as
\begin{align} \label{eq:def:TUV}
    \TT{\U,\V} := \{ \U\Z_1^* + \Z_2\V^* \colon \Z_1 \in \R^{n_2\times r}, \Z_2 \in \R^{n_1\times r} \}.
\end{align}
In a similar manner, we can define for $\Z = \U\SIGMA\V^* \in \MM{r,s}$ and $S = \supp(\Z) = \supp(\U) \subset [n_1]$ the tangent space of $\MM{r}$ restricted to $S$ as
\begin{align} \label{eq:TUVS}
    \TT{\U,\V,S} := \{ \U\Z_1^* + \Z_2\V^* \colon \Z_1 \in \R^{n_2\times r}, \Z_2 \in \R^{n_1\times r} \text{ with } \supp(\Z_2) = S \}.
\end{align}
As the following lemma shows, orthogonal projections onto the sets $\MM{r}^{n_1,n_2}$, $\NN{s}^{n_1,n_2}$, $\TT{\U,\V}$, and $\TT{\U,\V,S}$ can be efficiently computed.

\begin{lemma} \label{lem:Projection}
    We denote the projection operators onto $\MM{r}^{n_1,n_2}$ and $\NN{s}^{n_1,n_2}$ by $\T_r$ and $\H_s$. $\T_r$ truncates a matrix to the $r$ dominant singular values; $\H_s$ sets all but the $s$ in $\ell_2$-norm largest rows to zero. In case of ambiguities (multiple singular values/rows of same magnitude), by convention we choose the $r$ (respectively $s$) with smallest index.
    
    For $\U$ and $\V$ fixed, the orthogonal projection onto $\TT{\U,\V}$ is given by
    \begin{align*}
        \P_{\U,\V} := \P_{\TT{\U,\V}} \Z = \U\U^* \Z + \Z \V\V^* - \U\U^* \Z \V\V^*.
    \end{align*}
    For $S \subset [n_1]$ and $\U,\V$ fixed with $\supp(\U) = S$, the orthogonal projection onto $\TT{\U,\V,S}$ is given by
    \begin{align*}
        \P_{\U,\V,S} := \P_{\TT{\U,\V,S}} \Z
        &= \P_S(\U\U^* \Z + \Z \V\V^* - \U\U^* \Z \V\V^* ) \\
        &= \U\U^* \Z + \P_S \Z \V\V^* - \U\U^* \Z \V\V^*,
    \end{align*}
    where $\P_S$ projects to the row support $S$, i.e., it sets all rows to zero which are not indexed by $S$.
\end{lemma}

The proof of \Cref{lem:Projection} is provided in \Cref{sec:ProofOfLemProjection}. In contrast to the above named projections, the projection onto $\MM{r,s}^{n_1,n_2}$ is not tractable. However, \cite[Lemma 2.4]{eisenmann2021riemannian} shows that locally $\P_{\MM{r,s}}$ can be replaced by the concatenation of $\T_r$ and $\H_s$, i.e., for $\Z_\star \in \MM{r,s}$ and $\Z \approx \Z_\star$, one has that 
\begin{align*}
    \P_{\MM{r,s}} (\Z) = \T_r ( \H_s (\Z) ).
\end{align*}

For a matrix $\X\in \R^{n_1\times n_2}$ and $i \in [n_1]$, we set $\rho_{i} (\X) = \| (\X)_{i',:}\|_2$ where $i'$ is a row index corresponding to the $i$-th largest row of $\X$ in $\ell_2$-norm. More precisely, if $\overline\X$ is a decreasing rearrangement of $\X$ with rows ordered by magnitude in $\ell_2$-norm, then $\rho_{i} (\X) = \| (\overline\X)_{i,:}\|_2$. As the following lemma shows, the quantity $\rho_{s} (\X)$ determines a local neighborhood of $\X$ on which $\H_s$ preserves the row-support. 

\begin{lemma} \label{lem:SameSupport}
    Let $\X \in \R^{n_1 \times n_2}$ be a matrix with row-support $S \subset [n_1]$ and $|S| = s$. Then, for any $\Z \in \R^{n_1 \times n_2}$ with $\| \X - \Z \|_{\infty,2} := \max_{i\in [n_1]} \| \X_{i,:} - \Z_{i,:} \|_2 \le \frac{1}{2} \rho_{s} (\X)$ the matrix $\H_s(\Z)$ has row-support $S$.
\end{lemma}

\begin{proof}
    Note that
   \[
   \max_{i\in [n_1]} \| (\Z)_{i,:} - (\X)_{i,:} \|_2 
   \leq  \frac{1}{2} \rho_s(\X)
   \]
    implies that any non-zero row of $\X$ corresponds to a non-zero row of $\H_s(\Z)$ and hence yields the claim.
\end{proof}

A first important observation is that if $\A$ has the $(r,s)$-RIP, then the norm of kernel elements of $\A$ is bounded in the following way.

\begin{lemma} \label{lem:RIP}
    If $\A$ has the $(r,s)$-RIP with $\delta \in (0,1)$ and $\U \in \R^{n_1\times r}$,$\V \in \R^{n_2\times r}$ with $\supp(\U) = S$, $|S| \le s$, then
    \begin{align*}
        \norm{\XI}{F} \le \sqrt{1 + \frac{\norm{\A}{2\to 2}^2}{(1-\delta)}} \norm{\P_{\TT{\U,\V,S}}^\perp (\XI)}{F},
    \end{align*}
    for all $\XI \in \ker(\A)$.
\end{lemma}

The proof of \Cref{lem:RIP} is presented in \Cref{sec:ProofOfLemRIP}.

\begin{remark}
    If $\A$ is a Gaussian operator with standard deviation $\sqrt{\frac{1}{m}}$, one has with high probability that $\norm{\A}{2\to 2}^2 \approx \frac{n_1 n_2}{m}$.
\end{remark}

We use of the following lemma to characterize the solution of the weighted least squares problem \cref{eq:MatrixIRLS:Xdef}. Its proof is analogous to \cite[Lemma B.7]{KM21} and \cite[Lemma 5.2]{Daubechies10}.

\begin{lemma} \label{lemma:leastsquares:opt}
Let $\A: \Rnn \to \R^m$ and $\f{y} \in \R^m$. Let $\Wk{k}:\Rnn \to \Rnn$ be the weight operator \cref{eq:DefW} defined based on the information of  $\f{X}\hk \in \Rnn$. Then the solution of the weighted least squares step \cref{eq:MatrixIRLS:Xdef} of \Cref{algo:MatrixIRLS}
\begin{equation} \label{eq:IRLS:Xk:optimality:explicitformula}
\f{X}\hkk =\argmin\limits_{\A(\f{X})=\f{y}} \langle \f{X},\Wk{k}(\f{X}) \rangle,
\end{equation}
is unique and solves \cref{eq:IRLS:Xk:optimality:explicitformula} if and only if
\begin{equation} \label{eq_proof_JminX_1}
\A(\f{X}\hkk) = \f{y}
\qquad \text{ and } \qquad
\langle \Wk{k}(\f{X}\hkk),\XI \rangle = 0 \; \text{ for all } \; \XI \in \ker \A. 
\end{equation}
\end{lemma}

For any iterate $\Xk{k}$ of \Cref{algo:MatrixIRLS}, we furthermore abbreviate the tangent space \cref{eq:TUVS} of the fixed rank-$r$ manifold $\MM{r}$ restricted to $S$ at $\H_{s}(\Xk{k})$ by
\begin{align} \label{eq:Tk} 
    T_k = T_{\tU,\tV,\tSk{k}},
\end{align}
where $\tU \in \R^{n_1 \times r}$ and $\tV \in \R^{n_2 \times r}$ are matrices with leading\footnote{As $\tU$ and $\tV$ might not be unique, any set of $r$ leading singular vectors can be chosen in this definition.} $r$ singular vectors of $\H_{s}(\Xk{k})$ as columns, and $S \in [n_1]$ is the support set of the $s$ rows of $\Xk{k}$ with largest $\ell_2$-norm.

The following lemma is the first crucial tool for showing local quadratic convergence of Algorithm \ref{algo:MatrixIRLS}.

\begin{lemma} \label{lemma:B8:generalization}
   Let $\X_\star \in \MM{r,s}$ and let $\Xk{k}$ be the $k$-th iterate of \Cref{algo:MatrixIRLS} with rank and sparsity parameters $\widetilde{r} = r$ and $\widetilde{s} = s$, let $\delta_k, \varepsilon_k$ be such that $s_k$ and $r_k$ from \Cref{def:weight:operator} satisfy $s_k \geq s$ and $r_k \geq r$. Assume that there exists a constant $c > 1$  such that
   \begin{align} \label{eq:lemma:B8:RIP}
       \norm{\XI}{F} \le c \norm{\P_{T_k^\perp} (\XI)}{F}       \qquad \text{for all } \quad \XI \in \ker(\A),
   \end{align}
   where $T_k = T_{\tU,\tV,\tSk{k}}$ is as defined in \cref{eq:Tk} for matrices $\tU \in \R^{n_1 \times r}$ and $\tV \in \R^{n_2 \times r}$ of leading $r$ left and right singular vectors of $\H_{s}(\Xk{k})$ and $S \subset [n_1]$ is the support set of $\H_{s}(\Xk{k})$.
   Assume furthermore that
   \begin{align} \label{eq:lemma:B8:closeness}
        \specnorm{\Xk{k} - \X_\star} \le \min \left\{ \frac{1}{2} \rho_s(\X_\star), \min\Big\{\frac{1}{48},\frac{1}{19 c }\Big\} \sigma_r(\X_\star) \right\}.
   \end{align}
   Then,
   \begin{align*}
       \norm{\Xk{k+1} - \X_\star}{} &\le 4 c^2 \min\left\{\frac{\sigma_{r+1} (\Xk{k})}{\varepsilon_k},\frac{\rho_{s+1}(\Xk{k})}{\delta_k}  \right\}^2 \\
       &\qquad \cdot \left(  \norm{\Wlr{k}(\X_\star)}{*} +  \norm{\Wsp{k} \cdot \X_\star}{1,2} \right),
   \end{align*}
   where $\norm{\M}{1,2} = \sum_{i} \norm{\M_{i,:}}{2}$ denotes the row-sum norm of a matrix $\M$, and $\Wlr{k}$ and $\Wsp{k}$ are the weight operators \cref{eq:W:operator:action} and \cref{eq:Ws_def} from \Cref{def:weight:operator}.
\end{lemma}

The proof of \Cref{lemma:B8:generalization} is presented in \Cref{sec:ProofOfLemB8}.

\begin{remark} \label{rem:B8:generalization}
    By revisiting the proof of \Cref{lemma:B8:generalization} (omit the bound in \cref{eq:InnerProductBound} and keep the term $\langle \XI, \bar W \XI \rangle$ until the end), one can show under the same assumptions as in \Cref{lemma:B8:generalization} that
    \begin{equation*} 
       \| \XI \|_F^2 
       \le 4 c^2 \min\left\{\frac{\sigma_{r+1} (\Xk{k})}{\varepsilon_k},\frac{\rho_{s+1}(\Xk{k})}{\delta_k}\right\}^2 \Big\langle \XI, \big(\P_{\U,\V}^\perp \Wlr{k}  \P_{\U,\V}^\perp + \P_{\tSk{k}^c} \Wsp{k}  \P_{\tSk{k}^c}\big)   \XI \Big\rangle,
    \end{equation*} 
    where $\U$ and $\V$ are containing the left and right singular vectors of $\Xk{k}$, see \Cref{def:weight:operator}.
\end{remark}

The contribution of the norms of the weighted $\X_\star$ terms in Lemma \ref{lemma:B8:generalization} can be controlled by Lemmas \ref{lemma:MatrixIRLS:weight:nuclearnormbound} and \ref{lemma:MatrixIRLS:weight:l1bound} below.

\begin{lemma} \label{lemma:MatrixIRLS:weight:nuclearnormbound}
	Let $\Wlr{k} \colon \Rnn \to \Rnn$ be the rank-based weight operator \cref{eq:W:operator:action} that uses the spectral information of  $\Xk{k}$ and let $\X_\star \in \Rnn$ be a rank-$r$ matrix. Assume that there exists $0 < \zeta < \frac{1}{2}$ such that
\begin{equation} \label{eq:MatrixIRLS:closeness:assumption:B8}
\max\{\varepsilon_k ,\|\Xk{k} - \X_{\star}\| \} \leq \zeta \sigma_r(\X_\star).
\end{equation}	
	Then for each $1 \leq q \leq \infty$,
	\[
	\Sp{\Wlr{k}(\X_\star)}{q}  \leq \frac{r^{1/q}}{(1-\zeta) \sigma_r(\X_\star)} \left(\frac{1}{1-\zeta}\varepsilon_k^2  +  \varepsilon_k K_q \|\Xk{k} - \X_{\star}\| + 2 \|\Xk{k} - \X_{\star}\|^2  \right)
	\]
	and 
	\[
	\Sp{\Wlr{k}(\X_\star)}{q}  \leq \frac{1}{(1-\zeta) \sigma_r(\X_\star)} \left(\frac{r^{1/q}}{1-\zeta} \varepsilon_k^2  + 2 \Sp{\Xk{k} - \X_{\star}}{q}\left( \varepsilon_k + \|\Xk{k} - \X_{\star}\|  \right)\right)
	\]
	where $K_q$ is such that $K_q = 2^{1/q}$ for $1 \leq q \leq 2$ and $4 \leq q$, $K_q = \sqrt{2}$ for $2 < q \leq 4$ and $K_q =1$ for $q= \infty$.
\end{lemma}

\begin{lemma}
\label{lemma:MatrixIRLS:weight:l1bound}
Let $\Wsp{k} \in \R^{n_1\times n_1}$ be the row-sparsity-based weight operator \cref{eq:Ws_def} that uses the current iterate $\Xk{k}$ with $\delta_k=\min\left(\delta_{k-1}, \rho_{s+1}(\Xk{k})\right)$ and let $\X_\star \in \R^{n_1\times n_2}$ be an $s$-row-sparse matrix. Assume that there exists $0 < \zeta < \frac{1}{2}$ such that
\begin{align} \label{eq:l1boundAssumption}
    \| \Xk{k} - \X_\star \|_{\infty,2} = \max_{i\in [n_1]} \| (\Xk{k})_{i,:} - (\X_\star)_{i,:} \|_2 \le \zeta \rho_{s}(\X_\star),
\end{align}
where $\rho_s(\M)$ denotes the $\ell_2$-norm of the in $\ell_2$-norm $s$-largest row of $\M$. Then
\begin{align*}
    \| \Wsp{k}~\cdot~\X_\star \|_{1,2} \le \frac{s \delta_k^2}{(1-\zeta)^2 \rho_{s}(\X_\star)}
\end{align*}
\end{lemma}

\Cref{lemma:MatrixIRLS:weight:nuclearnormbound} is a refined version of \cite[Lemma B.9]{KM21} the proof of which we omit here.\footnote{This result is a technical result of an unpublished paper. In this paper, we only use that result as a tool. If the reviewers think that adding the proof is relevant here, we are happy to provide it.} The proof of \Cref{lemma:MatrixIRLS:weight:l1bound} is provided in \Cref{sec:ProofOfLemL1Bound}.
Finally, the following lemma will allow us to control the decay of the \texttt{IRLS} parameters $\delta_k$ and $\varepsilon_k$.

\begin{lemma}[{\cite[Lemma B.5]{KM21}}]
\label{lem:B5}
    Let $\X_\star \in \MM{r,s}$, assume that $\A$ has the $(r,s)$-RIP with $\delta \in (0,1)$, and let us abbreviate $n = \min\{n_1,n_2\}$. 

    Assume that the $k$-th iterate $\Xk{k}$ of Algorithm \ref{algo:MatrixIRLS} with $\widetilde{r} = r$ and $\widetilde{s} = r$ updates the smoothing parameters in \cref{eq:MatrixIRLS:epsdef} such that one of the statements  $\varepsilon_k = \sigma_{r+1} (\Xk{k})$ or $\delta_k =\rho_{s+1}(\Xk{k}) $ is true, and that $r_k \geq r$ and $s_k \geq s$. Furthermore, let
    \begin{align*}
        \varepsilon_k \leq   \frac{1}{48} \sigma_r(\X_\star)
    \end{align*}
    with $c_{\norm{\A}{2\to 2}} = \sqrt{1 + \tfrac{\norm{\A}{2\to 2}^2}{(1-\delta)}}$, let $\XIk{k} := \Xk{k} - \X_\star$ satisfy
    \begin{align} \label{eq:lemma:B5:closeness}
        \| \XIk{k} \| \le \min \Big\{ \frac{1}{2} \rho_s(\X_\star), \min\Big\{ \frac{1}{48}, \frac{1}{21 c_{\norm{\A}{2\to 2}}} \Big\} \sigma_r (\X_\star) \Big\}.
    \end{align}
    Then
    \begin{align*}
        \| \XIk{k} \|_F   \le  2\sqrt{2} \sqrt{n} c_{\norm{\A}{2\to 2}} \sqrt{4 \varepsilon_k^2 + \delta_k^2}.
    \end{align*}
\end{lemma}

The proof of \Cref{lem:B5} is provided in \Cref{sec:ProofOfLemB5}. We finally have all the tools to prove \Cref{thm:QuadraticConvergence}. Note that \cref{eq:closeness:assumption:mainthm_Simple} implies
\begin{align} \label{eq:closeness:assumption:mainthm}
    \| \Xk{k} - \X_\star \|
    &\le \min\left\{ \frac{1}{48 c_{\norm{\A}{2\to 2}}^2} \min \left\{ \frac{\sigma_r(\X_\star)}{r}, \frac{\rho_s(\X_\star)}{s} \right\}, \frac{1}{4 \mu \sqrt{5 n} c_{\norm{\A}{2\to 2}}} \right\}, 
\end{align}
and 
\begin{align}
    \label{eq:delta:assumption:mainthm}
    \varepsilon_k \leq \frac{1}{48} \sigma_r(\X_\star)
\end{align}
which we will use in the proof below. The latter follows from the fact that for $\tilde r = r$
\begin{align*}
    \varepsilon_k = \min\left(\varepsilon_{k-1}, \sigma_{r+1}(\mathbf{X}^{(k)})\right) \leq \sigma_{r+1}(\mathbf{X}^{(k)}) \leq || \mathbf{X}^{(k)} - \X_*|| \leq \sigma_r(\X_*)/48.
\end{align*}

\begin{proof}[of \Cref{thm:QuadraticConvergence}]
    First note, that by assumption $\tilde r = r$ and $\tilde s = s$. Furthermore, since $\frac{1}{ 48  c_{\norm{\A}{2\to 2}}^2} \leq \frac{1}{2}$, the closeness assumption \cref{eq:closeness:assumption:mainthm} implies that $\H_s(\Xk{k})$ and $\X_\star$ share the same support due to \Cref{lem:SameSupport}.
    
    Let $\Xk{k}$ be the $k$-th iterate of Algorithm \ref{algo:MatrixIRLS}. Since the operator $\A \colon  \R^{n_1 \times n_2} \to \R^m$ has the $(r,s)$-RIP with $\delta \in (0,1)$, Lemma \ref{lem:RIP} yields for all $\U \in \R^{n_1\times r}$,$\V \in \R^{n_2\times r}$ with $\supp(\U) = S$, $|S| \le s$, that
    \begin{align*}
        \norm{\XI}{F} \le  c_{\norm{\A}{2\to 2}} \norm{\P_{\TT{\U,\V,S}}^\perp (\XI)}{F},
    \end{align*}
    for any $\XI \in \ker(\A)$. Furthermore, due to our assumption that $\tilde s = s$ and $\tilde r = r$, the smoothing parameter update rules in \cref{eq:MatrixIRLS:epsdef}, i.e.,  $\delta_k=\min\left(\delta_{k-1}, \rho_{s+1}(\Xk{k})\right)$ and $\varepsilon_k = \min\left( \varepsilon_{k-1}, \sigma_{r+1} (\Xk{k})\right)$, imply that $r_k \geq r$ and $s_k \geq s$ for all $k$. We can thus apply \Cref{lemma:B8:generalization} for $\XIk{k} := \Xk{k} - \X_\star$ (note at this point that \cref{eq:closeness:assumption:mainthm} implies the closeness assumption \cref{eq:lemma:B8:closeness} of \Cref{lemma:B8:generalization}) and obtain
    \begin{equation} \label{eq:XIBound}
    \begin{split}
       \norm{\XIk{k+1}}{} &= 
       \norm{\Xk{k+1} - \X_\star}{} \\
       &\le  4 c_{\norm{\A}{2\to 2}}^2  \min\left\{\frac{\sigma_{r+1} (\Xk{k})}{\varepsilon_k},\frac{\rho_{s+1}(\Xk{k})}{\delta_k}  \right\}^2 \left(  \norm{\Wlr{k}(\X_\star)}{*} +  \norm{\Wsp{k} \cdot \X_\star}{1,2} \right),
       \end{split}
   \end{equation}
   where $\Wlr{k} \colon \Rnn \to \Rnn$ is the low-rank promoting part \cref{eq:W:operator:action} of the weight operator associated to $\Xk{k}$ and $\Wsp{k} \in \R^{n_1 \times n_1}$ the sparsity promoting part \cref{eq:Ws_def}.
   Since by assumption
   \begin{align*}
      \max(\varepsilon_k, \| \XIk{k} \|) \le \frac{1}{48} \sigma_r(\X_\star),
   \end{align*} 
   Lemma \ref{lemma:MatrixIRLS:weight:nuclearnormbound} yields
   \begin{align} \label{eq:WlrBound}
       \big\|\Wlr{k}(\X_\star)\big\|_{*}  \leq 0.995 \frac{42}{40 \sigma_r(\X_\star)} \left(\varepsilon_k^{2} r  + 2\varepsilon_k \|\Xk{k} - \X_{\star}\| + 2 \|\Xk{k} - \X_{\star}\|^2  \right).
   \end{align}
   Similarly, by assumption
   \begin{align*}
       \| \XIk{k} \|_{\infty,2} \le \| \XIk{k} \| \le \frac{1}{48 s} \rho_s(\X_\star) \le \frac{1}{48} \rho_s(\X_\star),
   \end{align*}
   such that \Cref{lemma:MatrixIRLS:weight:l1bound} yields
   \begin{align}
   \label{eq:WspBound}
       \| \Wsp{k}~\cdot~\X_\star \|_{1,2} \le 0.995 \frac{21 s \delta_k^2}{20 \rho_{s}(\X_\star)}.
   \end{align}
   Inserting \cref{eq:WlrBound} and \cref{eq:WspBound} into \cref{eq:XIBound} we obtain that 
   \begin{align} \label{eq:XIBoundIntermediate}
   \begin{split}
       \norm{\XIk{k+1}}{} \le  & \; 0.995\cdot 4.2 c_{\norm{\A}{2\to 2}}^2  \min\left\{\frac{\sigma_{r+1} (\Xk{k})}{\varepsilon_k},\frac{\rho_{s+1}(\Xk{k})}{\delta_k}  \right\}^2 \\
       &\cdot \bigg( \frac{r}{\sigma_r(\X_\star)} \left(\varepsilon_k^2  + 2\varepsilon_k \|\XIk{k}\| + 2\|\XIk{k}\|^2  \right) + \frac{2 s}{\rho_{s}(\X_\star)} \delta_k^2 \bigg).
       \end{split}
   \end{align}

Due to the assertion that $r_k \geq r$, it holds that $\varepsilon_k \leq \sigma_{r+1}(\Xk{k})$. 
Therefore,  \Cref{lem:B5addon} yields that
   \begin{equation*}
    \left(\varepsilon_k^2  + 2\varepsilon_k \|\XIk{k}\| + 2\|\XIk{k}\|^2  \right) 
    \leq 5  \|\XIk{k}\|^2.
   \end{equation*}
   and, since $s_k \geq s$, also that
   \[
   \delta_k^2 \leq \|\XIk{k}\|_{\infty,2}^2 \leq\|\XIk{k}\|^2,
   \]
   since $\delta_k \leq \rho_{s+1}(\Xk{k})$ in this case.

Thus, using the assertion that one of the statements $\varepsilon_k = \sigma_{r+1} (\Xk{k})$ or $\delta_k =\rho_{s+1}(\Xk{k}) $ is true, we obtain from \cref{eq:XIBoundIntermediate} that
   \begin{align} \label{eq:Xikp1:quadraticdecay}
   \begin{split}
       \norm{\XIk{k+1}}{} \le 0.995 \cdot 4.2 c_{\norm{\A}{2\to 2}}^2 
       \cdot \left( \frac{5r}{\sigma_r(\X_\star)} + \frac{2 s}{\rho_{s}(\X_\star)} \right)  \norm{\XIk{k}}{}^2 .
       \end{split}
   \end{align}
   For $\norm{\XIk{k}}{} < \frac{1}{48} c_{\norm{\A}{2\to 2}}^{-2} \min\{ \tfrac{\sigma_r(\X_\star)}{r}, \tfrac{\rho_{s}(\X_\star)}{s} \}$ (as implied by \cref{eq:closeness:assumption:mainthm}), this yields
   \begin{align} \label{eq:error:decay}
       \| \XIk{k+1} \| < 0.9 \| \XIk{k} \| 
   \end{align}
   and the quadratic error decay
   \begin{align*}
       \| \XIk{k+1} \| \le \mu  \| \XIk{k} \|^2 
   \end{align*}
   if we define $\mu = 4.179 c_{\norm{\A}{2\to 2}}^2 \big( \frac{5r}{\sigma_r(\X_\star)} + \frac{2s}{\rho_{s}(\X_\star)} \big)$.

   To show the remaining statement, we need to argue that the assertions of \Cref{thm:QuadraticConvergence} are satisfied not only for $k$, but for any $k+\ell$ with $\ell \geq 1$. For this, it is sufficient to show that 
   \begin{enumerate}
   \item $r_{k+1} \geq r$,
   \item $s_{k+1} \geq s$,
   \item $\varepsilon_{k+1} \leq  \frac{1}{48} \sigma_r(\X_\star)$,
   \item \cref{eq:closeness:assumption:mainthm} holds for $\Xk{k+1}$, and that
   \item one of the statements  $\varepsilon_{k+1} = \sigma_{r+1} (\Xk{k+1})$ or $\delta_{k+1} =\rho_{s+1}(\Xk{k+1}) $ is true,
   \end{enumerate}
   as in this case, $\Xk{k+\ell} \overset{\ell \to \infty}{\to} \X_\star$ follows by induction due to successive application of \cref{eq:error:decay}.

   For 1. and 2., we see that this follows from the smoothing parameter update rules \cref{eq:MatrixIRLS:epsdef} which imply that $\varepsilon_{k+1} \leq \sigma_{r+1}(\Xk{k+1})$ and $\delta_{k+1} \leq \rho_{s+1}(\Xk{k+1})$.

   3. follows from \cref{eq:delta:assumption:mainthm} and the fact that due to \cref{eq:MatrixIRLS:epsdef}, $(\varepsilon_{k})_{k\geq 1}$ is non-increasing. 4. is satisfied due to \cref{eq:error:decay} and \cref{eq:closeness:assumption:mainthm}.

   To show 5., we note that due to \cref{eq:closeness:assumption:mainthm}, the assertion \cref{eq:lemma:B5:closeness} is satisfied, and therefore it follows from \cref{eq:Xikp1:quadraticdecay} and \Cref{lem:B5} that
   \[
    \norm{\XIk{k+1}}{} \le 4.179 c_{\norm{\A}{2\to 2}}^2 \left( \frac{5r}{\sigma_r(\X_\star)} + \frac{2s}{\rho_{s}(\X_\star)} \right)   \norm{\XIk{k}}{} \cdot 2 \sqrt{2} \sqrt{n} c_{\norm{\A}{2\to 2}}   \sqrt{4 \varepsilon_k^2 + \delta_k^2}.
   \]
   We now distinguish the case (i) $\delta_k < \varepsilon_k$ and the case (ii) $\delta_k \geq \varepsilon_k$.

   In case (i), it holds that
   \begin{equation*}
   \begin{split}
   \sigma_{r+1}(\Xk{k+1}) \leq \norm{\XIk{k+1}}{} &\leq 4.179 c_{\norm{\A}{2\to 2}}^2 \left( \frac{5r}{\sigma_r(\X_\star)} + \frac{2s}{\rho_{s}(\X_\star)} \right)  2 \sqrt{10 n}c_{\norm{\A}{2\to 2}} \norm{\XIk{k}}{} \varepsilon_k
   \\
   &= \mu 2 \sqrt{10 n}c_{\norm{\A}{2\to 2}} \norm{\XIk{k}}{} \varepsilon_k \\
   &< \varepsilon_k,
   \end{split}
   \end{equation*}
   
   where the last inequality holds since by \eqref{eq:closeness:assumption:mainthm} the $k$-th iterate $\Xk{k}$ additionally satisfies
   \begin{equation} \label{eq:error:mu:ineq}
   \| \Xk{k} - \X_\star \| < \frac{1}{2\mu \sqrt{10} c_{\norm{\A}{2\to 2}}}.
   \end{equation}
    In this case, due to the smoothing parameter update rule \cref{eq:MatrixIRLS:epsdef}, we have that $\varepsilon_{k+1} = \sigma_{r+1}(\Xk{k+1})$.

   In case (ii), we have likewise that
   \[
   \rho_{s+1}(\Xk{k+1}) \leq \norm{\XIk{k+1}}{} \leq \mu 2 \sqrt{10 n}c_{\norm{\A}{2\to 2}} \norm{\XIk{k}}{} \delta_k < \delta_k, 
   \]
   due to \cref{eq:Xikp1:quadraticdecay}, \Cref{lem:B5}, and \cref{eq:error:mu:ineq}. Hence, $\delta_{k+1} = \rho_{s+1}(\Xk{k+1})$ which shows the remaining statement 5. and concludes the proof of \Cref{thm:QuadraticConvergence}.
   \end{proof}

\subsection{Proof of \Cref{thm:IRLS:majorization:MM}}
\label{sec:ProofOfMajorization}

1.) Let $\varepsilon, \delta > 0$ be arbitrary. Due to the additive structure of $\F_{\varepsilon,\delta}(\cdot)$, cf.\  \cref{eq:F:objective:def}, it is sufficient to establish that
\begin{equation} \label{eq:blocksparse:majorization}
\Fsp(\Z) \le \Qspd(\Z|\X) = \Fsp(\X) + \langle \nabla \Fsp(\X), \Z - \X \rangle + \frac{1}{2} \langle \Z - \X, \Wspd (\Z-\X) \rangle
\end{equation}
for any $\Z,\X \in \Rnn$, where $\Wspd: \Rnn \to \Rnn$ is defined analogously to \cref{eq:Ws_def} and
\begin{equation} \label{eq:lowrank:majorization}
\Flr(\Z) \le \Qlre(\Z|\X) = \Flr(\X)  + \langle \nabla \Flr (\X), \Z - \X \rangle + \frac{1}{2} \langle \Z - \X, \Wlre(\Z-\X) \rangle,
\end{equation}
for any $\Z,\X \in \Rnn$, where $\Wlre: \Rnn \to \Rnn$ is defined analogously to \cref{eq:W:operator:action}. 

The argument for \cref{eq:blocksparse:majorization} is standard in the \texttt{IRLS} literature \cite{Aftab-WCACV2015,Ochs-SIAM-J-IS2015,PengKuemmerleVidal-CVPR2023} and is based on the facts that both $\Qspd(\Z|\X)$ and $\Fsp(\Z)$ are row-wise separable, and that $t \mapsto f_{\sqrt{\tau}} (\sqrt{t})$ is concave and therefore majorized by its linearization: indeed, let $g_{\tau}: \R \to \R$ be such that
\[
g_{\tau}(t):=
\begin{cases}
       \frac{1}{2} \tau \log(e |t|/\tau), & \text{ if } |t| > \tau, \\
       \frac{1}{2} |t|, &  \text{ if } |t| \leq \tau.
    \end{cases}
\]
The function $g_{\tau}(\cdot)$ is continuously differentiable with derivative $g'_{\tau}(t) = \frac{\tau}{2\max(|t|,\tau)} \sign(t)$ and furthermore, concave restricted to the non-negative domain $\R_{\geq 0}$.

Therefore, it holds for any $t,t' \in \R_{\geq 0}$ that
\begin{align*}
g_{\tau}(t) \leq g_{\tau}(t') + g'_{\tau}(t')(t-t').
\end{align*}
We recall the definition $f_\tau (t) =        \frac{1}{2} \tau^2 \log(e t^2/\tau^2)$ for $|t| > \tau$ and $f_\tau (t) = \frac{1}{2} t^2$ for $|t| \leq \tau$ from \cref{eq:f_tau:definition} with derivative $f_\tau' (t) = \frac{\max(t^2,\tau^2) t}{ t^2} = \frac{\tau^2 t}{\max(t^2, \tau^2)}$. Thus, for any $x,z \in \R$, it follows that
\begin{equation*}
\begin{split}
f_{\tau}(z) = g_{\tau^2}(z^2) &\leq g_{\tau^2}(x^2) + g'_{\tau^2}(x^2)(z^2-x^2) \\
&= f_{\tau}(x)+ \frac{\tau^2}{2 \max(x^2,\tau^2)}(z^2-x^2), 
\end{split}
\end{equation*}
and inserting $\tau = \delta$, $z = \|\f{Z}_{i,:}\|_2$, $x = \|\X_{i,:}\|_2$ and summing over $i = 1, \ldots n_1$ implies that
\begin{equation*}
\begin{split}
\Fsp(\Z) &= \sum_{i=1}^{n_1} f_{\delta}( \| \f{Z}_{i,:}\|_2) \leq \Fsp(\X) + \sum_{i=1}^{n_1}   \frac{\delta^2}{2 \max(\|\f{X}_{i,:}\|_2^2,\delta^2)}(\|\f{Z}_{i,:}\|_2^2-\|\X_{i,:}\|_2^2) \\
&= \Fsp(\X)  +  \sum_{i=1}^{n_1} \frac{\delta^2}{\max(\|\X_{i,:}\|_2^2,\delta^2)}\langle \X_{i,:},\Z_{i,:} - \X_{i,:} \rangle +  \frac{1}{2} \sum_{i=1}^{n_1} \frac{ \big\|\Z_{i,:} - \X_{i,:} \big\|_2^2}{\max(\|\X_{i,:}\|_2^2/\delta^2,1)}
\end{split}
\end{equation*}
From the chain rule, it follows that for all $i=1,\ldots, n_1$ for which $\f{X}_{i,:} \neq 0$,
\begin{equation} \label{eq:derivative:fdelta}
\frac{d}{d \f{X}_{i,:}} f_{\delta}( \| \f{X}_{i,:}\|_2)  = f_{\delta}' (\| \f{X}_{i,:}\|_2) \frac{d \|\f{X}_{i,:}\|_2}{d \f{X}_{i,:}} =  \frac{ \delta^2 \|\X_{i,}\|_2}{\max( \|\X_{i,:}\|_2^2, \delta^2) } \frac{\X_{i,:}}{\|\X_{i,:} \|_2},
 =   \frac{\delta^2 \X_{i,:}}{\max( \|\X_{i,:}\|_2^2, \delta^2) }
\end{equation}
and therefore
\[
\Fsp(\Z) \leq  \Fsp(\X) + \langle \nabla  \Fsp(\X), \Z - \X \rangle + \frac{1}{2} \langle \Z - \X, \Wspd (\Z-\X) \rangle
\] 
which shows the majorization of \cref{eq:blocksparse:majorization}, recalling the definition 
\[
\Wspd = \diag \left( \max\big(\|\X_{i,:}\|_2^2/\delta^2,1)_{i=1}^{d_1}\big)^{-1}\right)
\]
of \cref{eq:Ws_def}.

The majorization of \cref{eq:lowrank:majorization} is non-trivial but follows in a straightforward way from \cite[Theorem 2.4]{K19} as the objective $\Flr(\Z)$ corresponds to the one of \cite[Theorem 2.4]{K19} up to a multiplicative factor of $\varepsilon^2$ and constant additive factors, and since the weight operator $\Wlre$ corresponds to the weight operator used in \cite[Chapter 2]{K19} for $p=0$.

2.) Due to the definition \cref{eq:def:objectives} of $\Fspk{k}(\cdot)$ and the derivative computation of \cref{eq:derivative:fdelta}, we observe that
\[
\nabla \Fspk{k}(\Xk{k}) = \diag \left(\left( \max \left( \big\|(\Xk{k})_{i,:}\big\|_2^2/\delta_k^2, 1 \right)^{-1} \right)_{i=1}^{d_1}\right) \Xk{k} = \Wsp{k} \cdot \Xk{k},
\]
comparing the resulting term with the definition of \cref{eq:Ws_def} of $\Wsp{k}$. Furthermore, an analogue equality follows from the the formula
\[
\nabla \Flrk{k} (\Xk{k})  = \begin{bmatrix}
    \U & \U_\perp 
    \end{bmatrix}
    \diag\left(  \left( \sigma_i^{(k)}\max \left( (\sigma_i^{(k)})^2/\varepsilon_k^2, 1 \right)^{-1}\right)_{i=1}^{d}  \right) \begin{bmatrix}
    \V^* \\ \V_\perp^* 
    \end{bmatrix} 
\]
with $\sigma_i^{(k)} = \sigma_i(\Xk{k})$ for any $i \leq d$, which is a direct consequence from the calculus of spectral functions \Cref{lemma:grad:spectralfct}, and inserting into the low-rank promoting weight operator formula \cref{eq:W:operator:action}
\[
\Wlr{k}(\Xk{k}) =  \begin{bmatrix}
    \U & \U_\perp 
    \end{bmatrix}  \SIGMAepsk{k}^{-1}
    \diag\left( \left( \sigma_i^{(k)} \right)_{i=1}^{d} \right)
    \SIGMAepsk{k}^{-1}   \begin{bmatrix}
    \V^* \\ \V_\perp^* 
    \end{bmatrix}  = \nabla \Flrk{k} (\Xk{k})
\]
Inserting $\nabla \Fspk{k}(\Xk{k}) = \Wsp{k} \cdot \Xk{k}$ and $\nabla \Flrk{k} (\Xk{k}) = \Wlr{k}(\Xk{k})$ into the definitions of $\Qlrk(\Z|\Xk{k})$ and $\Qspk(\Z|\Xk{k})$, we see that it holds that
\[
    \Qlrk(\Z|\Xk{k}) = \Flrk{k}(\Xk{k}) + \frac{1}{2} \left( \langle \Z , \Wlr{k} (\Z) \rangle -\langle \Xk{k},\Wlr{k}(\Xk{k}) \rangle \right)
\]
and 
\[
    \Qspk(\Z|\Xk{k}) = \Fspk{k}(\Xk{k}) + \frac{1}{2} \left( \langle \Z , \Wsp{k} \Z \rangle -\langle \Xk{k},\Wsp{k} \Xk{k} \rangle \right).
\]
Therefore, we see that the weighted least squares solution $
\Xk{k+1}$ of \cref{eq:MatrixIRLS:Xdef} for $k+1$ coincides with the minimizer of 
\begin{equation} \label{eq:quadratic:mini}
\begin{split}
&\min\limits_{\Z:\A(\Z)=\f{y}} \left[ \Qlrk(\Z|\Xk{k}) + \Qspk(\Z|\Xk{k}) \right] \\
&= \min\limits_{\Z:\A(\Z)=\f{y}} \bigg[\Flrk{k}(\Xk{k})+\Fspk{k}(\Xk{k}) \\
&\;\;\;\;\;\;\quad\quad\quad\quad+ \frac{1}{2} \left( \langle \Z , \Wk{k} (\Z) \rangle -\langle \Xk{k},
    \Wk{k}(\Xk{k}) \rangle \right)\bigg] 
\end{split}
\end{equation}
with the weight operator $\Wk{k}$ of \cref{eq:DefW}, which implies that
\begin{equation} \label{eq:Q:decrease}
\Qlrk(\Xk{k+1}|\Xk{k}) + \Qspk(\Xk{k+1}|\Xk{k}) \leq \Qlrk(\Xk{k}|\Xk{k}) + \Qspk(\Xk{k}|\Xk{k}).
\end{equation}
Using the majorization \cref{eq:majorization:Fepsdeltak} established in Statement 1 of \Cref{thm:IRLS:majorization:MM} and \cref{eq:Q:decrease}, it follows that
\begin{align}
\label{eq:IntermediateFBound}
\begin{split}
\F_{\varepsilon_k,\delta_k}(\Xk{k+1}) 
    &\leq \Qlrk(\Xk{k+1}| \Xk{k}) + \Qspk(\Xk{k+1}|\Xk{k}) 
    \\
    &\le \Qlrk(\Xk{k}| \Xk{k}) + \Qspk(\Xk{k}|\Xk{k})  \\
    &= \Flrk{k}(\Xk{k})+\Fspk{k}(\Xk{k}) = \F_{\varepsilon_k,\delta_k}(\Xk{k}),
\end{split}
\end{align}
using in the third line that $\Qlrk(\Xk{k}|\Xk{k})= \Flrk{k}(\Xk{k})$ and $\Qspk(\Xk{k}|\Xk{k}) = \Fspk{k}(\Xk{k})$.

To conclude, it suffices to show that $\varepsilon \mapsto \F_{\varepsilon,\delta_k}(\Xk{k+1})$ and $\delta \mapsto \F_{\varepsilon_k,\delta}(\Xk{k+1})$ are non-decreasing functions, since \cref{eq:IntermediateFBound} then extends to
\begin{align*}
    \F_{\varepsilon_{k+1},\delta_{k+1}}(\Xk{k+1}) 
    \le\F_{\varepsilon_{k},\delta_{k+1}}(\Xk{k+1}) 
    \le \F_{\varepsilon_k,\delta_k}(\Xk{k+1}) 
    \leq \F_{\varepsilon_k,\delta_k}(\Xk{k}),
\end{align*}
where we used that the sequences $\varepsilon_k$ and $\delta_k$ defined in \Cref{algo:MatrixIRLS} are decreasing.
So let us prove this last claim. We define for $t \in \R$ the function $h_t:\R_{>0} \to \R$ such that $h_t(\tau) = f_{\tau}(t)$, i.e.,
\[
h_{t}(\tau)=
\begin{cases}
    \frac{1}{2} t^2, &  \text{ if } \tau \geq |t|, \\
    \frac{1}{2} \tau^2 \log(e t^2/\tau^2), & \text{ if } \tau < |t|.
    \end{cases}
\]
This function is continuously differentiable with $h_{t}'(\tau) = 0$ for all $\tau > |t|$ and 
\[
h_{t}'(\tau) = \tau \left( \log(e t^2/\tau^2) -  1 \right)
\]
for $\tau < |t|$, which implies that $h_{t}'(\tau)\geq 0$ for all $\tau \geq 0$ and thus shows that $\varepsilon \mapsto \F_{\varepsilon,\delta_k}(\Xk{k+1})$ and $\delta \mapsto \F_{\varepsilon_k,\delta}(\Xk{k+1})$ are non-decreasing functions due to the additive structure of $\F_{\varepsilon,\delta}(\Xk{k+1})$ and \cref{eq:F:objective:def}.

3.) First, we argue that $(\Xk{k})_{k \geq 1}$ is a bounded sequence: Indeed, if $\overline{\varepsilon} := \lim_{k \to \infty} \varepsilon_k > 0$ and $\overline{\delta} := \lim_{k \to \infty} \delta_k > 0$, we note that
\begin{equation*}
 \begin{split}
 &\frac{1}{2}\overline{\varepsilon}^2 \log(e \|\Xk{k}\|^2/\overline{\varepsilon}^2) + \frac{1}{2}\overline{\delta}^2 \log(e \max_{i} \|\Xk{k}\|_{\infty,2} /\overline{\delta}^2)
 \\ &= \frac{1}{2}\overline{\varepsilon}^2 \log(e \sigma_1^2(\Xk{k})/\overline{\varepsilon}^2) + \frac{1}{2}\overline{\delta}^2 \log(e \max_{i} \|\Xk{k}_{i,:}\|_2/\overline{\delta}^2) \\
    &\leq \frac{1}{2}\varepsilon_k^2 \log(e \sigma_1^2(\Xk{k})/\varepsilon_k^2) + \frac{1}{2}\delta_k^2 \log(e \max_{i} \|\Xk{k}_{i,:}\|_2/\delta_k^2) \\
    &\leq  \Flrk{k}(\Xk{k}) + \Fspk{k}(\Xk{k}) = \F_{\varepsilon_k,\delta_k}(\Xk{k}) \leq   \F_{\varepsilon_1,\delta_1}(\Xk{1}) \\
    &\leq \frac{1}{2}\min(d_1,d_2) \sigma_1^2(\Xk{1}) + \frac{1}{2}d_1 \max_{i}\|\Xk{1}_{i,:}\|_2^2 =: C_{\Xk{1}},
     \end{split}
\end{equation*}
which implies that $\{\|\Xk{k}\|\}_{k\geq 1}$ is bounded by a constant that depends on $C_{\Xk{1}}$.

Furthermore, we note that the optimality condition of \cref{eq:MatrixIRLS:Xdef} (see \Cref{lemma:leastsquares:opt}) implies that $\Xk{k+1}$ satisfies
	\[
	\langle \Wk{k}(\Xk{k+1}), \XI \rangle  = 0 \text{ for all }\XI \in \ker \A  \text{ and } \A(\Xk{k+1}) =\f{y}.
	\]
 Choosing $\XI = \Xk{k+1} - \Xk{k}$ and using the notation $W^{(k)}=\Wk{k}$ we see that
	\begin{equation} \label{eq:proof:MatrixIRLS:monotonicity:5a}
	\begin{split}
	& \langle \Xk{k+1},W^{(k)}(\Xk{k+1})\rangle - \langle \Xk{k},W^{(k)}(\Xk{k})\rangle   \\
	&=   \langle \Xk{k+1},W^{(k)}(\Xk{k+1})\rangle -  \langle \Xk{k},W^{(k)}(\Xk{k})\rangle - 2 \langle W^{(k)}(\Xk{k+1}),\Xk{k+1} - \Xk{k}\rangle \\
 	&= -\left( \langle \Xk{k+1},W^{(k)}(\Xk{k+1})\rangle - 2 \langle W^{(k)}(\Xk{k}),\Xk{k+1}\rangle + \langle \Xk{k},W^{(k)}(\Xk{k})\rangle \right) \\
	&= - \big\langle(\Xk{k+1} -\Xk{k}),W^{(k)}(\Xk{k+1} -\Xk{k})\big\rangle.
	\end{split}
	\end{equation}
 Due to the definition of $W^{(k)}$, we note that its smallest singular value (interpreted as matrix operator) can be lower bounded by 
 
\begin{equation*}
 \begin{split}
 \sigma_{\min}(W^{(k)})&\geq \sigma_{\min}\big(\Wlr{k}\big)
    + \sigma_{\min}\big(\Wsp{k}\big)   \geq \delta_k^2/\max_{i} \big\|\Xk{k}_{i,:}\big\|_2^2 + \varepsilon_k^2 /\sigma_{1}^2(\Xk{k})
\\ 
&\geq \overline{\delta}^2/ c_{\text{sp},\Xk{1}} + \overline{\varepsilon}^2 / c_{\text{lr},\Xk{1}},
 \end{split}
\end{equation*}
 where $c_{\text{sp},\Xk{1}}$ and $c_{\text{lr},\Xk{1}}$ are constants that satisfy $c_{\text{sp},\Xk{1}} \leq \overline{\delta}^2\exp(C_{\Xk{1}} /\overline{\delta}^2 -1) $ and $c_{\text{lr},\Xk{1}} \leq \overline{\varepsilon}^2\exp(C_{\Xk{1}} /\overline{\varepsilon}^2 -1) $.
 
Combining this with \cref{eq:proof:MatrixIRLS:monotonicity:5a}, the monotonicity according to Statement 2 of \Cref{thm:IRLS:majorization:MM}, and \cref{eq:quadratic:mini}, it follows that 
\begin{equation*}
 \begin{split}
 \F_{\varepsilon_k,\delta_k}(\Xk{k})- \F_{\varepsilon_{k+1},\delta_{k+1}}(\Xk{k+1}) &\geq \frac{1}{2} \big\langle(\f{X}\hk -\Xk{k+1}),W^{(k)}(\f{X}\hk -\Xk{k+1})\big\rangle \\
 &\geq  \frac{1}{2}\left( \overline{\delta}^2/ c_{\text{sp},\Xk{1}} + \overline{\varepsilon}^2 / c_{\text{lr},\Xk{1}}\right) \|\Xk{k+1}-\Xk{k}\|_F^2.
 \end{split}
\end{equation*}
 Summing over all $k$, this implies that $\lim_{k\to \infty} \|\Xk{k+1}-\Xk{k}\|_F = 0$.

 Since $(\Xk{k})_{k\geq 1}$ is bounded, each subsequence of $(\Xk{k})_{k\geq 1}$ has a convergent subsequence. Let $(\Xk{k_{\ell}})_{\ell\geq 1}$ be such a sequence with $\lim_{\ell \to \infty} \Xk{k_{\ell}} = \bar{\f{X}}$, i.e., $\bar{\f{X}}$ is an accumulation point of the sequence. As the weight operator $W^{(k_{\ell})}$ depends continuously on $\f{X}^{(k_{\ell})}$, there exists a weight operator $\bar{W}: \Rnn \to \Rnn$ such that $\bar{W}=\lim_{\ell \to \infty} W^{(k_{\ell})}$. 
	
Since $\lim_{k\to \infty} \|\Xk{k+1}-\Xk{k}\|_F = 0$, it also holds that $\f{X}^{(k_{\ell}+1)} \to \bar{\f{X}}$ and therefore
	\[
	\langle \nabla \F_{\overline{\varepsilon},\overline{\delta}}(\bar{\f{X}}) , \XI \rangle =  \langle \bar{W}(\bar{\f{X}}),\XI \rangle = \lim_{\ell \to \infty} \langle W^{(k_{\ell})}(\f{X}^{(k_{\ell}+1)}),\XI \rangle  = 0
	\]
	for all $\XI \in \ker \A$. The statement is shown as this is equivalent to $\bar{\f{X}}$ being a stationary point of $\F_{\overline{\varepsilon},\overline{\delta}}(\cdot)$ subject to the linear constraint $\{\f{Z} \in \Rnn: \A(\f{Z}) = \f{y}\}$.

\section{Technical addendum}
\label{sec:TechnicalAddendum}

\subsection{Auxiliary Results}
In the proof of \Cref{thm:IRLS:majorization:MM}, we use the following result about the calculus of spectral functions.
\begin{lemma}[{\cite{Lewis05_Nonsm1},\cite[Proposition 7.4]{forawa11}}] \label{lemma:grad:spectralfct}
	Let $F: \R^{d_1 \times d_2} \to \R$ be a spectral function $F = f \circ \sigma$ with an associated function $f: \R^d \to \R$ that is absolutely permutation symmetric. Then, $F$ is differentiable at $\f{X} \in \R^{d_1 \times d_2}$ if and only if $f$ is differentiable at $\sigma(\f{X}) \in \R^d$. 
	
		In this case, the gradient $\nabla F$ of $F$ at $\f{X}$ is given by 
		\[
			\nabla F(\f{X}) = \f{U} \diag\big(\nabla f(\sigma(\f{X})\big) \f{V}^*
		\]
		if $\f{X} = \f{U} \diag\big(\sigma(\f{X})\big) \f{V}^*$ for unitary matrices $\f{U} \in \R^{d_1 \times d_1}$ and $\f{V} \in \R^{d_2 \times d_2}$.\footnote{Here, for $\f{v} \in \R^{\min(d_1,d_2)}$, $\diag(\f{v}) \in \R^{d_1 \times d_2}$ refers to the matrix with diagonal elements $\f{v}_i$ on its main diagonal and zeros elsewhere.}
\end{lemma}

\subsection{Proof of \Cref{lem:Projection}}
\label{sec:ProofOfLemProjection}

The projection operators for $\MM{r}^{n_1,n_2}$, $\NN{s}^{n_1,n_2}$, and $\TT{\U,\V}$ are well-known, see e.g.\ \cite{boumal2023introduction}. To see the final statement assume that $\U$ has row-support $S$ and note that $\P_{\U,\V,S}$ is idempotent, i.e.,
    \begin{align*}
        &\P_{\U,\V,S} \P_{\U,\V,S} \Z \\
        &= \U\U^* ( \U\U^* \Z + \P_S \Z \V\V^* - \U\U^* \Z \V\V^*) 
        + \P_S( \U\U^* \Z + \P_S \Z \V\V^* - \U\U^* \Z \V\V^* ) \V\V^* \\
        &- \U\U^* ( \U\U^* \Z + \P_S \Z \V\V^* - \U\U^* \Z \V\V^*) \V\V^* \\
        &= \U\U^* \Z + \U\U^* \P_S \Z \V\V^* - \U\U^* \Z \V\V^* 
        + \U\U^* \Z \V\V^* + \P_S \Z \V\V^* - \U\U^* \Z \V\V^* \\
        &- \U\U^* \Z \V\V^* - \U\U^* \P_S \Z \V\V^* + \U\U^* \Z \V\V^* \\
        &= \U\U^* \Z + \P_S \Z \V\V^* - \U\U^* \Z \V\V^* 
        = \P_{\U,\V,S} \Z.
    \end{align*}
    One can easily check that $\P_{\U,\V,S}$ acts as identity when applied to matrices in $\TT{\U,\V,S}$ and that $\P_{\U,\V,S} = \P_{\U,\V,S}^*$ since $\langle \Z', \P_{\U,\V,S} \Z \rangle_F = \langle \P_{\U,\V,S} \Z', \Z \rangle_F$, for any $\Z,\Z'$. This proves the claim.

\subsection{Proof of \Cref{lem:RIP}}
\label{sec:ProofOfLemRIP}

Let $\XI \in \ker(\A)$. Note that
\begin{align*}
    0 = \norm{\A(\XI)}{2} 
    = \norm{\A(\P_{\TT{\U,\V,S}} (\XI) + \P_{\TT{\U,\V,S}}^\perp (\XI))}{2}
    \ge \norm{\A(\P_{\TT{\U,\V,S}} (\XI))}{2} - \norm{\A( \P_{\TT{\U,\V,S}}^\perp (\XI))}{2}
\end{align*}
By the RIP we hence get that
\begin{align*}
    \norm{\P_{\TT{\U,\V,S}} (\XI)}{F}^2
    &\le \frac{1}{1-\delta} \norm{\A(\P_{\TT{\U,\V,S}} (\XI))}{2}^2
    \le \frac{1}{1-\delta} \norm{\A( \P_{\TT{\U,\V,S}}^\perp (\XI))}{2}^2 \\
    &\le \frac{\norm{\A}{2\to 2}^2}{(1-\delta)} \norm{\P_{\TT{\U,\V,S}}^\perp (\XI)}{F}^2.
\end{align*}
Consequently,
\begin{align*}
    \norm{\XI}{F}^2 
    = \norm{\P_{\TT{\U,\V,S}} (\XI)}{F}^2 + \norm{\P_{\TT{\U,\V,S}}^\perp (\XI)}{F}^2
    \le \round{1 + \frac{\norm{\A}{2\to 2}^2}{(1-\delta)}} \norm{\P_{\TT{\U,\V,S}}^\perp (\XI)}{F}^2.
\end{align*}

\subsection{Proof of \Cref{lemma:B8:generalization}}
\label{sec:ProofOfLemB8}

In the proof of \Cref{lemma:B8:generalization} we will use the following fact.

\begin{lemma} \label{lem:W_diagonalForm}
    Let $\Wlr{k}$ be the weight operator defined in \cref{eq:W:operator:action}, which is based on the matrices $\U \in \R^{n_1 \times r_k}$ and $\V \in \R^{n_2 \times r_k}$ of leading $r_k$ left and right singular vectors of $\Xk{k}$.
    If $\f{M} \in T_{\U,\V}$, then $\Wlr{k} (\f{M}) \in T_{\U,\V}$. If $\f{M} \in T_{\U,\V}^\perp$, then $\Wlr{k} (\f{M}) \in T_{\U,\V}^{\perp}$.
\end{lemma}

\begin{proof}
    If $\f{M} \in T_{\U,\V}$, there exist $\f{M}_1 \in \R^{r_k \times r_k}$, $\f{M}_2 \in \R^{r_k \times (n_2-r_k)}$, $\f{M}_3 \in \R^{(n_1-r_k) \times r_k}$ such that
   \[
   \f{M} = \begin{bmatrix}
    \U & \U_\perp 
    \end{bmatrix} 
    \begin{bmatrix}
    \f{M}_1  & \f{M}_2 \\
    \f{M}_3 & 0
    \end{bmatrix}
    \begin{bmatrix}
    \V^* \\ \V_\perp^* 
    \end{bmatrix},
   \]
   e.g., see \cite[Proposition 2.1]{Vandereycken13}.
   We thus observe that the weight operator $\Wlr{k}: \R^{n_1 \times n_2} \to \R^{n_1 \times n_2}$ from \cref{eq:W:operator:action2} satisfies
   \begin{equation*}
   \begin{split}
   &\Wlr{k} (\f{M}) \\
   &= \begin{bmatrix}
    \U & \U_\perp 
    \end{bmatrix} \round{ \Hk{k} \circ \round{
    \begin{bmatrix}
    \U^* \\ \U_\perp^* 
    \end{bmatrix}
    \begin{bmatrix}
    \U & \U_\perp 
    \end{bmatrix} \begin{bmatrix}
    \f{M}_1  & \f{M}_2 \\
    \f{M}_3 & 0
    \end{bmatrix}
    \begin{bmatrix}
    \V^* \\ \V_\perp^* 
    \end{bmatrix}
    \begin{bmatrix}
    \V & \V_\perp 
    \end{bmatrix}
    } }
    \begin{bmatrix}
    \V^* \\ \V_\perp^* 
    \end{bmatrix} \\
    &= 
    \begin{bmatrix}
    \U & \U_\perp 
    \end{bmatrix} \round{ \Hk{k} \circ
    \begin{bmatrix}
    \f{M}_1  & \f{M}_2 \\
    \f{M}_3 & 0
    \end{bmatrix}
    } 
    \begin{bmatrix}
    \V^* \\ \V_\perp^* 
    \end{bmatrix}  \\
    &=     \begin{bmatrix}
    \U & \U_\perp 
    \end{bmatrix} \begin{bmatrix}
    \f{H}_1^{(k)}\circ\f{M}_1  &  \f{H}_2^{(k)}\circ \f{M}_2 \\
    \f{H}_3^{(k)} \circ \f{M}_3 & 0
    \end{bmatrix}
    \begin{bmatrix}
    \V^* \\ \V_\perp^* 
    \end{bmatrix}  \in T_{\U,\V}.
    \end{split}
 \end{equation*}
   Similarly, if $\f{M} \in T_{\U,\V}^\perp$, there exists $\f{M}_4 \in \R^{(n_1-r_k) \times (n_2-r_k)}$ such that $\f{M} =  \U_{\perp} \f{M}_4 (\V_{\perp})^{*}$ and 
   \begin{equation*}
   \begin{split}
   \Wlr{k} (\f{M}) &=     \begin{bmatrix}
    \U & \U_\perp 
    \end{bmatrix} \round{ \Hk{k} \circ
    \begin{bmatrix}
    0 & 0 \\
    0 & \f{M}_4
    \end{bmatrix}
    } 
    \begin{bmatrix}
    \V^* \\ \V_\perp^* 
    \end{bmatrix}
  \\
    &=\begin{bmatrix}
    \U & \U_\perp 
    \end{bmatrix}
    \begin{bmatrix}
    0 & 0 \\
    0 & \f{M}_4 / \varepsilon_k^2
    \end{bmatrix}
    \begin{bmatrix}
    \V^* \\ \V_\perp^* 
    \end{bmatrix} \in T_{\U,\V}^{\perp}.
    \end{split}
   \end{equation*}
\end{proof}

\begin{proof}[of \Cref{lemma:B8:generalization}]
Let $\XI \in \Rnn$ be arbitrary. We start with some simple but technical observations. First note that by Lemma \ref{lem:Projection}, if $T_k = T_{\tU,\tV,\tSk{k}}$,
   \begin{align}
   \begin{split} \label{eq:OrthDecomp}
      \P_{T_k^\perp} \XI
      &= (\id - \P_{T_k})(\XI) \\
      &= \XI - \Big( \tU\tU^* \XI + \P_{\tSk{k}}\XI \tV\tV^* - \tU\tU^* \XI \tV\tV^* \Big) \\
      &= (\id - \P_{\tU,\tV}) \XI + (\id - \P_{\tSk{k}}) \XI \tV\tV^* \\
      &= \P_{\tU,\tV}^\perp \XI + \P_{\tSk{k}^c} \XI \tV\tV^*,
   \end{split}
   \end{align}
   with
   \begin{align} \label{eq:Orthogonality}
       \Big\langle \P_{\tU,\tV}^\perp \XI, \P_{\tSk{k}^c} \XI \tV\tV^* \Big\rangle
       = \Big\langle \tU_\perp \tU_\perp^* \XI \tV_\perp \tV_\perp^*, \P_{\tSk{k}^c} \XI \tV \tV^* \Big\rangle
       = 0,
   \end{align}
   where we used that $\P_{\tU,\tV}^\perp \XI = (\id - \P_{\tU,\tV}) \XI = \tU_\perp \tU_\perp^* \XI \tV_\perp \tV_\perp^*$, 
   for $\tU_\perp \in \R^{n_1 \times (n_1-r)}$ and $\tV_\perp \in \R^{n_2 \times (n_2-r)}$ being the complementary orthonormal bases of $\tU$ and $\tV$, and that $\tV_\perp^* \tV = \0$. 
   
   Second, let now $\U \in \R^{n_1 \times r}$ and $\V \in \R^{n_2 \times r}$ be matrices with $r$ leading left and right singular vectors of $\Xk{k}$ in their columns which coincide with the matrices $\U$ and $\V$ from \Cref{def:weight:operator} in their first $r$ columns. Then it follows from \cref{eq:OrthDecomp,eq:Orthogonality} that
   \begin{align} \label{eq:PTkperpXI}
   \begin{split}
       &\norm{\P_{T_k^\perp} (\XI) }{F}^2 \\
       &= \norm{\P_{\tU,\tV}^\perp \XI}{F}^2 + \norm{\P_{\tSk{k}^c} \XI \tV\tV^*}{F}^2 \\
       &= \norm{\P_{\U,\V}^\perp \XI + (\P_{\tU,\tV}^\perp - \P_{\U,\V}^\perp) \XI}{F}^2 + \norm{\P_{\tSk{k}^c} \XI (\V\V^* +(\tV\tV^*- \V\V^*))}{F}^2 \\
       &\le 2 \left( \norm{\P_{\U,\V}^\perp \XI}{F}^2 + \norm{(\P_{\tU,\tV}^\perp - \P_{\U,\V}^\perp) \XI}{F}^2 + \norm{\P_{\tSk{k}^c} \XI\V\V^*}{F}^2 + \norm{\P_{\tSk{k}^c} \XI (\tV\tV^*- \V\V^*)}{F}^2 \right).
   \end{split}
   \end{align}
   By an argument analogous to \cref{eq:Orthogonality}, we observe that
   \begin{equation*}
     \norm{\P_{\U,\V}^\perp \XI}{F}^2 + \norm{\P_{\tSk{k}^c} \XI\V\V^*}{F}^2  = \langle \P_{\U,\V}^\perp \XI + \P_{\tSk{k}^c} \XI\V\V^* ,\P_{\U,\V}^\perp \XI  + \P_{\tSk{k}^c} \XI\V\V^* \rangle = \langle \widetilde{\XI}, \widetilde{\XI} \rangle,  
\end{equation*}
where $\widetilde{\XI} = \P_{\mathcal{M}}(\XI)$ is an element of the subspace $\mathcal{M} = \mathcal{M}_1 \oplus \mathcal{M}_2 \subset \Rnn$ that is the direct sum of the subspaces $\mathcal{M}_1 := \{\P_{\U,\V}^\perp \Z: \Z \in \Rnn\} $ and $\mathcal{M}_2 := \{\P_{\tSk{k}^c} \Z\V\V^*: \Z \in \Rnn\} $.

Let now $\Wlr{k}$ be the rank promoting part of the weight operator from \cref{eq:W:operator:action} and $\Wsp{k}$ be the row-sparsity promoting part from \cref{eq:Ws_def}. Note that the restriction of 
\[
  \bar W := \P_{\U,\V}^\perp \Wlr{k}  \P_{\U,\V}^\perp + \P_{\tSk{k}^c} \Wsp{k}  \P_{\tSk{k}^c} 
\]
to $\mathcal{M}$ is invertible as its first summand is invertible on $\mathcal{M}_1 = \TT{\U,\V}^{\perp}$, its second summand is invertible on $\mathcal{M}_2$ (recall that the weight operators are positive definite), and $\mathcal M_1 \perp \mathcal M_2$. 
Therefore it holds that
   \begin{equation*}
     \begin{split}
     &\norm{\P_{\U,\V}^\perp \XI}{F}^2 + \norm{\P_{\tSk{k}^c} \XI\V\V^*}{F}^2 \\
     &= \langle \widetilde{\XI}, \widetilde{\XI}  \rangle 
     =  \Big\langle \bar W_{\vert \mathcal{M}}^{1/2} \widetilde{\XI}, \bar W_{\vert \mathcal{M}}^{-1} \bar W_{\vert \mathcal{M}}^{1/2} \widetilde{\XI} \Big\rangle \\
     &\leq \sigma_{1} \left( \bar W_{\vert \mathcal{M}}^{-1} \right) \Big\langle \widetilde{\XI}, \bar W_{\vert \mathcal{M}}  \widetilde{\XI} \Big\rangle = \frac{1}{  \sigma_{\operatorname{min}} \left( \bar W_{\vert \mathcal{M}}\right)}
     \Big\langle \widetilde{\XI}, \bar W   \widetilde{\XI} \Big\rangle 
     \\
     &\leq \frac{1}{  \sigma_{\operatorname{min}} \left( \left( \P_{\U,\V}^\perp \Wlr{k}  \P_{\U,\V}^\perp\right)_{\vert \mathcal{M}}\right) + \sigma_{\operatorname{min}} \left(\left(\P_{\tSk{k}^c} \Wsp{k}  \P_{\tSk{k}^c} \right)_{\vert \mathcal{M}}\right)} \Big\langle \widetilde{\XI}, \bar W \widetilde{\XI} \Big\rangle 
     \\
     &\leq \frac{1}{  \frac{\varepsilon_k^2}{\sigma_{r+1}^2(\Xk{k})} + \frac{\delta_k^2}{\rho_{s+1}^2(\Xk{k})}} \Big\langle \XI, \bar W \XI \Big\rangle. 
     \end{split}
   \end{equation*}
 In the first inequality, we used that $\bar W_{\vert \mathcal{M}}$ is positive definite. In the second inequality, we used that $\sigma_{\min}(A+B) \geq \sigma_{\min}(A) + \sigma_{\min}(B)$, for any positive semidefinite operators $A$ and $B$, and and in the third inequality that $\langle \widetilde{\XI}, \bar W \widetilde{\XI} \rangle \le \langle \XI, \bar W \XI \rangle$. The latter observation can be deduced as follows: Note that, by the self-adjointness of $\bar{W}$,
 \begin{align*}
    \langle \XI, \bar W \XI \rangle 
    &= \langle \P_{\mathcal M} (\XI), \bar W \P_{\mathcal M} (\XI) \rangle 
    + \langle \P_{\mathcal M} (\XI), \bar W \P_{\mathcal M}^\perp (\XI) \rangle 
    + \langle \P_{\mathcal M}^\perp (\XI), \bar W \P_{\mathcal M} (\XI) \rangle  
    + \langle \P_{\mathcal M}^\perp (\XI), \bar W \P_{\mathcal M}^\perp (\XI) \rangle\\
    &= \langle \widetilde{\XI}, \bar W \widetilde{\XI} \rangle 
    + 2 \langle \P_{\mathcal M}^\perp (\XI), \bar W \P_{\mathcal M} (\XI) \rangle  
    + \langle \P_{\mathcal M}^\perp (\XI), \bar W \P_{\mathcal M}^\perp (\XI) \rangle.
 \end{align*}
 Since $\bar W$ is positive semi-definite (due to the fact that both $\Wlr{k}$ and $\Wsp{k}$ are positive definite), all that remains is to argue that the mixed term on the right-hand side vanishes. To this end, note that $\P_{\U,\V}^\perp \Z = \P_{\tSk{k}^c} \Z \V\V^* = 0$, for any $\Z \in \mathcal M_\perp$ and compute
 \begin{align*}
     &\langle \P_{\mathcal M}^\perp (\XI), \bar W \P_{\mathcal M} (\XI) \rangle \\
     &= \langle \P_{\mathcal M}^\perp (\XI), \P_{\U,\V}^\perp (\Wlr{k}  (\P_{\U,\V}^\perp (\P_{\mathcal M_1} (\XI) ))) \rangle
     + \langle \P_{\mathcal M}^\perp (\XI), \P_{\tSk{k}^c} \Wsp{k}  \P_{\tSk{k}^c} \P_{\mathcal M_2} (\XI) \rangle \\
     &= \langle \P_{\U,\V}^\perp(\P_{\mathcal M}^\perp (\XI)), \Wlr{k}  (\P_{\U,\V}^\perp (\P_{\mathcal M_1} (\XI))) \rangle
     + \langle \P_{\tSk{k}^c} \P_{\mathcal M}^\perp (\XI) \V\V^*,  \Wsp{k}  \P_{\tSk{k}^c} \XI \rangle \\
     &= 0.
 \end{align*}
 We can now continue by estimating
   \begin{align} \label{eq:InnerProductBound}
    \begin{split}
     \Big\langle \XI, \bar W \XI \Big\rangle 
     &=  \Big\langle \XI, \P_{\U,\V}^\perp \Wlr{k}  \P_{\U,\V}^\perp \XI \Big\rangle + \Big\langle \XI, \P_{\tSk{k}^c} \Wsp{k}  \P_{\tSk{k}^c} \XI \Big\rangle \\
     &= \Big\langle \P_{\U,\V}^\perp(\XI) , \Wlr{k}  \P_{\U,\V}^\perp (\XI) \Big\rangle + \Big\langle \P_{\tSk{k}^c}\XI, \Wsp{k}  \P_{\tSk{k}^c} \XI \Big\rangle  \\
     &\leq \Big\langle \XI, \Wlr{k}  \XI \Big\rangle + \Big\langle \XI,  \Wsp{k} \XI \Big\rangle \\
     &= \Big\langle \XI, \Wk{k} \XI \Big\rangle,
     \end{split}
   \end{align}
 using the positive semidefiniteness of $\Wlr{k}$ and $ \Wsp{k}$ in the last inequality. To be precise, the last inequality can be argued as follows: Due to complimentary supports $S$ and $S^c$, we see that
    \begin{align} \label{eq:MixedTermsSparse}
   \begin{split}
       \langle \XI, \Wsp{k}  \XI \rangle 
       &= \langle \P_{\tSk{k}} \XI, \Wsp{k}  \P_{\tSk{k}} \XI \rangle
       + \langle \P_{\tSk{k}^c} \XI, \Wsp{k}  \P_{\tSk{k}^c} \XI \rangle \\
       &+ \underbrace{\langle \P_{\tSk{k}} \XI, \Wsp{k} \P_{\tSk{k}^c} \XI \rangle}_{=0}
       + \underbrace{\langle \P_{\tSk{k}^c} \XI, \Wsp{k}\P_{\tSk{k}} \XI \rangle}_{=0} \\
       &\geq \langle \P_{\tSk{k}^c} \XI, \Wsp{k}  \P_{\tSk{k}^c} \XI \rangle.
   \end{split}
   \end{align}
 Similarly, we note that $ \Wlr{k}$ acts diagonally on $T_{\U,\V}$ and $T_{\U,\V}^\perp$. Indeed, we have by Lemma \ref{lem:W_diagonalForm} that if $\f{M} \in T_{\U,\V}$, then $\Wlr{k} (\f{M}) \in T_{\U,\V}$ and if $\f{M} \in T_{\U,\V}^\perp$, then $\Wlr{k} (\f{M}) \in T_{\U,\V}^{\perp}$, which implies
   \begin{align*}
       \langle \P_{\U,\V} \XI, \Wlr{k}(\P_{\U,\V}^\perp \XI) \rangle = 0 \quad \text{and} \quad \langle \P_{\U,\V}^\perp \XI, \Wlr{k}(\P_{\U,\V} \XI) \rangle = 0
   \end{align*}
   due to the orthogonality of elements in $T_{\U,\V}^{\perp}$ and $T_{\U,\V}$, respectively, and therefore it follows from $\XI = T_{\U,\V}(\XI) + T_{\U,\V}^{\perp}(\XI)$ that
   \begin{align} \label{eq:MixedTermsLR}
   \begin{split}
       \langle \XI, \Wlr{k}(\XI) \rangle &= \langle \P_{\U,\V} \XI, \Wlr{k}(\P_{\U,\V} \XI) \rangle + \langle \P_{\U,\V}^\perp \XI, \Wlr{k}(\P_{\U,\V}^\perp \XI) \rangle \\
       &+ \underbrace{\langle \P_{\U,\V} \XI, \Wlr{k}(\P_{\U,\V}^\perp \XI) \rangle}_{=0} + \underbrace{\langle \P_{\U,\V}^\perp \XI, \Wlr{k}(\P_{\U,\V} \XI) \rangle}_{=0} \\
       &\geq \langle \P_{\U,\V}^\perp \XI, \Wlr{k}(\P_{\U,\V}^\perp \XI) \rangle.
   \end{split}
   \end{align}
   Combining the previous estimates with \cref{eq:PTkperpXI} and noticing that
   \[
   \frac{1}{  \frac{\varepsilon_k^2}{\sigma_{r+1}^2(\Xk{k})} + \frac{\delta_k^2}{\rho_{s+1}^2(\Xk{k})}} \leq \min\left\{ \frac{\sigma_{r+1}^2(\Xk{k})}{\varepsilon_k^2} , \frac{\rho_{s+1}^2(\Xk{k})}{\delta_k^2}\right\},
   \]
   we obtain 
      \begin{equation} \label{eq:sum:Wop:innerproduct}
   \begin{split}
      \norm{\P_{T_k^\perp} (\XI) }{F}^2  &\leq 2 \left(\min\left\{ \frac{\sigma_{r+1}^2(\Xk{k})}{\varepsilon_k^2} , \frac{\rho_{s+1}^2(\Xk{k})}{\delta_k^2}\right\}\Big\langle \XI, \Wk{k} \XI \Big\rangle\right)
      \\
       &+ 2 
       \norm{(\P_{\tU,\tV}^\perp - \P_{\U,\V}^\perp) \XI}{F}^2 + 2 \norm{\P_{\tSk{k}^c} \XI (\tV\tV^*- \V\V^*)}{F}^2. 
   \end{split}
   \end{equation}

   Next, we control the last two summands in \cref{eq:sum:Wop:innerproduct} using matrix perturbation results. 
   Recall that $\U_\star \in \R^{n_1 \times r}$ and $\V_\star \in \R^{n_2 \times r}$ are the singular vector matrices of the reduced singular value decomposition of $\X_\star$. First observe that
    \begin{equation*}
    \begin{split}
    &\norm{(\P_{\tU,\tV} - \P_{\U_\star,\V_\star}) \XI}{F} \\
    &= \norm{(\tU \tU^* - \U_\star \U_\star^*) \XI (\id  - \V_\star \V_\star^*) + (\id -\tU \tU^*) \XI ( \tV \tV^* - \V_\star \V_\star^*) }{F} \\
    &\leq \norm{\tU \tU^* - \U_\star \U_\star^*}{} \norm{\XI}{F} \norm{\id  - \V_\star \V_\star^*}{} + \norm{\id -\tU \tU^*}{} \norm{\XI}{F} \norm{\tV \tV^* - \V_\star \V_\star^*}{} \\
    & \leq  \left( \norm{\tU \tU^* - \U_\star \U_\star^*}{} + \norm{\tV \tV^* - \V_\star \V_\star^*}{}  \right)  \norm{\XI}{F}. 
    \end{split}
    \end{equation*}
    Now note that, by \cite[Lemma 1]{Cai-RateOptimal2018} and \cite[Theorem 3.5]{Lyu-ExactSinTheta2020}, we obtain
    \begin{equation} \label{eq:LB8:2}
   \begin{split}
      \norm{\tU \tU^* - \U_\star \U_\star^*}{} + \norm{\tV \tV^* - \V_\star \V_\star^*}{}
      &\le 2\left(\specnorm{ \U_\star\tU_{\perp}^*} + \specnorm{\V_\star\tV_{\perp}^*}\right) 
      \leq 4 \frac{\specnorm{\H_{s}(\Xk{k}) - \X_\star }}{\sigma_r(\X_\star)- \sigma_{r+1}(\H_{s}(\Xk{k}))} \\
       &\leq 4 \frac{\specnorm{\Xk{k} - \X_\star }}{\sigma_r(\X_\star)- \sigma_{r+1}(\Xk{k})} \leq 4 \frac{\specnorm{\Xk{k} - \X_\star }}{(1- 1/48) \sigma_r(\X_\star)},
   \end{split}
   \end{equation}
   where we used some small observations in the last two inequalities: First, $\sigma_{r+1}(\H_{s}(\Xk{k})) \leq  \sigma_{r+1}(\Xk{k})$, which follows from the rectangular Cauchy interlacing theorem  \cite[Theorem 23]{dax2010extremum}. Second, according to \Cref{lem:SameSupport} and \cref{eq:lemma:B8:closeness}, the row-support $S$ of $\H_{s}(\Xk{k})$ coincides with the row-support $S_\star = \{i \in [n_1]: \|(\X_{\star})_{i,:}\|_2 \neq 0\}$ of $\X_{\star}$ and hence $\H_{s}(\Xk{k}) - \X_\star$ is a submatrix of $\Xk{k} - \X_\star$. Finally, $\sigma_{r+1}(\Xk{k}) = \specnorm{\T_r(\Xk{k}) - \Xk{k}} \leq \specnorm{\X_\star - \Xk{k}} \leq \frac{1}{48} \sigma_r(\X_\star)$ due to \cref{eq:lemma:B8:closeness}. 
   Consequently,
   \begin{align*}
       \norm{(\P_{\tU,\tV} - \P_{\U_\star,\V_\star}) \XI}{F} \le 4 \frac{\specnorm{\Xk{k} - \X_\star }}{(1- 1/48) \sigma_r(\X_\star)} \norm{\XI}{F}
   \end{align*}
   and, by a similar argument,
   \begin{align*}
       \norm{(\P_{\U_\star,\V_\star} - \P_{\U,\V}) \XI}{F} \le 4 \frac{\specnorm{\Xk{k} - \X_\star }}{(1- 1/48) \sigma_r(\X_\star)} \norm{\XI}{F},
   \end{align*}
   such that it follows from $\P_{\U,\V}^\perp = \id - \P_{\U,\V}$ and $\P_{\tU,\tV}^\perp = \id - \P_{\tU,\tV}$ that
   \begin{equation} \label{eq:LB8:1}
   \begin{split}
       \norm{(\P_{\tU,\tV}^\perp - \P_{\U,\V}^\perp) \XI}{F} &= \norm{(\P_{\tU,\tV} - \P_{\U,\V}) \XI}{F} \leq  \norm{(\P_{\tU,\tV} - \P_{\U_\star,\V_\star}) \XI}{F} +  \Big\|(\P_{\U_\star,\V_\star} - \P_{\U,\V}) \XI \Big\|_F \\
       &\leq 8 \frac{\specnorm{\Xk{k} - \X_\star }}{(1- 1/48) \sigma_r(\X_\star)}\norm{\XI}{F}.
   \end{split}
   \end{equation}
   
   To estimate the fourth summand in \cref{eq:sum:Wop:innerproduct}, we argue analogously that
   \begin{equation} \label{eq:LB8:3}
   \begin{split}
       \norm{\P_{\tSk{k}^c} \XI (\tV\tV^* - \V\V^*)}{F} &\leq  \norm{\XI (\tV\tV^* - \V\V^*)}{F} \leq \left(\specnorm{ \tV\tV^* - \V_\star\V_\star^*} + \specnorm{\V_\star\V_\star^* - \V\V^*}\right)\norm{\XI}{F}  \\
       &\leq 2\left(\specnorm{ \V_\star\tV_{\perp}^*} + \specnorm{\V_\star\V_{\perp}^*}\right)\norm{\XI}{F} \\
       &\leq 2\left(\frac{\specnorm{\H_{s}(\Xk{k}) - \X_\star }}{\sigma_r(\X_\star)- \sigma_{r+1}(\H_{s}(\Xk{k}))} + \frac{\specnorm{\Xk{k} - \X_\star }}{\sigma_r(\X_\star)- \sigma_{r+1}(\Xk{k})}\right)\norm{\XI}{F} \\
       &\leq 4 \frac{\specnorm{\Xk{k} - \X_\star }}{\sigma_r(\X_\star)- \sigma_{r+1}(\Xk{k})}\norm{\XI}{F} \leq 4 \frac{\specnorm{\Xk{k} - \X_\star }}{(1- 1/48) \sigma_r(\X_\star)}\norm{\XI}{F},
   \end{split}
   \end{equation}
   using again that $\sigma_{r+1}(\Xk{k}) \leq \specnorm{\X_\star - \Xk{k}} \leq \frac{1}{48} \sigma_r(\X_\star)$ due to \cref{eq:lemma:B8:closeness}. 
   
Let now $\XIk{k+1} = \Xk{k+1} - \X_\star$. Combining \cref{eq:lemma:B8:RIP} and \cref{eq:sum:Wop:innerproduct}-\cref{eq:LB8:2} we can proceed to estimate that
   \begin{align}
   \begin{split}
   \label{eq:B8intermediateStepII}
       &\| \XIk{k+1} \|_F^2 \\
       &\le  c^2 \norm{\P_{T_k^\perp} (\XIk{k+1})}{F}^2 \\
       &\le 2 c^2 \min\left\{\frac{\sigma_{r+1} (\Xk{k})}{\varepsilon_k},\frac{\rho_{s+1}(\Xk{k})}{\delta_k}  \right\}^2 \big\langle \XIk{k+1}, \Wk{k} \XIk{k+1} \big\rangle
       \\
       &+  2 \Big[\Big(\frac{8}{47/48}\Big)^2 + \Big(\frac{4}{47/48}\Big)^2 \Big] c^2 \| \XIk{k+1} \|_F^2 \frac{\specnorm{\Xk{k} - \X_\star }^2}{\sigma_r^2(\X_\star)} \\ 
       &\le 2 c^2 \min\left\{\frac{\sigma_{r+1} (\Xk{k})}{\varepsilon_k},\frac{\rho_{s+1}(\Xk{k})}{\delta_k}\right\}^2 \big\langle \XIk{k+1}, \Wk{k}(\XIk{k+1}) \big\rangle
       + 167 c^2 \| \XIk{k+1} \|_F^2 \frac{1}{(19 c)^2} \\
       &\le  2 c^2 \min\left\{\frac{\sigma_{r+1} (\Xk{k})}{\varepsilon_k},\frac{\rho_{s+1}(\Xk{k})}{\delta_k}\right\}^2 \big\langle \XIk{k+1}, \Wk{k}(\XIk{k+1}) \big\rangle + \frac{1}{2} \|  \XIk{k+1} \|_F^2,
    \end{split}
   \end{align}
   where the third inequality follows from 
   \cref{eq:DefW}  and \cref{eq:lemma:B8:closeness}.
   Hence, rearranging \cref{eq:B8intermediateStepII} yields
   \begin{align*}
       \| \XIk{k+1} \|_F^2 
       \le 4 c^2 \min\left\{\frac{\sigma_{r+1} (\Xk{k})}{\varepsilon_k},\frac{\rho_{s+1}(\Xk{k})}{\delta_k}\right\}^2 \big\langle \XIk{k+1}, \Wk{k}(\XIk{k+1}) \big\rangle.
   \end{align*}
   By \Cref{lemma:leastsquares:opt}, we know that $\Xk{k+1}$ fulfills
   \begin{align*}
       0 = \big\langle \XIk{k+1}, \Wk{k}(\Xk{k+1}) \big\rangle = \big\langle \XIk{k+1}, \Wk{k}(\XIk{k+1}) \big\rangle + \big\langle \XIk{k+1}, \Wk{k}(\X_\star) \big\rangle
   \end{align*}
   such that we conclude that
   \begin{align*}
       &\norm{\XIk{k+1}}{}^2 \leq \norm{ \XIk{k+1}}{F}^2 \\ 
       &\le 4 c^2 \min\left\{\frac{\sigma_{r+1} (\Xk{k})}{\varepsilon_k},\frac{\rho_{s+1}(\Xk{k})}{\delta_k}  \right\}^2  \big\langle \XIk{k+1}, \Wk{k}(\XIk{k+1}) \big\rangle
       \\
       &= -4 c^2 \min\left\{\frac{\sigma_{r+1} (\Xk{k})}{\varepsilon_k},\frac{\rho_{s+1}(\Xk{k})}{\delta_k}  \right\}^2 \big\langle \XIk{k+1}, \Wk{k}(\X_\star) \big\rangle \\
        &= - 4 c^2  \min\left\{\frac{\sigma_{r+1} (\Xk{k})}{\varepsilon_k},\frac{\rho_{s+1}(\Xk{k})}{\delta_k}  \right\}^2 \left( \big\langle \XIk{k+1},   \Wlr{k}(\X_\star) \big\rangle + \big\langle \XIk{k+1}, \Wsp{k} \cdot \X_\star \big\rangle \right)
        \\
        &\le 4 c^2 \min\left\{\frac{\sigma_{r+1} (\Xk{k})}{\varepsilon_k},\frac{\rho_{s+1}(\Xk{k})}{\delta_k}  \right\}^2 \!\! \left( \norm{\Wlr{k}(\X_\star)}{*} \norm{\XIk{k+1}}{} \!\! + \! \norm{\XIk{k+1}}{} \! \norm{\Wsp{k} \cdot \X_\star}{1,2} \right) \\
        &= 4 c^2 \min\left\{\frac{\sigma_{r+1} (\Xk{k})}{\varepsilon_k},\frac{\rho_{s+1}(\Xk{k})}{\delta_k}  \right\}^2 \left( \norm{\Wlr{k}(\X_\star)}{*} \!+\!  \norm{\Wsp{k} \cdot \X_\star}{1,2} \right) \norm{\XIk{k+1}}{},
   \end{align*}
   which completes the proof. We used in the penultimate line Hölder's inequality and that
   \begin{align*}
       |\langle \AA,\B \rangle_F|
       = \left| \sum_{i,j} A_{i,j} B_{i,j} \right|
       \le \sum_i \norm{\AA_{i,:}}{2} \norm{\B_{i,:}}{2}
       \le \Big( \max_{i} \norm{\AA_{i,:}}{2} \Big) \cdot \sum_{i} \norm{\B_{i,:}}{2}
       \le \norm{\AA}{} \norm{\B}{1,2},
   \end{align*}
   for all matrices $\AA,\B$.
\end{proof}

\subsection{Proof of \Cref{lemma:MatrixIRLS:weight:l1bound}}
\label{sec:ProofOfLemL1Bound}

Note that by \Cref{lem:SameSupport} and \cref{eq:l1boundAssumption}, $S := \supp(\H_{s}(\Xk{k})) = \supp(\X_\star)$. 

Since by assumption $\delta_k \le \rho_s(\Xk{k})$ we have by definition of $\Wsp{k}$ that $\Z~:=~\Wsp{k}~\cdot~\X_\star$ is a matrix with row-support $\tSk{k}$ and rows
   \begin{align*}
       \Z_{i,:} =\min \left\{ \frac{\delta_k^2}{\big\|(\Xk{k})_{i,:}\big\|_2^2}, 1 \right\}  (\X_\star)_{i,:},
   = \frac{\delta_k^2}{\big\|(\Xk{k})_{i,:}\big\|_2^2} (\X_\star)_{i,:}
   \end{align*} 
   for $i \in \tSk{k}$.    
   Now note that if \cref{eq:l1boundAssumption} holds, then
   \begin{align*}
       \| \Z_{i,:} \|_2
       = \frac{ \delta_k^2 \| (\X_\star)_{i,:} \|_2}{\| (\Xk{k})_{i,:} \|_2^2}
       \le \frac{\delta_k^2}{(1-\zeta)^2 \rho_{s}(\X_\star)},
   \end{align*}
   where we used in the last estimate that with \cref{eq:l1boundAssumption} and $\| (\X_\star)_{i,:} \|_2 \ge \rho_s(\X_\star)$, for $i \in S$, we have
   \begin{align*}
       \| (\Xk{k})_{i,:} \|_2^2
       &\ge (\| (\X_\star)_{i,:} \|_2 - \| (\Xk{k})_{i,:} - (\X_\star)_{i,:} \|_2)^2
       \ge (\| (\X_\star)_{i,:} \|_2 - \zeta \rho_s(\X_\star) )^2 \\
       &\ge \Big( (1-\zeta) \rho_{s}(\X_\star) \Big) \Big( (1-\zeta) \| (\X_\star)_{i,:} \|_2 \Big),
   \end{align*}
   for all $i \in \tSk{k}$. The claim easily follows since $\Z$ has only $s$ non-zero rows.

\subsection{Proof of \Cref{lem:B5}}
\label{sec:ProofOfLemB5}

In the proof of \Cref{lem:B5}, we use a simple technical observation.

\begin{lemma}
\label{lem:B5addon}
    Let $\X_\star \in \MM{r,s}$, let $\Xk{k}$ be the $k$-th iterate of Algorithm \ref{algo:MatrixIRLS}, and abbreviate $\XIk{k} = \Xk{k}-\X_\star$. Then the following two statements hold true:
    \begin{enumerate}
    \item If $\varepsilon_{k} \leq  \sigma_{r+1} (\Xk{k})$, then $\varepsilon_k \leq \|\XIk{k}\|$.
    \item If $\delta_{k} \leq  \rho_{s+1} (\Xk{k})$, then $\delta_k \leq \|\XIk{k}\|_{\infty,2}$.
    \end{enumerate}
\end{lemma}

\begin{proof}
    By defining $[\Xk{k}]_r$ to be the best rank-$r$ approximation of $\Xk{k}$ in any unitarily invariant norm, we bound 
   \begin{align*}
       \varepsilon_{k} \leq  \sigma_{r+1}(\Xk{k})
       = \| \Xk{k} - [\Xk{k}]_r \|
       \le \| \Xk{k} - \X_\star \|
       = \| \XIk{k} \|,
   \end{align*}
   where the inequality follows the fact that $\X_\star$ is a rank-$r$ matrix. 

   Similarly, for the second statement, we have that
   \[
   \delta_{k} \leq  \rho_{s+1} (\Xk{k}) =  \| \Xk{k} - \H_{s}(\Xk{k}) \|_{\infty,2} \leq  \| \Xk{k} - \X_\star \|_{\infty,2} =\| \XIk{k} \|_{\infty,2} ,
   \]
   using that $\X_\star$ is $s$-row sparse.
\end{proof}

\begin{proof}[of \Cref{lem:B5}]
First, we note that, using \Cref{lem:RIP}, the observation in \Cref{rem:B8:generalization} yields
\begin{equation} \label{eq:Xi:bound:epsdelta:1}
\begin{split}
    \| \XI \|_F^2 
       &\le 4 c_{\norm{\A}{2\to 2}}^2 \min\left\{\frac{\sigma_{r+1} (\Xk{k})}{\varepsilon_k},\frac{\rho_{s+1}(\Xk{k})}{\delta_k}\right\}^2 \\ 
       &\qquad \cdot \Big\langle \XI, \big(\P_{\U,\V}^\perp \Wlr{k}  \P_{\U,\V}^\perp + \P_{\tSk{k}^c} \Wsp{k}  \P_{\tSk{k}^c}\big)   \XI \Big\rangle
\end{split}
\end{equation}
 for all $\XI \in \mathbb R^{n_1\times n_2}$ as the assumption \cref{eq:lemma:B5:closeness} implies $\specnorm{\Xk{k} - \X_\star} \le \min\Big\{\frac{1}{48},\frac{1}{19 c }\Big\} \sigma_r(\X_\star)$. Thus, this holds in particular also for $\XIk{k} = \Xk{k}-\X_\star$. (Recall that $\U$ and $\V$ contain the leading singular vectors of $\Xk{k}$, see \Cref{def:weight:operator}.)
 We estimate that
 \begin{equation*}
\begin{split}
&\sqrt{\big\langle \P_{\U,\V}^\perp\XIk{k}, \Wlr{k}  \P_{\U,\V}^\perp  \XIk{k} \big\rangle} = \norm{(\Wlr{k})^{1/2} (\P_{\U,\V}^\perp(\XIk{k}) )   }{F} \\
&\leq  \norm{(\Wlr{k})^{1/2} (\P_{\U,\V}^\perp(\X_\star) )  }{F} + \norm{(\Wlr{k})^{1/2} (\P_{\U,\V}^\perp(\Xk{k}) )   }{F} \\
&\leq \norm{(\Wlr{k})^{1/2} (\X_\star)    }{F} + \sqrt{\sum_{i=r+1}^{\min(n_1,n_2)} \frac{\sigma_i^2}{\max\big(\frac{\sigma_i^2}{\varepsilon_k^2},1\big)}} \leq \norm{(\Wlr{k})^{1/2} (\X_\star)    }{F} + \sqrt{n-r} \varepsilon_k
\end{split}
 \end{equation*}
 
 Furthermore, since $\max(\varepsilon_k,\|\XIk{k}\|) \leq \frac{1}{48} \sigma_r(\X_\star)$ by assumption and $\varepsilon_k \leq \|\XIk{k}\|$ by \Cref{lem:B5addon}, we can use a variant of \Cref{lemma:MatrixIRLS:weight:nuclearnormbound} to obtain
 \begin{equation*}
 \begin{split}
 \norm{(\Wlr{k})^{1/2} (\X_\star)    }{F}  &\leq \frac{48}{47} \left(\sqrt{r} \varepsilon_k  +2 \varepsilon_k \frac{\|\XIk{k}\|_F}{\sigma_r(\X_\star)} + 2 \frac{\|\XIk{k}\| \|\XIk{k}\|_F}{\sigma_r(\X_\star)} \right) \\
 &\leq 1.04 \sqrt{r} \varepsilon_k +  \frac{4.16 \|\XIk{k}\|}{\sigma_r(\X_\star)}  \|\XIk{k}\|_F.
 \end{split}
 \end{equation*}
On the other hand, we note that $\XIk{k}$ restricted to $\tSk{k}^c$ coincides with the restriction of $\Xk{k}$ to $\tSk{k}^c$ under assumption \cref{eq:lemma:B5:closeness}, cf.\ \Cref{lem:SameSupport}, and therefore 
\begin{equation*}
\begin{split}
\big\langle \P_{\tSk{k}^c} \XIk{k},  \Wsp{k}  \P_{\tSk{k}^c}  \XIk{k} \big\rangle &= \big\langle \P_{\tSk{k}^c} \Xk{k},  \Wsp{k}  \P_{\tSk{k}^c}    \Xk{k} \big\rangle \\
&=  \sum_{i=s+1}^{n_1} \frac{\| (\Xk{k})_{i,:} \|_2^2}{\max\{\| (\Xk{k})_{i,:} \|_2^2/\delta_k^2,1 \}} \leq (n_1-s) \delta_k^2.
\end{split}
\end{equation*}
With the estimate of above, this implies that
\begin{align*}
    \begin{split}
    &\Big\langle \XI, \big(\P_{\U,\V}^\perp \Wlr{k}  \P_{\U,\V}^\perp + \P_{\tSk{k}^c} \Wsp{k}  \P_{\tSk{k}^c}\big)   \XI \Big\rangle\! \\
    &\leq \! \Big(\!\norm{(\Wlr{k})^{1/2} (\X_\star)    }{F}\!\! +\! \sqrt{n-r} \varepsilon_k\Big)^2\!\!\! +\!  (n_1-s) \delta_k^2 \\
    &\leq \frac{13}{4} r \varepsilon_k^2 +  \frac{52 \|\XIk{k}\|^2}{\sigma_r^2(\X_\star)} \|\XIk{k}\|_F^2 + 3 (n-r) \varepsilon_k^2 + (n_1-s) \delta_k^2.
    \end{split}
\end{align*}

Inserting these estimates into \cref{eq:Xi:bound:epsdelta:1}, we obtain
\[
\| \XIk{k} \|_F^2      \le 4 c_{\norm{\A}{2\to 2}}^2  \min\left\{\frac{\sigma_{r+1} (\Xk{k})}{\varepsilon_k},\frac{\rho_{s+1}(\Xk{k})}{\delta_k}\right\}^2 \left( \frac{13}{4} n \varepsilon_k^2 + n_1 \delta_k^2 + \frac{52 \|\XIk{k}\|^2}{\sigma_r^2(\X_\star)}  \|\XIk{k}\|_F^2 \right).
\]
If now either one of the two equations  $\varepsilon_k = \sigma_{r+1} (\Xk{k})$ or $\delta_k =\rho_{s+1}(\Xk{k}) $ is true, it follows that
\begin{equation*}
\begin{split}
\| \XIk{k} \|_F^2 &\leq c_{\norm{\A}{2\to 2}}^2 \left( 13 n \varepsilon_k^2 + 4 n_1 \delta_k^2 \right) + c_{\norm{\A}{2\to 2}}^2 \frac{208 \|\XIk{k}\|^2}{\sigma_r^2(\X_\star)}  \|\XIk{k}\|_F^2 \\
&\leq c_{\norm{\A}{2\to 2}}^2 \left( 13 n \varepsilon_k^2 + 4 n_1 \delta_k^2 \right) + \frac{1}{2} \|\XIk{k}\|_F^2 
\end{split}
\end{equation*}
if the proximity condition $\|\XIk{k}\| = \|\Xk{k} - \X_{\star}\| \leq \frac{1}{21 c_{\norm{\A}{2\to 2}}} \sigma_r(\X_\star)$ is satisfied. Rearranging the latter inequality yields the conclusion of \Cref{lem:B5}.
\end{proof}

\end{document}